\documentclass[10pt]{article} % For LaTeX2e
\usepackage[preprint]{tmlr}
\usepackage{xcolor}
\definecolor{citationblue}{RGB}{0,70,140}
\definecolor{refgreen}{RGB}{0, 158, 115}
\usepackage[
    colorlinks=true,
    citecolor=blue,
    linkcolor=refgreen,
    urlcolor=blue
]{hyperref}
\usepackage{amsmath,amsfonts,bm}

\def\eqref#1{equation~\ref{#1}}
\def\1{\bm{1}}

\DeclareMathAlphabet{\mathsfit}{\encodingdefault}{\sfdefault}{m}{sl}
\SetMathAlphabet{\mathsfit}{bold}{\encodingdefault}{\sfdefault}{bx}{n}

\usepackage{graphicx}
\usepackage{url}
\usepackage{subcaption}
\usepackage{booktabs}
\usepackage{tabularx}
\usepackage{pifont}
\usepackage{multirow}
\usepackage{adjustbox}
\newcommand{\cmark}{\textcolor{green!60!black}{\ding{51}}}  % ✓ in dark green
\newcommand{\xmark}{\textcolor{red}{\ding{55}}}

\newtheorem{theorem}{Theorem}[section]
\newtheorem{lemma}[theorem]{Lemma}
\newtheorem{definition}{Definition}
\newtheorem{assumption}{Assumption}

\newtheorem{proof}{Proof}
\usepackage[ruled,vlined,linesnumbered]{algorithm2e}
\SetKwInput{KwInput}{Input}

\title{Long-Memory Reservoir Computing for Data-Scarce Dengue Forecasting}

\author{\name Rahul Goswami
  \email rahul.goswami@iitg.ac.in \\
  \addr Department of Mathematics,
  Indian Institute of Technology Guwahati, India \\
  SAFIR, Sorbonne University Abu Dhabi,
  United Arab Emirates
  \AND
  \name Shinjini Paul
  \email s.paul@math.leidenuniv.nl \\
  \addr Leiden University,
  Leiden, South Holland, Netherlands
  \AND
  \name Palash Ghosh
  \email palash.ghosh@iitg.ac.in \\
  \addr Department of Mathematics,
  Indian Institute of Technology Guwahati, India
  \AND
  \name Tanujit Chakraborty
  \email tanujit.chakraborty@sorbonne.ae \\
  \addr SAFIR, Sorbonne University Abu Dhabi,
  United Arab Emirates \\
  Sorbonne Center for Artificial Intelligence,
  Sorbonne University, Paris, France   
}

\def\month{MM}  % Insert correct month for camera-ready version
\def\year{YYYY} % Insert correct year for camera-ready version
\def\openreview{\url{https://openreview.net/forum?id=XXXX}} % Insert correct link to OpenReview for camera-ready version

\makeatletter
\AtBeginDocument{%
  \@ifundefined{headrulewidth}{}{%
    \renewcommand{\headrulewidth}{0pt}%
  }%
}
\makeatother
\begin{document}

\maketitle

\begin{abstract}
Accurate dengue forecasting is crucial for public health planning, but remains challenging because incidence series are often short, noisy, non-stationary, nonlinear, and often affected by long-range temporal dependence. Fractional differencing in Autoregressive Fractionally Integrated Moving Average (ARFIMA) helps balance non-stationarity and persistence, but its linear structure limits its ability to capture nonlinear dynamics. Deep neural networks can model nonlinear patterns, but usually require large training samples and do not explicitly encode statistical long memory. Echo State Networks (ESNs), a widely used reservoir computing framework, are attractive in this setting because they retain nonlinear recurrent dynamics while training only a simple readout, making them suitable for data-scarce scenarios. However, standard ESNs lack long-term memory from a time-series perspective. This study proposes a long-memory reservoir computing framework that integrates dedicated long-memory and short-memory ESN reservoirs with a ridge-regression readout. We introduce two variants: Fractional ESN (fESN), which incorporates fractional-differencing dynamics into the reservoir to encode long-range dependence directly, and Wavelet ESN (wESN), which extracts stable low-frequency components through wavelet smoothing before modeling them with a memory-aware reservoir. We establish theoretical guarantees for closed-loop reservoir dynamics, showing that standard ESNs induce short-memory processes under mild conditions, whereas the proposed long-memory reservoirs generate polynomially decaying dependence consistent with statistical long memory. Across multiple dengue datasets and forecasting horizons, fESN and wESN outperform statistical and deep learning baselines. Combining conformal prediction with fESN and wESN provides distribution-free calibrated uncertainty intervals for operational public health decision-making. The `memory-esn' Python package offers an implementation of our proposed approaches.

\end{abstract}

\section{Introduction}
Dengue remains a major public health concern in many tropical and subtropical regions, where timely forecasting can support preparedness, resource allocation, vector-control planning, and early warning interventions \citep{halstead2007dengue, wearing2006ecological, bhatt2013global}. Most infections cause a sudden onset of high fever, severe headache, muscle and joint pain, and skin rash. However, the illness can sometimes progress to dengue hemorrhagic fever (a severe form involving plasma leakage and bleeding) or dengue shock syndrome (a critical drop in blood pressure that can lead to organ failure and death if untreated) \citep{iqbal2025clinical}. The disease is prevalent in more than 100 countries, particularly in tropical and subtropical regions, placing roughly half of the world's population at risk, with Asia alone accounting for about 70\% of the global dengue burden \citep{WHO2025Dengue, MalariaConsortium2024}. Transmission occurs mainly through the bites of infected female Aedes aegypti mosquitoes, as well as Aedes albopictus, with the spread driven by a combination of environmental conditions, climate change, urbanization, globalization, mosquito ecology, and human movement \citep{whitehorn2015comparative}. Over the past five decades, the number of reported cases has increased by more than thirty-fold, including an eightfold global rise between 2000 and 2019. In 2024, the Americas alone reported more than 10.6 million cases, and India's annual total rose from about half a million in 2000 to more than five million in 2019 \citep{salehi2025global}. Large outbreaks can overwhelm healthcare systems, and in severe cases without prompt treatment, the case fatality rate can approach 20\%. Beyond the health impact, dengue imposes a substantial socioeconomic burden on many endemic countries through medical costs, lost productivity, and emergency response demands \citep{rehan2022dengue}. These trends underscore the urgent need for accurate, localized, and timely forecasts of dengue incidence as early warning tools to prevent and control outbreaks \citep{chakraborty2019forecasting, hii2012forecast}. Such forecasts enable proactive public health planning, targeted vector control measures, efficient resource allocation, and timely risk communication. 

This work is motivated by the absence of a single forecasting approach that simultaneously accommodates four salient features of dengue incidence series: long-range dependence, non-stationarity, nonlinear dynamics, and limited sample sizes. In practice, epidemiologists and public health analysts rely on classical linear models such as Autoregressive Integrated Moving Average (ARIMA) to forecast dengue incidence \citep{chen2025assessing,hasan2024dengue}. Despite the inherent simplicity and interpretability of these models, their mathematical formulation relies on several assumptions of linearity that are routinely violated in dengue data \citep{tang2022qualitative}. 
% ARIMA models can address some non-stationarity via differencing, but they still assume short memory, where correlations between current and past observations decay rapidly with lag , and linear relationships. However, dengue incidence series arise from a complex nonlinear interplay of climate variability, human mobility, mosquito ecology, virus serotype interactions, and socioeconomic conditions. 
% A key reason classical models remain dominant is that they are built around the principle of stationarity, which is central to both theoretical guarantees and practical forecasting. A time series is stationary when its statistical characteristics remain invariant over time; in many practical applications, the weaker notion of stationarity requiring only time invariant mean and autocovariance suffices. Stationarity plays a central role for two reasons. First, the limiting distributions of parameter estimators for models describing the data-generating process are typically guaranteed only under stationarity assumptions. Second, a stationary model produces series whose statistical structure does not evolve over time, enabling reliable long-horizon forecasting by iteratively propagating model-generated values forward. Consequently, transforming a non-stationary series into a stationary one is standard practice, most commonly via integer differencing.
ARIMA addresses the issue of nonstationarity via differencing; however, differencing carries a significant drawback. Integer-order differencing eliminates long-range dependence (or ``memory''), with autocorrelations that decay slowly and preserve the influence of past conditions over extended horizons, which is inherent in many real-world epidemic datasets \citep{schaffer2021interrupted}. This creates a fundamental dilemma: for meaningful estimation and stable long-term forecasting, the data-generating process should be stationary, yet it must also retain the intrinsic long-memory structure present in the observed data. This issue was addressed by \citet{GrangerJoyeux1980} via introducing fractional differencing \citep{Hosking1981}, which simultaneously achieves stationarity and preserves long memory. Autoregressive Fractionally Integrated Moving Average (ARFIMA) can represent slowly decaying correlations through fractional differencing; however, in practice, these models are constrained by rigid parametric forms and struggle with the nonlinear nature of real-world dengue data \citep{devianto2022hybrid}. This mismatch between model assumptions and the complexity of the data calls for more flexible, data-driven methods that can handle nonlinearity and capture long-term dependencies without restrictive structural assumptions.

A large class of deep learning models has been introduced that implements latent transitions through hidden states and observation maps using neural networks, learning parameters by backpropagation through time \citep{baydin2018automatic}. The most prominent examples are recurrent neural networks (RNNs) \citep{rumelhart1986learning} and their gated variants, Long Short–Term Memory (LSTM) \citep{hochreiter1997long}, and Gated Recurrent Units (GRUs) \citep{cho2014learning}. These models relax linearity and stationarity assumptions, and, in practice, they often outperform classical statistical baselines for dengue incidence forecasting \citep{zhao2023deep, llerena2021forecasting, panja2023epicasting}. However, the common intuition that a persistent hidden state confers genuine long-range dependence has recently been questioned. Under mild regularity conditions, the update equations of standard RNNs and LSTMs define geometrically ergodic network processes whose correlations decay exponentially, i.e., statistically speaking, they exhibit short memory \citep{zhao2020rnn}. Related analyses for GRUs report similar conclusions, even though gating alleviates vanishing gradients and can yield empirical gains in some tasks \citep{yang2025memory}. This limitation is salient for dengue, where multi-seasonal dynamics of mosquitoes and viruses and slowly varying socio-environmental drivers induce persistence; if a forecaster `forgets' too quickly, it risks discarding informative long-horizon signals.

To address the short-memory issue in RNNs and LSTMs, \citet{zhao2020rnn} adopted the classical idea of embedding a fractionally integrated (power-law) memory filter \citep{Hosking1981, granger1980introduction} within the recurrent architecture. This yielded memory-augmented variants such as the memory-augmented RNN (MRNNF) and the memory-augmented LSTM (MLSTMF), in which a learnable ``memory order'' controls the persistence of the impulse response. Analogously, the embedded filter endows the network with a polynomially decaying influence of past inputs, enabling dependence that extends well beyond the short horizons accessible to vanilla RNNs and LSTMs. The key idea in \citet{zhao2020rnn} is to maintain a memory state and update it with the fractionally differentiated input, denoted as the memory input, and the previous memory state, allowing the network to retain information from the distant past. The updates are similar to the usual hidden-state update in a vanilla RNN. Similarly, memory-augmented GRUs (namely, MGRUF) were introduced by allowing the fractional differencing parameter to adapt to changing temporal structure \citep{yang2025memory}. Although both fixed- and dynamic memory-augmented recurrent models can, in principle, capture nonlinear dynamics and long-range dependence, they require large training sets to achieve stable estimation. In dengue surveillance, counts are typically aggregated at weekly or monthly resolution, producing short samples and coarse temporal granularity, which together impede reliable estimation and generalization of high-capacity networks \citep{xu2020forecast}.

To address these limitations, this research aims to develop long-memory reservoir computing frameworks that enable fast training and robust processing of limited time series data (e.g., small sample datasets such as epidemic time series). Echo state networks (ESNs) \citep{jaeger2001echo} retain nonlinear recurrent dynamics while training only a ridge-regression readout, making them computationally efficient, stable, and better suited to data-scarce dengue forecasting. ESNs have become popular for time-series forecasting because they provide a lightweight method to learn nonlinear dynamical patterns while avoiding full backpropagation through recurrent weights \citep{banerjee2022predicting}. \citet{pathak2018model} used reservoir computing for model-free and hybrid prediction of chaotic and spatiotemporally chaotic systems, demonstrating its value when the governing dynamics are only partially known. Recent applications of ESNs to understand complex systems in ecology and physiology can be found in \citep{yang2026self}. Epidemic applications of ESNs include predicting burst synchronization and multistable dynamics \citep{roy2022model, ghosh2021reservoir}, highlighting their usefulness for complex-network time series and nonlinear dynamical regimes. Although recent reservoir-memory models store and retrieve dynamical attractors \citep{kong2024reservoir} or use attention over reservoir states to improve time-series classification \citep{ma2021echo}, they treat memory mainly as an architectural representation mechanism. Relatedly, \citet{carroll2022optimizing} emphasized that the fading-memory length of a reservoir computer is a key tunable dynamical property; however, this notion concerns how long transient input effects persist in a stable reservoir, whereas our work targets statistical long memory through polynomially decaying dependence \citep{geweke1983estimation}. In contrast, our focus is on long memory in the statistical time-series sense, characterized by slowly decaying polynomial dependence. Our aim is to encode fractional long-range dependence within ESNs while preserving their data efficiency, making them suitable for short, nonlinear, and persistent dengue incidence series.

\subsection*{Our Contributions} The key contributions from our study are as follows.
\begin{itemize}
\item {\bf Methodological Contribution:} We introduce a long-memory reservoir computing framework for data-scarce dengue forecasting. We propose two concrete variants of the framework, namely Fractional ESN (fESN) and Wavelet ESN (wESN). The fESN model injects a fractional-difference memory input into the memory reservoir, allowing the network to encode long-range dependence through polynomially decaying fractional-filter coefficients. The wESN model (particularly useful for high-frequency data or seasonal data) first applies wavelet smoothing to extract persistent low-frequency structure and then constructs the memory input from the smoothed signal. These two variants provide complementary mechanisms for jointly handling long-range dependence, non-stationarity, nonlinear dynamics, and small-sample constraints in dengue forecasting.

\item {\bf Theoretical Contribution:} We develop theoretical guarantees for the memory properties of ESN-based generative dynamics. Our main theoretical contributions are as follows: (a) 
Closed-loop ESNs are geometrically ergodic under mild conditions and therefore induce short-memory processes (Theorem~\ref{thm:esn-geo-ergodic} and~\ref{thm:esn_nonlinear}); (b) Vanilla ESNs cannot generate long-memory network processes (Theorem~\ref{thm:vanilla-no-lm}); and (c) Our proposed memory-augmented variants (fESN and wESN) recover genuine long-range dependence under mild conditions (Theorem~\ref{thm:fesn-lm} and~\ref{thm:wesn-lm}). These results provide a theoretical explanation for why explicit memory reservoirs are necessary for forecasting dengue series with persistent temporal dependence in data scarcity.

\item {\bf Empirical Contribution:} We provide a comprehensive empirical evaluation using nine dengue incidence datasets across multiple forecast horizons, where data are collected on monthly/weekly frequencies. The proposed fESN and wESN models are compared against statistical, recurrent, attention-based, memory-augmented recurrent, and reservoir-based baselines using four forecasting error metrics. The results show that the proposed models achieve the best average ranks across the main metrics, with fESN attaining the strongest overall performance and wESN providing closely comparable accuracy. Diebold--Mariano test \citep{diebold2002comparing} further confirms statistically significant forecasting gains over most competing methods.

\item {\bf Uncertainty Quantification and Validation of Theoretical results via Simulation:} We complement point forecasting with uncertainty quantification \citep{vovk2005conformal}, robustness analysis, and simulation diagnostics. Using split conformal prediction (a model-agnostic approach), we construct calibrated prediction intervals for ESN, fESN, and wESN forecasts and show that fESN often provides a favorable balance between empirical coverage and interval width. An ablation study demonstrates that wESN is robust to the choice of wavelet family and decomposition level, motivating the use of a simple Haar wavelet with one decomposition level. Finally, an ARFIMA-based simulation study verifies the theoretical memory claims by contrasting the exponential correlation decay of vanilla ESNs with the polynomial decay induced by the proposed long-memory reservoirs.

\end{itemize}

The remainder of this paper is organized as follows. Section~\ref{sec:motivating_example} presents a motivating example based on the dengue incidence series from Colombia, to illustrate the long-memory and data-scarcity problem in epidemic time-series forecasting. Section~\ref{sec:preliminaries} outlines the notations, definitions, and ESN architecture. Section~\ref{sec:prop} proposes two new variants of long-memory ESNs. Section~\ref{sec:asymptotics} provides the asymptotic results. Section~\ref{sec:exp} outlines the experimental procedures and results, along with conformal prediction. Section~\ref{sec:simulation} presents a simulation study, and Section~\ref{sec:discussion} concludes the paper with limitations and future research directions.

\section{Motivating Example: Colombia Dengue Incidence} \label{sec:motivating_example}

Dengue incidence data exhibit a distinctive combination of long-range dependence, non-stationary patterns, nonlinear dynamics, and limited sample sizes resulting from coarse temporal granularity. Since cases are typically reported on a weekly or monthly basis, only fifty-two or twelve observations are available each year. These structural patterns are evident in the weekly dengue incidence cases of Colombia reported during 2005 to 2016. As highlighted in Figure~\ref{fig:colombia}, panel (A) shows the historical dengue time series and illustrates its non-stationary behavior (also see Table~\ref{tab:epidemic_summary}). The series displays prolonged periods of elevated and reduced dengue activity, with outbreaks emerging and diminishing over extended intervals rather than fluctuating around a stable mean. Following a sudden rise or decline in case counts, the system does not revert quickly to baseline conditions; instead, the effects of these shocks persist for many weeks, creating a pronounced long-memory pattern in the underlying transmission dynamics (seasonal also in this case). This persistence is further highlighted in panel (B) of Figure~\ref{fig:colombia}, where the autocorrelation function (ACF) decays slowly with increasing lag and remains statistically significant across a wide range of weeks. Such a gradual decline confirms that dengue dynamics are governed by long-term dependence rather than short-lived fluctuations. Additionally, the series exhibits nonlinear behavior (as evidenced by statistical test results in Table~\ref{tab:epidemic_summary}), reflecting complex interactions within the transmission process. Panel (C) of Figure~\ref{fig:colombia} shows that, even after removing the dominant linear seasonal effect, captured by the 52-week lag, the residuals continue to vary systematically in ways indicative of regime-dependent responses or saturation effects. 

\begin{figure}
    \centering
    \includegraphics[width=0.98\linewidth]{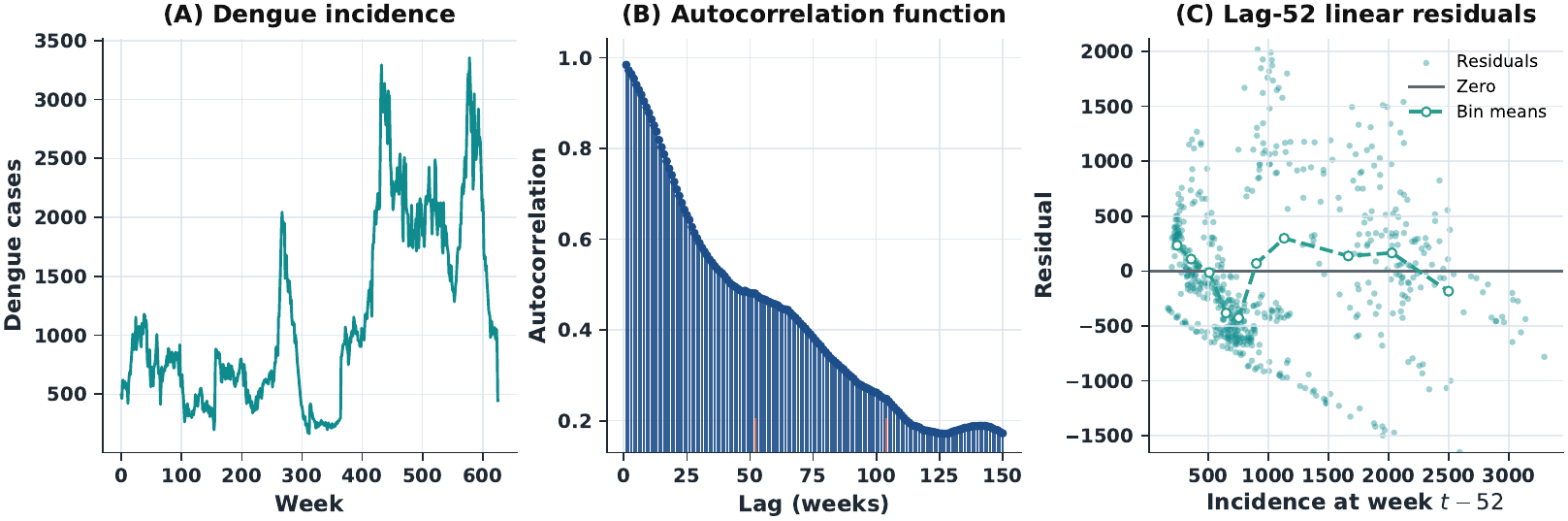}
    \caption{Weekly dengue incidence (Colombia). (A) Raw series shows large, clustered outbreaks and pronounced seasonality. (B) Autocorrelation up to 150 lags (vertical ticks at seasonal multiples of 52). The ACF decays slowly and remains positive far into the tail, consistent with long memory (power law persistence rather than the rapid exponential decay of short-memory processes). (C) Residuals from the linear mean. The bin–averaged residuals (dots, dashed) depart systematically from zero, indicating state-dependent (nonlinear) mean dynamics at the annual horizon.}
    \label{fig:colombia}
\end{figure}

For instance, seasons with very high incidence in the previous year often lead to outcomes that deviate from linear expectations, reflecting the influence of susceptibility depletion, vector ecology, or other nonlinear transmission dynamics. Such state-dependent patterns violate the affine conditional mean assumed by classical linear models and cannot be expressed as a simple linear superposition of past values. These intertwined features of dengue incidence dynamics pose a significant challenge for standard time series forecasting techniques, which typically address them in isolation rather than jointly. Classical ARIMA and ARFIMA models address non-stationarity or long-range dependence only partially: integer differencing destroys persistence, while fractional differencing preserves it but remains essentially linear and cannot capture nonlinear dengue transmission dynamics. Deep recurrent models can learn nonlinear patterns, but standard RNN, LSTM, and GRU-type architectures lack genuine statistical long memory under geometric-ergodicity arguments \citep{zhao2020rnn, yang2025memory}, while memory-augmented variants such as MRNN, MLSTM, and MGRU are often too ``data-hungry'' for short epidemiological series \citep{ahmadi2023statistical, panja2023epicasting, pathak2026deep}.

This mismatch between the statistical properties of the dengue incidence data and the assumptions underlying existing approaches underscores the need for a forecasting framework that can jointly capture long-range dependence, non-stationarity, nonlinear dynamics, and stable learning in small-sample settings. Motivated by these limitations, we propose long-memory ESNs (via fractional differencing and Wavelet decomposition) that incorporate an evolving memory reservoir capable of modeling long-memory behavior and non-stationarity while simultaneously accommodating nonlinear state dependence and maintaining robustness in low-sample environments.

\section{Preliminaries} \label{sec:preliminaries}

\subsection{Notation and Problem Setup}
In this study, the dengue incidence data for a region reflect a univariate time series, and we aim to develop a forecasting framework that can accurately predict the future dynamics of the epidemic incidence. In the mathematical notation, we use univariate (scalar) representations, and multivariate extensions (vector-valued quantities) are denoted in bold. Let $\{u(t)\}_{t=1}^T$ denote the dengue incidence cases, observed at time $t=1,\dots,T$. Given the historical information set $\mathcal{F}_T=\{u(1),\dots,u(T)\}$, our objective is to generate $H$–step–ahead forecasts $\left\{\hat{u}(T + h)\right\}_{h = 1}^H$. We denote the backshift operator by $B$, which is given by $B(u(t)) = u(t-1)$. Following \citet{zhao2020rnn}, we model the data‐generating mechanism of a recurrent network with input $u(t)$, model output $z(t)$, and target $y(t)$ as
\begin{equation}
\label{eq:network_process}
y(t) \;=\; z(t) \;+\; \epsilon(t), \qquad t\in\mathbb{Z},
\end{equation}
where $\{\epsilon(t)\}$ is an i.i.d.\ noise sequence. We refer to the generated target sequence, defined by Eqn.~\ref{eq:network_process}, as the network process.  Further, for the theoretical studies in Section~\ref{sec:asymptotics}, we consider a closed-loop 1-step recurrent setup without loss of generality, i.e., without any exogenous inputs and $u(t) = y(t-1)$. We denote the spectral radius of a matrix by $\rho(\cdot)$ and the uniform distribution by $\mathcal{U}[\cdot,\cdot]$.

\subsection{Echo State Networks (ESNs)}
ESNs are recurrent models in the reservoir computing family in which the recurrent reservoir is randomly initialized and kept fixed, while only a linear readout is trained \citep{jaeger2001echo}. Let \(u(t)\in\mathbb{R}\) denote the input at time \(t\), \(y(t)\in\mathbb{R}\) the forecasting target, and $\mathbf{x}(t)\in\mathbb{R}^{p}$ the reservoir state, where \(p\) is the reservoir size. A standard ESN evolves according to
\begin{align}
\label{eq:esn-loss}
\ell(t) &= \bigl\lVert z(t)-y(t)\bigr\rVert^2_2, \\
\label{eq:esn-readout}
z(t) &= g\!\bigl(\,\mathbf{W}_{\!\mathrm{out}}\,\mathbf{x}(t)\,\bigr),\\
\label{eq:esn-update}
\mathbf{x}(t) &= \tau\!\bigl(\,\mathbf{W}_{\!xu}\,u(t)\;+\;\mathbf{W}_{\!xx}\,\mathbf{x}(t-1) + \eta(t)\,\bigr), 
\end{align}
for \(t=1,\dots, T\), where \(\tau(\cdot)\) is an elementwise nonlinear activation function (e.g., $\tanh$), $g(\cdot)$ is the output activation (typically the identity for regression), $\mathbf{W}_{\!xu}\in\mathbb{R}^{p\times 1}$ is the input weight matrix, $\mathbf{W}_{\!xx}\in\mathbb{R}^{p\times p}$ is the recurrent reservoir matrix, $\mathbf{W}_{\!\mathrm{out}}\in\mathbb{R}^{1\times p}$ is learned from data, and $\eta(t) \sim \mathcal{N}(0, \sigma^2 \mathbf{I_p})$ is an optional reservoir noise term. We omit bias terms for simplicity and adopt this convention for all recurrent formulations used in the paper. We use the zero initial condition $\mathbf{x}(0)=\mathbf{0}_{p}$ and set $u(t)=0$ for $t\le0$ to initialize the dynamics. The matrices $\mathbf{W}_{\!xx}$ and $\mathbf{W}_{\!xu}$ are fixed after random initialization, and $\mathbf{W}_{\!xx}$ is commonly rescaled so that its spectral radius satisfies $\rho(\mathbf{W}_{\!xx})<1$, which is a standard sufficient design condition for the echo state property (ESP).

During training, the reservoir is run on the observed input to collect states $\{\mathbf{x}(t)\}$. After a washout of $T_0$ steps, only the readout $\mathbf{W}_{\!\mathrm{out}}$ is learned by minimizing
\begin{equation*}
\label{eq:esn:emp-risk}
\frac{1}{T-T_0}\sum_{t=T_0+1}^{T}\ell(t),
\end{equation*}
usually via least squares or ridge regression. An optional ``direct'' input–to–output skip when $g$ is the identity function can be included by augmenting the readout with the current input
\[
z(t)\;=\;g\!\Bigl(\,
\underbrace{\begin{bmatrix}\mathbf{W}_{\!\mathrm{out}}; & w_u\end{bmatrix}}_{\in\mathbb{R}^{1\times (p+1)}}
\begin{bmatrix}\mathbf{x}(t)\\ u(t)\end{bmatrix}\Bigr),
\]
in which case the design matrix stacks \(\{[\mathbf{x}(t);u(t)]\}\) over \(t=T_0+1,\dots,T\) and $w_u$ indicates the weight of the skip connection. Thus, ESNs retain nonlinear recurrent dynamics through the fixed high-dimensional reservoir, while avoiding backpropagation through recurrent weights.

For this reason, ESNs are computationally efficient and attractive in limited-data forecasting settings. However, as shown in Section~\ref{sec:asymptotics}, closed-loop ESNs are geometrically ergodic under suitable stability conditions and therefore exhibit short-memory behavior (refer to Theorem~\ref{thm:esn-geo-ergodic} and~\ref{thm:esn_nonlinear}). Consequently, while standard ESNs can capture nonlinear short-term dynamics, they cannot reproduce the slowly decaying dependence structure of long-memory processes (e.g., dengue epidemic time series in Section \ref{sec:motivating_example}). This limitation motivates the long-memory ESN architectures proposed in the next section.

\section{Proposed Method} \label{sec:prop}

\begin{figure}[ht!]
  \centering
  \includegraphics[]{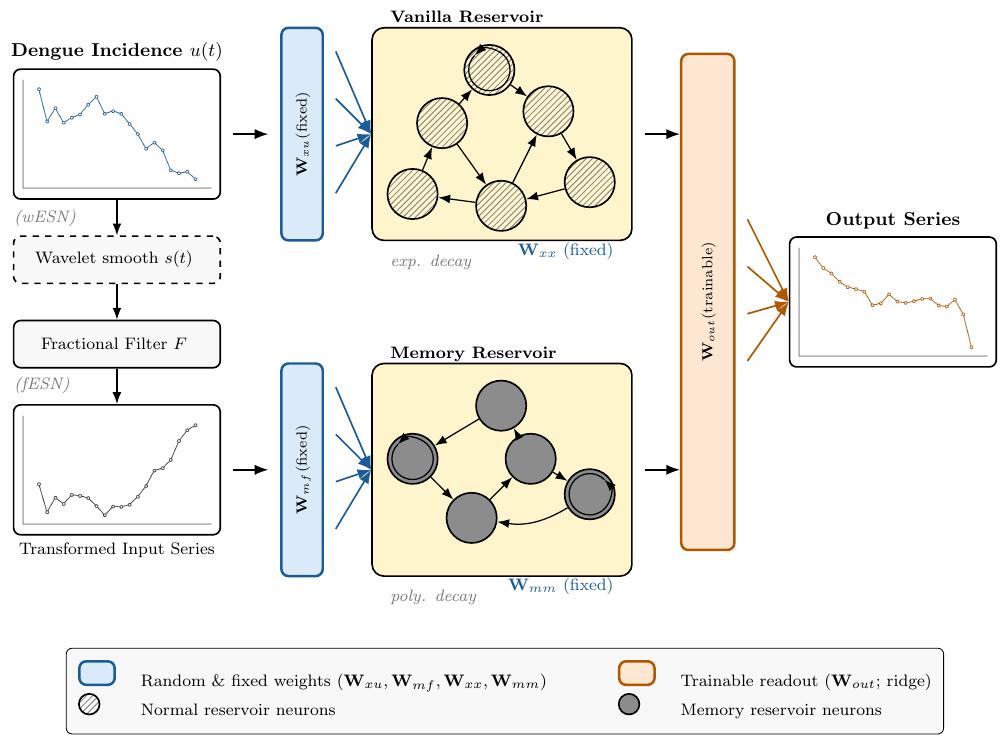}
  \caption{Flow diagram for {\it fESN} and {\it wESN}: Blue components denote \emph{random, fixed} weights (reservoir and input couplings). The orange component denotes the \emph{trainable} linear readout learned via ridge regression. The memory path includes an optional wavelet smoothing \(s(t)\) (for wESN) followed by fractional/memory filtering \(f(t)\); the standard path feeds the conventional reservoir. For fESN, it does not include the Wavelet smooth part and instead applies fractional differencing directly to the original time series. The standard reservoir captures \emph{short-term} dynamics (exponential decay), whereas the memory reservoir in fESN and wESN sustains \emph{long-term} dynamics (polynomial decay). States from both reservoirs are concatenated and mapped to the output. Direct connection from inputs is not represented.}
  \label{fig:maesn_arch}
\end{figure}

\subsection{Motivation}
RNNs and their variants like LSTM and GRUs can incorporate memory filters directly by concatenating the memory-filtered input with the previous memory state to form the current memory state and the weights are learned by backpropagation through time \citep{zhao2020rnn, yang2025memory}. ESNs, in contrast, do not update reservoir weights by backpropagation; they are fixed (random), and the reservoir is typically large. This makes it nontrivial to incorporate memory in the same way as in RNNs. To resolve this, we propose a parallel reservoir that ingests a memory input, rather than concatenating features. 
The memory input is injected through its own random input weights and propagated in a dedicated memory reservoir. Further, the memory state is updated by an equation analogous to Eqn.~\ref{eq:esn-update}, using separate memory-reservoir weights. This design integrates explicit long-memory filtering with standard ESN dynamics while avoiding backpropagation through recurrent weights, thereby preserving the computational efficiency and small-sample suitability of reservoir computing.

Motivated by these, we propose two variants of ESNs for modeling long-range dependence in univariate time series:
\emph{Fractional Echo State Network} (fESN) 
and \emph{Wavelet Echo State Network} (wESN). Both models extend the standard ESN by introducing a dedicated memory mechanism that captures persistent dependencies, while keeping the internal weights fixed after random initialization. We construct two coupled reservoirs for fESN and wESN to address this task. The first, the vanilla reservoir (named to make the difference explicit), captures short-term nonlinear dynamics directly from the raw input sequence. The second, termed the memory reservoir, is driven by memory input and is specifically designed to retain and represent long-range temporal dependencies. The key difference between the two models is how the input is fused in the memory reservoir. In wESN, the input is first smoothed by a wavelet transformation before fractional filtering is applied. The subsequent subsections describe each model in detail, followed by a unified section outlining the procedures for weight initialization and output weight training. The overall architecture of the proposed framework is illustrated in Figure~\ref{fig:maesn_arch}.

\subsection{Fractional Echo State Network (fESN)} \label{sec:fESN}
ESNs maintain an internal state according to Eqn.~\ref{eq:esn-update} that is updated by the input sequence. However, this state forgets past information rapidly. To address this limitation and enhance memory capacity, a dedicated state is introduced to retain input information over larger time horizons. A fractional filter for a univariate series $u(t)$ is defined by
$ F(u(t); d) = ((1 - B)^d - 1)u(t)$ where $d \in (0,1/2)$ is the fractional-differencing parameter. This filter is commonly used to incorporate long memory in time series analysis \citep{GrangerJoyeux1980, zhao2020rnn, yang2025memory}. Due to its ability to retain information from the distant past, it can be interpreted as a memory filter and is particularly suitable for recurrent processes. The filter acts as a soft-attention mechanism on past inputs, in which the attention weights decay polynomially. The soft-attention weights $\omega_k(d)$ are given by the differencing operator $(1-B)^d$, which can be expressed through its binomial expansion as 
\begin{equation} \label{eq:soft_weight}
(1 - B)^d = \sum_{k=0}^{\infty} \omega_k(d) B^k, \quad \omega_k(d) = \frac{\Gamma(-d+k)}{k!\Gamma(-d)},
\end{equation}
where $\Gamma(\cdot)$ denotes the gamma function. Since the series is infinite, in practice, we often neglect the weights for larger $k$ by setting a cutoff at some $k=K$. Hence, we define the cutoff-based truncated memory filter $F_{K}(u(t); d)$, whose soft weights are defined by
\begin{equation*}
    \omega_k(d) = \begin{cases}
        \frac{\Gamma(-d+k)}{k!\Gamma(-d)} & k \leq K\\
        0 & k > K.
    \end{cases}
\end{equation*}
We refer to $u(t)$ as the input and $F(u(t);d)$ as the memory input produced by the memory filter $F$ with fractional differencing parameter $d$. When necessary, we use its truncated version $F_{K}(u(t);d)$.

At time $t$, the network maintains $\mathbf{x}(t)\in\mathbb{R}^p$ (vanilla reservoir state) and $\mathbf{m}(t)\in\mathbb{R}^q$ (memory reservoir state). The main reservoir evolves from the raw input and its own past state:
\begin{equation}
\mathbf{x}(t) \;=\; \tanh\!\Bigl(\mathbf{W}_{\!xx}\,\mathbf{x}(t-1) + \mathbf{W}_{\!xu}\,u(t) + \boldsymbol{\eta}_x(t) \Bigr), \quad \boldsymbol{\eta}_x(t)\sim\mathcal{N}(\mathbf{0},\,\sigma_x^2\mathbf{I}_p)
\label{eq:fesn-main-update}
\end{equation}
where $\mathbf{W}_{\!xx}\in\mathbb{R}^{p\times p}$ is the recurrent weight matrix and $\mathbf{W}_{\!xu}\in\mathbb{R}^{p\times 1}$ the input weight vector (both are fixed after random initialization). The memory reservoir does not take $u(t)$ directly.  Instead, it receives a memory input that emulates fractional differencing of order $d\in(0,1/2)$.  For a window length $K$, define $f(t) \;=\; F_{K}(u(t);d)$ whose coefficients decay polynomially as $k^{-d-1}$ \citep{Hosking1981}. Thus, $f(t)$ retains information from the distant past, in contrast to exponential forgetting exhibited by the standard reservoir. The memory reservoir evolves as
\begin{equation}
\mathbf{m}(t) \;=\; \tanh\!\Bigl(\mathbf{W}_{\!mm}\,\mathbf{m}(t-1) + \mathbf{W}_{\!mf}\,f(t)+ \boldsymbol{\eta}_m(t) \Bigr), \quad \boldsymbol{\eta}_m(t)\sim\mathcal{N}(\mathbf{0},\,\sigma_m^2\mathbf{I}_q),
\label{eq:fesn-memory-update}
\end{equation}
where $\mathbf{W}_{\!mm}\in\mathbb{R}^{q\times q}$ is the recurrent weight matrix of the memory reservoir and $W_{mf}\in\mathbb{R}^{q\times 1}$ is a vector of the filter input weights. We define the input to the model as $\mathbf{u}'(t) = [u(t);f(t)]$ and states
as $\mathbf{x}'(t) = [\mathbf{m}(t)^\top;\mathbf{x}(t)^\top]$. Here $\mathbf{x}(t)$ is derived from input, as shown in Eqn.~\ref{eq:fesn-main-update}. The model output for the $h^{th}$ horizon is computed as:
\begin{align}
\label{eq:output_fesn}
 z_h(t) 
&= \mathbf{W}_{\!\mathrm{out}}^{(h)}\begin{bmatrix}
\mathbf{x}'(t)^\top;  \mathbf{u}'(t)^\top 
\end{bmatrix}^\top,
\qquad
\mathbf{W}_{\!\mathrm{out}}^{(h)} \in \mathbb{R}^{1 \times (p+q+2)},
\end{align} 
where $\mathbf{W}_{\!\mathrm{out}}^{(h)}$ denotes the output weight matrix corresponding to the $h^{th}$ horizon. The fESN therefore augments a standard ESN reservoir with a memory reservoir that is driven by a fractional–difference filter.  This architecture enables the network to capture both short–term nonlinear effects (via $\mathbf{x}(t)$) and persistent long–range dependencies (via $\mathbf{m}(t)$ and $f(t)$), while maintaining fixed internal weights after initialization. The loss at each time step, for predicting horizon h, is given by: $\ell_{h}(t) = \| z_h(t) - u(t + h)\|^2_2$, where $u(t+h)$ denotes the observed ground truth for horizon $h$.

\subsection{Wavelet Echo State Network (wESN)} \label{sec:wESN}

wESN extends the fESN by introducing a wavelet-based smoothing step before fractional memory filtering. The key idea is that raw input signals $u(t)$ may contain high-frequency noise or seasonal fluctuations that obscure the underlying long-range structure and nonstationarity. By first extracting a smooth approximation using the Maximal Overlap Discrete Wavelet Transform (MODWT) \citep{percival2000wavelet}, the model emphasizes persistent low-frequency components that are essential for capturing long memory.

The MODWT is a shift-invariant variant of the discrete wavelet transform (DWT) \citep{mallat2002theory}. Unlike the classical DWT, which downsamples after filtering and hence reduces signal length, the MODWT keeps all $T$ coefficients at each level. This ensures that transformed sequences remain aligned with the original decomposition time series. Let $\{\alpha(m)\}$ and $\{\xi(m)\}$ denote the scaling (low pass) and wavelet (high pass) filters of a chosen wavelet family respectively. At decomposition level $j$, these filters are dilated by inserting $2^{j-1}-1$ zeros between nonzero taps:
\begin{equation*}
 \alpha^{(j)}(m) =
\begin{cases}
\alpha(k), & m = 2^{j-1}k, \\[4pt]
0, & \text{otherwise},
\end{cases}
\qquad
\xi^{(j)}(m) =
\begin{cases}
\xi(k), & m = 2^{j-1}k, \\[4pt]
0, & \text{otherwise}.
\end{cases} 
\end{equation*}
Given the approximation coefficients $v^{(j-1)}(t)$ from the previous level 
(with $v^{(0)}(t)=u(t)$), the level--$j$ MODWT computes
\begin{equation*}
  v^{(j)}(t) = \sum_{m=0}^{L_j-1} \alpha^{(j)}(m)\;v^{(j-1)}(t-m \bmod T) 
\qquad \text{and} \quad
w^{(j)}(t) = \sum_{m=0}^{L_j-1} \xi^{(j)}(m)\;v^{(j-1)}(t-m \bmod T),  
\end{equation*}
where $v^{(j)}(t)$ are the \emph{smooth} coefficients, $w^{(j)}(t)$ are the 
\emph{detail} coefficients, and $L_j$ is the length of the dilated filter at level $j$. Both $v^{(j)}$ and $w^{(j)}$ retain length $T$, ensuring shift--invariance. For example, we consider the Haar filters
\begin{equation*}
   \alpha = \Bigl[\tfrac{1}{\sqrt2}, \tfrac{1}{\sqrt2}\Bigr],
\qquad
\xi = \Bigl[\tfrac{1}{\sqrt2}, -\tfrac{1}{\sqrt2}\Bigr]; 
\end{equation*}
and at level 1, the MODWT yields
\begin{equation*}
  v^{(1)}(t) = \tfrac1{\sqrt2}\,u(t) + \tfrac1{\sqrt{2}}\,u(t-1 \bmod T),  \qquad 
  w^{(1)}(t) = \tfrac1{\sqrt2}\,u(t) - \tfrac1{\sqrt2}\,u(t-1 \bmod T).
\end{equation*}
Thus, $v^{(1)}(t)$ is a local average (smoothed signal) and $w^{(1)}(t)$ captures local differences. In wESN, we retain only the smooth component, $s(t) = v^{(1)}(t)$ as the effective input.

The wESN employs two interacting reservoirs that are updated in parallel. The main reservoir models follow the rule of vanilla reservoir (See Eqn.~\ref{eq:esn-update}), while the memory reservoir captures long memory using the fractional filter:
\begin{equation}
f(t) \;=\; F_{K}(s(t);d).
\label{eq:fractional-filter}
\end{equation}
The essential difference between fESN and wESN is that the fractional memory in wESN is constructed from the smoothed wavelet signal $s(t)$ rather than the raw input $u(t)$. This ensures that long-range dependencies are estimated from low-frequency dynamics, while high-frequency noise or local fluctuations are suppressed.

\subsection{Weight Initialization and Training}  
\label{subsec:training}

Both fESN and wESN follow the mechanism of reservoir computing in which the recurrent and input weight matrices are randomly initialized once and then kept fixed throughout training. The model learns only the output weights, resulting in faster and more stable training compared to gradient-based recurrent architectures. This is also the reason ESNs and our proposed variants are lightweight models and effective in data-scarce situations. 

The reservoir matrix $\mathbf{W}_{\!xx}\in\mathbb{R}^{p\times p}$ is initialized with i.i.d.\ entries drawn from the uniform distribution $\mathcal{U}\!\left[-1/2,\,1/2\right]$ and then sparsified by randomly setting a fixed proportion $\varphi_x$ of its entries to zero. Finally, it is rescaled so that its spectral radius equals a prescribed value
\(\rho_x\in(0,1)\), i.e.,
\[
\mathbf{W}_{\!xx}\;\leftarrow\;\frac{\rho_x}{\rho(\mathbf{W}_{\!xx})}\,\mathbf{W}_{\!xx}.
\]
The input weights are
\(\mathbf{W}_{\!xu}\in\mathbb{R}^{p\times 1}\) with entries taken from \(\mathcal{U}[-1,1]\) and scaled by a factor $\psi_x$. For the memory reservoir, the recurrent weights \(\mathbf{W}_{\!mm}\in\mathbb{R}^{q\times q}\) are generated
analogously with entries drawn from \(\mathcal{U}[-1/2,1/2]\), sparsified at proportion \(\varphi_m\), and globally rescaled to
spectral radius \(\rho_m\in(0,1)\) by
$$
\mathbf{W}_{\!mm}\;\leftarrow\;\frac{\rho_m}{\rho(\mathbf{W}_{\!mm})}\,\mathbf{W}_{\!mm}.
$$
The filter–input weights are \(\mathbf{W}_{\!mf}\in\mathbb{R}^{q\times 1}\) with entries from
\(\mathcal{U}[-1,1]\) and scaled by a factor $\psi_m$. After running the reservoirs on the observed series, we form a single feature matrix shared across horizons. For each \(t=1, 2, 3, \dots, T\), define
\begin{equation*}
  \boldsymbol{\phi}(t)
=\begin{bmatrix}\mathbf{x}'(t)^\top \,; \mathbf{u}'(t)^\top\end{bmatrix}
\in\mathbb{R}^{1 \times d'}
\end{equation*}
with \(d'=p+q+2\). These vectors are stacked row-wise after discarding an initial set of states. It is a common practice in reservoir-based computing to eliminate transient states, so we define the ratio of the discarded states as the washout ratio, defined and denoted by $\zeta = \frac{T_0}{T}$, which is a tunable parameter to get $T_0 = {\zeta T}$. The stacked rows can be represented as follows:
\begin{equation*}
   \mathbf{Z}
=\begin{bmatrix}
\boldsymbol{\phi}(T_0);\boldsymbol{\phi}(T_0+1); \dots ; \boldsymbol{\phi}(T)\end{bmatrix}^\top
\in\mathbb{R}^{n\times d'},\quad n=T-T_0+1. 
\end{equation*}
For a forecasting horizon $h$, we collect the target values in
$
\mathbf{u}^{(h)}=\big[u({T_0} + h), u({T_0+1} + h), \dots,u({T} + h)\big]^\top\in\mathbb{R}^{n}.
$
For each \(h\in\{1,\dots,H\}\), we estimate a readout via ridge regression by solving the following:
\begin{equation*}
   \widehat{\mathbf{W}}_{out}^{(h)}
=\arg\min_{\mathbf{W}\in\mathbb{R}^{d'}}
\ \bigl\|\mathbf{u}^{(h)}-\mathbf{Z}\mathbf{W}\bigr\|_2^2
\;+\;\lambda\,\|\mathbf{W}\|_2^2, 
\end{equation*}
where $\lambda>0$ is the regularization parameter. The estimated output weights can be transposed to match the dimension for Eqn.~\ref{eq:output_fesn}. The first-order optimality conditions give the normal equations
\begin{equation*}
    \Bigl(\mathbf{Z}^\top\mathbf{Z}+\lambda\,\mathbf{I}_{d'}\Bigr)\,
\widehat{\mathbf{W}}^{(h)}
=\mathbf{Z}^\top\mathbf{u}^{(h)}.
\end{equation*}
Let \(\mathbf{G}=\mathbf{Z}^\top\mathbf{Z}+\lambda\,\mathbf{I}_d\) and we compute
its Cholesky factorization \(\mathbf{G}=\mathbf{R}^\top\mathbf{R}\) with
\(\mathbf{R}\in\mathbb{R}^{d'\times d'}\) being upper–triangular. For each horizon \(h\), the ridge solution is obtained by solving the pair of triangular systems
\[
\mathbf{R}^\top\mathbf{b}^{(h)}=\mathbf{Z}^\top\mathbf{u}^{(h)},
\qquad
\mathbf{R}\,\widehat{\mathbf{W}}^{(h)}=\mathbf{b}^{(h)},
\]
which yields the unique ridge solution. Since matrix \(\mathbf{Z}\) (hence \(\mathbf{G}\) and \(\mathbf{R}\)) does not
depend on $h$, the same Cholesky factor \(\mathbf{R}\) is reused for all horizons, while the right–hand side \(\mathbf{Z}^\top\mathbf{u}^{(h)}\)
changes with \(h\).  
To select the regularization parameter $\lambda$, we employ leave–one–out cross–validation (LOOCV). In linear ridge regression, the LOOCV error can be computed in closed form without retraining the model multiple times, which makes hyperparameter tuning highly efficient \citep{golub1979generalized}. This classical statistical trick avoids expensive re-estimation loops and further contributes to the computational advantages of our architecture. The detailed discussion on the LOOCV trick is deferred to Appendix~\ref{app:loocv}.  

\section{Theoretical Properties}
\label{sec:asymptotics}
In this section, we analyze the asymptotic properties of the proposed models. We begin by defining long memory, geometric ergodicity, and long memory network processes. Then, we present theoretical guarantees showing that a closed-loop (generative) ESN defines a short-memory process, and that fESN and wESN exhibit statistical long-memory behavior. Recent studies have established that RNNs, LSTMs, and GRUs are short-memory under mild conditions by showing their geometric ergodicity property \citep{zhao2020rnn, yang2025memory}. Extending their results for ESNs is not immediate since the reservoir weights are randomized. For simplicity, we consider a standard ESN without a direct input link. The same setting is adopted for our proposed long memory ESNs, omitting direct connections from both the raw input and the memory input. Also, we consider a univariate input to the model, representing dengue incidence over time (as in the empirical evaluation), without any exogenous variables. We begin by defining the notion of statistical long memory from a time series perspective.

\begin{definition}[Long memory]\label{def:long_memory}
Let $\{U(t):t\in\mathbb{Z}\}$ be a weakly stationary univariate process with autocovariance
$\gamma_U(k)$ and spectral density
\begin{equation*}
f_U(\lambda)=\frac{1}{2\pi}\sum_{k=-\infty}^{\infty}\gamma_U(k)\,e^{-ik\lambda},
\qquad \lambda\in[-\pi,\pi].
\end{equation*}
We say that the process $\{U(t)\}$ exhibits \emph{long memory} if $\sum_{k=-\infty}^{\infty}|\gamma_U(k)|=\infty$, and
\emph{short memory} otherwise, i.e., if $0<\sum_{k=-\infty}^{\infty}|\gamma_U(k)|<\infty$.
Equivalently, as $|\lambda|\to 0$, $f_U(\lambda)\to\infty$ for long memory, while
$f_U(\lambda)\to c_0\in(0,\infty)$ for short memory processes.
\end{definition}

Among long-memory models, the ARFIMA($p,d,q$) specification is one of the most widely used \citep{GrangerJoyeux1980, Hosking1981}. An ARFIMA($p,d,q$) process $U(t)$ is given by
\begin{equation*}
(1-B)^d U(t) = V(t),
\end{equation*}
where $V(t)$ follows an ARMA($p,q$) process representing the short-run dependence structure, and $d$ denotes the fractional differencing parameter. %In practice, one usually assumes $-1/2 < d < 1/2$ so that $U(t)$ remains stationary and invertible.
As established by \citet{Hosking1981}, when $k \to \infty$, the coefficients $\{\omega_k(d)\}$ from Eqn.~\ref{eq:soft_weight} satisfy $\omega_k(d) \sim k^{-d-1}$, while the lag-$k$ autocorrelation behaves as $\rho_U(k) \sim k^{2d-1}$. Moreover, \cite{McLeodHipel1978} noted that for $0 < d < 1/2$, the ARFIMA($p,d,q$) process $U(t)$ displays hyperbolically decaying dependence with pronounced persistence, reflecting long-memory behavior. This gradual polynomial decay distinguishes long-memory dynamics from short-memory models such as ARMA, whose coefficients decline at an exponential rate.

\begin{definition}[Geometric Ergodicity] \label{def:geometric_ergodicity}
In the context of a temporal homogeneous Markov chain $\{U(t)\}$ with state space $(S,\mathcal{F})$, where $\mathcal{F}$ is countably generated, we say $\{U(t)\}$ is geometrically ergodic if there exists a constant $0 < \rho < 1$ and a probability measure $\pi$ on $\mathcal{F}$ such that for every $u \in S$,
\begin{equation*}\label{eq:geom_ergodic}
\rho^{-s}\,\bigl\lVert P^s(u,\cdot)\;-\;\pi \bigr\rVert_{TV} \;\longrightarrow\;0
\quad\text{as }s \to \infty,
\end{equation*}
where $P^s(u,A) = \Pr\bigl(U(t+s) \in A \mid U(t) = u\bigr)$ denotes the $s$‐step transition probability and $\lVert \cdot\rVert_{TV}$ is the total variation norm.
\end{definition}

For a geometrically ergodic process $\{U(t)\}$, the Harris theorem \citep{meyn2012markov} implies that its autocorrelation function $\rho_U(k)$ decays exponentially fast as 
$k \to \infty$, i.e.,
$\rho_U(k)\sim\rho^{\,k}, \; 0 < \rho < 1$. Consequently, the corresponding auto-covariances $\gamma_U(k)$ are absolutely summable, i.e., $\sum_{k=-\infty}^{\infty} |\gamma_U(k)| < \infty$, since $\mathrm{Var}(U(t)) < \infty$. This exponential decay of correlations ensures that any geometrically ergodic process exhibits short memory \citep{meyn2012markov}. Next, we define a long-memory network process which plays a central role in our study.
\begin{definition}[Long‐Memory Network Process]
Consider the process defined by
\begin{equation*}
    y(t)=\sum_{k=0}^\infty \mathbf{A}_k\,u(t - k)+\varepsilon(t),
\end{equation*}
where $\mathbf{A}_k = \{(\mathbf{A}_k)_{ij}\}$ are coefficient matrices and $\varepsilon(t)$ is a noise term. The process $\{y(t)\}$ is said to exhibit long memory if $\exists \ i, j$ such that
\begin{equation*}
   (\mathbf{A}_k)_{i,j} \;\sim\;k^{-d-1}
\quad\text{as }k\to\infty, 
\end{equation*}
for some memory parameter $d\in(0,1/2)$. In this case, $y(t)$ depends on inputs from the distant past with only polynomial decay in influence.
\end{definition}
To proceed with the theoretical analysis, we impose the following assumptions.
\begin{assumption}\label{ass:network1}
   The density of the noise process $\varepsilon(t)$ is continuous and strictly positive on $\mathbb{R}$. Moreover, for some integer $\kappa\ge 2$, $\mathbb{E}|\varepsilon(t)|^\kappa<\infty$.
\end{assumption}

\begin{assumption}\label{ass:esp}
   The spectral radius of each reservoir matrix (e.g., $\mathbf{W}_{xx}$ and $\mathbf{W}_{mm}$) is strictly less than $1$.
\end{assumption}

Under Assumption~\ref{ass:network1}, the noise process $\varepsilon(t)$ has a continuous positive density everywhere on $\mathbb{R}$ and  $\varepsilon(t)$ has finite $\kappa$-th moments for some $\kappa \ge 2$. These conditions ensure that the Markov chain induced by the network process (see Eqn.~\ref{eq:network_process}) is aperiodic and irreducible (roughly, the Markov chain can move between any two states in finite time with positive probability), which implies ergodicity. Such assumptions are common in Markov chain theory and allow the use of Harris' Theorem \citep{meyn2012markov}, which connects drift conditions to geometric ergodicity.

Assumption~\ref{ass:esp} is necessary for ESP, under the standard compactness condition introduced by \cite{jaeger2001echo}. ESP states that the infinite history of past inputs uniquely determines the reservoir state. Equivalently, no two distinct infinite input sequences can lead to the same reservoir state. When the state transition map is continuous in both the state and input, the properties of uniform state contraction, state forgetting, and input forgetting are all equivalent to the network possessing the ESP. 

% \begin{remark}
% \citet{jaeger2001echo} showed that the echo state property holds even when the spectral radius exceeds 1, provided that the input is persistently nonzero. This condition is not appropriate for dengue incidence data because during non-seasonal periods, the series often takes the value zero, violating the persistence requirement.
% \end{remark}

A recurrent network can be abstracted as a temporal homogeneous Markov chain, following \citet{zhao2020rnn}. Concretely, let Eqn.~\ref{eq:network_process} denote the data-generating process. Let $y(t)\in\mathbb{R}$ be the output of the process, which is assumed to be closed-loop. Hence, we have that the input at $t^{th}$ step is the output of $(t-1)^{th}$ step, i.e., $u(t) = y(t-1)$. Let $\mathbf{x}(t)\in\mathbb{R}^p$ be the hidden state and $\varepsilon(t)\in\mathbb{R}$ be i.i.d. noise at time $t$. Then, the system can be written as
\begin{equation*}
  \begin{bmatrix}
y(t)\\
\mathbf{x}(t)
\end{bmatrix}
=
\mathcal{M}\bigl(y(t-1),\,\mathbf{x}(t-1)\bigr)
+
\begin{bmatrix}
\varepsilon(t)\\
\mathbf{0}_p
\end{bmatrix},  
\end{equation*}
where $\mathcal{M}:\mathbb{R}^{p+1}\to\mathbb{R}^{p+1}$ is the state‐transition function and $\mathbf{0}_p$ is the $p$-dimensional zero vector. This representation is called a recurrent network process. A linear echo state network process can be defined by taking the output and activation function in Eqn.~\ref{eq:esn-readout} and~\ref{eq:esn-update} as identity, and is given by
\begin{equation} \label{eq:closed_loop_esn_linear}
    \mathbf{x}(t)=\mathbf{W}_{\!xx}\,x(t-1)+\mathbf{W}_{\!xy}y(t-1) + \eta(t),\qquad
    y(t)=\mathbf{W}_{\!out}\,x(t)+\varepsilon(t),
\end{equation} 
where $\mathbf{W}_{\!xx} \in \mathbb{R}^{p \times p}$, $\mathbf{W}_{\!xy} \in \mathbb{R}^{p \times 1}$ and $\mathbf{W}_{\!out} \in \mathbb{R}^{1\times p}$. The state and output evolve continuously as
% \begin{equation}
% \label{eq:esn-linear-block}
% \begin{bmatrix}y(t)\\[2pt]\mathbf{x}(t)\end{bmatrix}
% =
% \underbrace{\begin{bmatrix}
% \mathbf{W}_{out}\mathbf{W}_{\!xy}, & \mathbf{W}_{\mathrm{out}}\mathbf{W}_{\!xx} \\
% \mathbf{W}_{\!xy}, & \mathbf{W}_{\!xx}\\
% \end{bmatrix}}_{\mathbf{M}_{\mathrm{ESN}}}
% \begin{bmatrix}y(t-1)\\[2pt]\mathbf{x}(t-1)\end{bmatrix}
% +\begin{bmatrix}\varepsilon(t)\\ \mathbf{0}_p\end{bmatrix},
% \end{equation}
% If the matrix $\mathbf{M}_{ESN}$ has spectral radius less than 1, then the ESN process is geometrically ergodic and hence exhibits short memory. The result follows directly from Theorem 2 of \citet{zhao2020rnn}. The following lemma proves that the echo state network is derived by $\mathbf W_{xx}+\mathbf W_{xy}\mathbf W_{\mathrm{out}}$.
\begin{equation}
\label{eq:esn-linear-block}
\begin{bmatrix}y(t)\\[2pt]\mathbf{x}(t)\end{bmatrix}
=
\underbrace{\begin{bmatrix}
\mathbf{W}_{out}\mathbf{W}_{\!xy}, & \mathbf{W}_{\mathrm{out}}\mathbf{W}_{\!xx} \\
\mathbf{W}_{\!xy}, & \mathbf{W}_{\!xx}\\
\end{bmatrix}}_{\mathbf{M}_{\mathrm{ESN}}}
\begin{bmatrix}y(t-1)\\[2pt]\mathbf{x}(t-1)\end{bmatrix}
+\begin{bmatrix}\varepsilon(t)\\ \eta(t)\end{bmatrix}.
\end{equation}
If the matrix $\mathbf{M}_{ESN}$ has spectral radius less than 1, then the ESN process is geometrically ergodic and hence exhibits short memory. The result follows directly from Theorem 2 of \citet{zhao2020rnn}. 
%The following lemma proves that the echo state network is derived by $\mathbf W_{xx}+\mathbf W_{xy}\mathbf W_{\mathrm{out}}$.

\begin{lemma}[Spectral radius of the ESN block]\label{lem:block-spectrum}
Let $p\in\mathbb N$, $\mathbf W_{\mathrm{out}}\in\mathbb R^{1\times p}$, $\mathbf W_{xy}\in\mathbb R^{p\times 1}$, and $\mathbf W_{xx}\in\mathbb R^{p\times p}$. Define
\[
\mathbf M_{\mathrm{ESN}}=
\begin{bmatrix}
\mathbf W_{\mathrm{out}}\mathbf W_{xy} & \mathbf W_{\mathrm{out}}\mathbf W_{xx} \\
\mathbf W_{xy} & \mathbf W_{xx}
\end{bmatrix},
\qquad
\mathbf F=\mathbf W_{xx}+\mathbf W_{xy}\mathbf W_{\mathrm{out}}\in\mathbb R^{p\times p}.
\]
Then, $\mathbf M_{\mathrm{ESN}}$ is similar to
\(
\begin{bmatrix}
\mathbf F & \mathbf W_{xy}\\
\mathbf 0 & 0
\end{bmatrix}.
\)
Consequently, the eigenvalues of  $M_{\mathrm{ESN}}$ are those of $\mathbf F\text{ together with }0$ and, in particular,
$\rho(\mathbf M_{\mathrm{ESN}})=\rho(\mathbf F).$
\end{lemma}

Lemma~\ref{lem:block-spectrum} implies that the dynamics of the ESN are entirely determined by $\mathbf{F}$. In the following, we show that for the ESP to hold, the spectral radius
of $F$ must be less than 1.

% \begin{lemma}[ESP of closed loop ESN]\label{lem:esp-iff}
% Let $p\in\mathbb N$. Consider the closed-loop \emph{linear} ESN with univariate readout $y(t)\in\mathbb R$ and reservoir state $\mathbf{x}(t)\in\mathbb R^p$ as given in Eqn.~\ref{eq:closed_loop_esn_linear}. Define closed-loop state matrix $F$ as in Lemma~\ref{lem:block-spectrum} and fix an arbitrary noise path $\{\varepsilon(t)\}_{t\ge 0}$ applied to the initial states $\mathbf{x}_1(0)$ and $\mathbf{x}_2(0)$. Define the differences $\Delta\mathbf{x}(t)=\mathbf{x}_1(t)-\mathbf{x}_2(t)$ and $\Delta y(t)=y_1(t)-y_2(t)$. Then, we have $
% \Delta\mathbf{x}(t)=\mathbf{F}^{\,t}\,\Delta\mathbf{x}(0)$ and $\Delta y(t)=\mathbf{W}_{\mathrm{out}}\,\Delta\mathbf{x}(t)$. Then, the following are equivalent:
% \begin{enumerate}
% \item[(i)] \emph{(ESP)} For the fixed noise path $\{\varepsilon(t)\}$ and all initial states,
% $\Delta\mathbf{x}(t)\to \mathbf{0}$ and $\Delta y(t)\to 0$ as $t\to\infty$.
% \item[(ii)] All eigenvalues of $F$ lie strictly inside the unit disc, i.e., $\rho(F) < 1.$
% \end{enumerate}
% In particular, by Lemma~\ref{lem:block-spectrum}, $\rho(\mathbf{M}_{\mathrm{ESN}})=\rho(\mathbf{F})$, so ESP holds if and only if $\rho(\mathbf{M}_{\mathrm{ESN}})<1$.
% \end{lemma}
\begin{lemma}[ESP of closed loop ESN]\label{lem:esp-iff}
Let $p\in\mathbb N$, consider the closed-loop \emph{linear} ESN with univariate readout $y(t)\in\mathbb R$ and reservoir state $\mathbf{x}(t)\in\mathbb R^p$ as given in Eqn.~\ref{eq:closed_loop_esn_linear}. Define closed-loop state matrix $F$ as in Lemma~\ref{lem:block-spectrum} and fix arbitrary noise paths $\{\varepsilon(t)\}_{t\ge 0}$ and $\{\eta(t)\}_{t\ge 0}$ applied to the initial states $\mathbf{x}_1(0)$ and $\mathbf{x}_2(0)$. Define the differences $\Delta\mathbf{x}(t)=\mathbf{x}_1(t)-\mathbf{x}_2(t)$ and $\Delta y(t)=y_1(t)-y_2(t)$. Then, we have $
\Delta\mathbf{x}(t)=\mathbf{F}^{\,t}\,\Delta\mathbf{x}(0)$ and $\Delta y(t)=\mathbf{W}_{\mathrm{out}}\,\Delta\mathbf{x}(t)$. Then, the following are equivalent:
\begin{enumerate}
\item[(i)] \emph{(ESP)} For the fixed noise paths and all initial states,
$\Delta\mathbf{x}(t)\to \mathbf{0}$ and $\Delta y(t)\to 0$ as $t\to\infty$.
\item[(ii)] All eigenvalues of $F$ lie strictly inside the unit disc, i.e., $\rho(F) < 1.$
\end{enumerate}
In particular, by Lemma~\ref{lem:block-spectrum}, $\rho(\mathbf{M}_{\mathrm{ESN}})=\rho(\mathbf{F})$, so ESP holds if and only if $\rho(\mathbf{M}_{\mathrm{ESN}})<1$.
\end{lemma}

ESP asserts that under a fixed driving noise, the long-run behavior of the network is independent of the initial condition. For a linear closed-loop ESN, the difference between two trajectories driven by the same noise evolves autonomously as a homogeneous linear system. Thus, the ESP is equivalent to asymptotic stability of the system, i.e., its state-transition matrix has spectral radius $<1$.

Assumption~\ref{ass:esp} controls the open-loop reservoir ($\rho(\mathbf{W}_{xx})<1$), which is sufficient for the ESP in the input-driven case. Closing the loop changes the
effective state matrix to $\mathbf{F}=\mathbf{W}_{\!xx}+\mathbf{W}_{\!xy}\mathbf{W}_{\!\mathrm{out}}$. In closed loop, the ESP is exactly equivalent to
$\rho(\mathbf{F})<1$, and by Lemma~\ref{lem:block-spectrum}, $\rho(\mathbf{M}_{\mathrm{ESN}})<1$.
Thus, using $\rho(\mathbf{M}_{\mathrm{ESN}})<1$ (or $\rho(\mathbf{F})<1$) is a justified and necessary stability requirement for the closed ESN.

We now show that if $\mathbf{M}_{\mathrm{ESN}}$ has a spectral radius strictly less than 1, linear ESN is geometrically ergodic and thus possesses short memory.

\begin{theorem}\label{thm:esn-geo-ergodic}
Consider the closed ESN in Eqn.~\ref{eq:esn-linear-block} and suppose that $
\rho\!\left(\mathbf{M}_{\mathrm{ESN}}\right)<1$. Under Assumption~\ref{ass:network1},
the Markov chain
\begin{equation*}
   \mathbf{Y}(t)=\big(y(t);\,\mathbf{x}(t)^\top\big)^\top \in \mathbb{R}^{1+p} 
\end{equation*}
is geometrically ergodic.
\end{theorem}

Now, we extend this result to a nonlinear ESN with a continuous and bounded activation function.

\begin{theorem} \label{thm:esn_nonlinear}
Suppose the output and activation functions $g(\cdot)$ and $\tau(\cdot)$ respectively in Eqn.~\ref{eq:esn-readout}-~\ref{eq:esn-update} are continuous and bounded. If Assumption~\ref{ass:network1} holds, then the ESN process is geometrically ergodic and has short memory.
\end{theorem}

Now, we show that ESN cannot exhibit a long memory network process.

\begin{theorem}[Geometric tail of a vanilla ESN]\label{thm:vanilla-no-lm}
Under Assumption~\ref{ass:esp}, a linear ESN process cannot generate a long-memory network process according to Definition~\ref{def:long_memory}. In particular, for any input $u(t)$ with $\mathbb{E}|u(t)|^2<\infty$, the output is a linear time-invariant filter with absolutely summable impulse response.
\end{theorem}

Next, we will demonstrate that fESN exhibits a long memory network process. Intuitively, the reservoir paths provide exponentially decaying kernels and convolving them with the fractional weights ${\omega_k(d)}$ preserves the regularly varying polynomial tail $k^{-d-1}$. The preliminary finite moving average $s(t)$ simply rescales the tail.

\begin{theorem}[fESN has polynomial impulse tail]
\label{thm:fesn-lm}
Let $d\in(0,1/2)$ and consider the fESN process
\begin{equation*}
  \mathbf{x}(t)=\mathbf{W}_{xx}\mathbf{x}(t-1)+\mathbf{W}_{xu}\,u(t) + \eta_x(t),\qquad
\mathbf{m}(t)=\mathbf{W}_{mm}\mathbf{m}(t-1)+\mathbf{W}_{mf}\,\big((1-B)^d-1\big)u(t) + \eta_m(t).
\end{equation*}
Under Assumptions~\ref{ass:network1} and ~\ref{ass:esp}, $\mathbf{W}_{\mathrm{out},x}\,\mathbf{x}(t)+\mathbf{W}_{\mathrm{out},m}\,\mathbf{m}(t) + \varepsilon(t)$ admits the expansion $\sum_{k\ge0} A_k\,u(t-k)$, where  $A_k=C_k+D_k$
with $\|C_k\|=O(\alpha^k)$ for some $\alpha\in(0,1)$ and $\|D_k\|=O(k^{-(d+1)})$. Consequently, fESN is a long memory network process.
\end{theorem}

Finally, we show that wESN can also generate long-memory behavior. 

\begin{theorem}
\label{thm:wesn-lm}
Under Assumption~\ref{ass:network1} and Theorem~\ref{thm:fesn-lm} holds, finite–length smooth \(s(t)=\sum_r g_r u(t-r)\), and \(d\in(0,\tfrac12)\), the coefficients of the wESN input–output expansion satisfy \(A_k\sim \widetilde K\,k^{-\tilde{d}-1}\). Hence, wESN realizes long memory.
\end{theorem}

Theorems~\ref{thm:esn-geo-ergodic}--\ref{thm:wesn-lm} establish the central theoretical contrast of this paper: standard ESNs are stable but inherently short-memory, whereas the proposed fESN and wESN architectures introduce polynomially decaying impulse responses consistent with statistical long memory. Thus, explicit memory reservoirs in fESN and wESN provide a mechanism that enables reservoir computing to model persistent dengue dynamics. Section~\ref{sec:simulation} shows the simulation study to provide a controlled diagnostic of the memory mechanism described in Theorems~\ref{thm:esn-geo-ergodic}--\ref{thm:wesn-lm}. All proofs are provided in Appendix~\ref{app:proofs}.

\section{Experimental Evaluation} \label{sec:exp}
This section evaluates the proposed fESN and wESN models on benchmark dengue forecasting datasets \citep{clarke2024global, panja2023epicasting}, whose lengths range from 108 to 1196 observations owing to monthly or weekly reporting frequencies. These datasets exhibit long-range dependence and limited sample sizes, making them suitable testbeds for data-scarce epidemic forecasting. We compare the proposed models against state-of-the-art statistical, neural, and reservoir-based benchmarks, followed by statistical significance testing, an ablation study, and uncertainty quantification for probabilistic forecasting.

\subsection{Dataset} \label{subsec:dataset}
In this study, we consider nine dengue incidence datasets from Asia, South America, and North America, thereby capturing a wide spectrum of climatic conditions and transmission environments. The Asian locations include major urban centers such as Bangkok, Ahmedabad (India), Hong Kong, Singapore, and the Philippines, which contain well-established dengue surveillance systems. Among the regions in South America, we consider the dengue dataset from Iquitos (Peru), Colombia, and Venezuela. In contrast, the North American region includes dengue cases from the hyperendemic area of San Juan (Puerto Rico). Each of these datasets differs in both temporal length and resolution. The Philippines dataset, for example, contains monthly dengue cases aggregated across all 17 administrative regions and expressed per 100,000 population, spanning from 2008 to 2016, whereas the San Juan dataset consists of weekly observations from 1990 to 2013. Most datasets cover a temporal horizon of approximately 10-15 years, providing sufficient duration to capture multiple epidemic cycles as well as the influence of multi-year climate oscillations such as El Niño and La Niña. Overall, six datasets provide dengue incidence counts in weekly frequency (424-1,196 observations), enabling the detection of rapid epidemic fluctuations, while the remaining three provide monthly data.

To analyze the global characteristics of the dengue datasets, we examine several key statistical features, including stationarity, nonlinearity, seasonality, and long-range dependence. The stationarity of a time series, indicating that its statistical properties remain constant over time, is assessed using the Kwiatkowski-Phillips-Schmidt-Shin (KPSS) test, which evaluates the null hypothesis that the series is stationary \citep{tseries}. Nonlinearity, a fundamental characteristic of a time series that influences the choice of an appropriate modeling framework, is determined using Terasvirta’s neural network test, which considers the null hypothesis that the observed series is linear \citep{terasvirta1993power}. Seasonality, reflecting recurrent temporal patterns, is detected using the Ollech and Webel combined test \citep{seastests}. To evaluate long-range dependence, we estimate the fractional differencing parameter $d$ using the Whittle estimator implemented in the `WhittleEst' function of the \texttt{longmemo} package \citep{longmemo}. The results of these statistical evaluations, summarized in Table~\ref{tab:epidemic_summary}, indicate that all dengue incidence datasets exhibit statistically significant long-range dependence, with estimated $d$ values ranging from $0.34$ to $0.49$ and corresponding $p$-values satisfying $p < 2.22 \times 10^{-16}$; the datasets are mostly non-stationary, except for the Iquitos and San Juan datasets. The p-values from Teraesvirta’s test suggest that all datasets are nonlinear except those from Hong Kong, Singapore, and Venezuela. Additionally, the datasets from Colombia, Hong Kong, Iquitos, and Venezuela display seasonal fluctuations at varying temporal scales. We additionally report the corrected rescaled-range estimate of the Hurst exponent, $H^{'}$, computed from the raw series using the \texttt{hurstexp} function of the \texttt{pracma} package \citep{pracma}, which suggests long-memory behavior for all datasets.

\begin{table}[ht]
\centering
\caption{Global characteristics of dengue datasets. In the table, \cmark indicates the presence of a feature, while \xmark $\;$ indicates its absence.}
\label{tab:epidemic_summary}
\begin{adjustbox}{max width=\textwidth}
\begin{tabular}{lllrcccccc}
\toprule
Dataset & Time span & Frequency & Length  & $d$ & $H^{'}$ & Long-term dependence & Non-stationary & Seasonal & Nonlinear  \\
\midrule
Ahmedabad    & 2005-2012 & Weekly  &  424 & 0.49 & 0.73 & \cmark & \cmark & \xmark & \cmark  \\
Bangkok      & 2003-2017 & Monthly &  180 & 0.49 & 0.70 &\cmark & \cmark & \xmark & \cmark  \\
Colombia     & 2005-2016 & Weekly  &  626 & 0.49 & 1.00 &\cmark & \cmark & \cmark & \cmark  \\
Hong Kong    & 2002-2017  & Monthly  & 192  & 0.34 & 0.89 &\cmark & \cmark & \cmark & \xmark \\
Iquitos      & 2002-2013 & Weekly  &  598 & 0.49 &  0.67 &\cmark & \xmark & \cmark & \cmark  \\
Philippines  & 2008-2016 & Monthly &  108 & 0.40 &  1.02  &\cmark & \cmark & \xmark & \cmark  \\
San Juan     & 1990-2013 & Weekly  & 1196 & 0.49 &  0.82  &\cmark & \xmark & \xmark & \cmark  \\
Singapore   & 2000-2015  & Weekly  & 838  & 0.49 &  0.92  &\cmark & \cmark & \xmark & \xmark \\
Venezuela    & 2002-2014 & Weekly  &  660 & 0.49 &  0.89  &\cmark & \cmark & \cmark & \xmark  \\
\bottomrule
\end{tabular}
\end{adjustbox}
\end{table}

\subsection{Baseline Models} \label{sec:baseline}

To evaluate the forecasting performance of the proposed fESN and wESN architectures, we compare them against a diverse collection of baseline statistical methods, state-of-the-art neural architectures, and long-memory models. Among the statistical frameworks, we consider the modified Exponential Smoothing (ETS) \citep{hyndman2008forecasting} and ARFIMA models \citep{GrangerJoyeux1980}, which capture long-term linear dependence through their fractional differencing. The deep learning techniques include MLP-based models such as Neural Basis Expansion Analysis for Time Series (N-BEATS) \citep{oreshkin2019n} and Time Series Mixer (TSMixer) and recurrent architectures such as, vanilla RNN \citep{rumelhart1986learning}, LSTM \citep{hochreiter1997long} and GRU \citep{cho2014learning}. In addition to the base models, we consider their memory‐augmented deep learning counterparts such as MRNNF, MLSTMF \citep{zhao2020rnn}, and MGRUF \citep{yang2025memory}, which explicitly incorporate long‐range information via fractional or gating‐based mechanisms with a non-dynamic differencing parameter. Our evaluation also includes attention-based transformer models, namely the Temporal Fusion Transformer (TFT) \citep{lim2021temporal}, which dynamically weights historical information across multiple time scales, and the Wavelet Transformer (wTransformer) \citep{sasal2022transformers}, which applies MODWT-based decomposition to model long-term dynamics. Additionally, we analyze the benefits of the memory-augmented mechanism in the proposed models by comparing their performance with the (Vanilla) base ESN model \citep{jaeger2001echo}. Overall, the baseline models span the spectrum from linear statistical models to deep, recurrent, and attention‐based approaches, providing a comprehensive evaluation of our proposed models.

\subsection{Experimental Setup} \label{sec:exp_setup}
To validate the generalizability of the proposed fESN and wESN frameworks, we perform rolling-window forecast evaluations across three distinct forecast horizons (short-term, medium-term, and long-term). During training, models use direct multi-step-ahead forecasting, in which the model observes the previous $L$ values to predict the next $H$ values, with the look-back window $L$ set equal to the forecast horizon $H$. For the weekly dataset, the forecast horizon $H$ spans 13, 26, and 52 weeks for short, medium, and long-term periods, while for the monthly dataset it spans 3, 6, and 12 months, respectively. In addition, hyperparameter tuning for the proposed frameworks is performed via Bayesian optimization (described in Appendix~\ref{subsec:hyper}) using a validation set of length $2H$  for each dataset. To ensure reliable performance estimates, each forecasting model is executed five times, accounting for the stochastic nature of model training. Furthermore, we assess the performance of the proposed models and benchmark forecasters using four key evaluation metrics, namely Mean Absolute Error (MAE), Mean Absolute Scaled Error (MASE), Root Mean Square Error (RMSE), and symmetric Mean Absolute Percentage Error (sMAPE) \citep{hyndman2018forecasting}. MAE, a scale-dependent measure, is robust to outliers, whereas the squared-error-based RMSE penalizes large deviations more heavily. The scale-independent MASE facilitates comparison across series, with values below 1 indicating improvement over a naive benchmark. The percentage-based sMAPE enables comparison across datasets with differing scales. The mathematical formulation of these metrics is discussed in Appendix~\ref{subsec:metric}. By definition, lower values of these metrics indicate better forecasting performance.

\subsection{Empirical Results} \label{sec:result}

\begin{figure}[ht!]
    \centering
    \includegraphics[width=\linewidth]{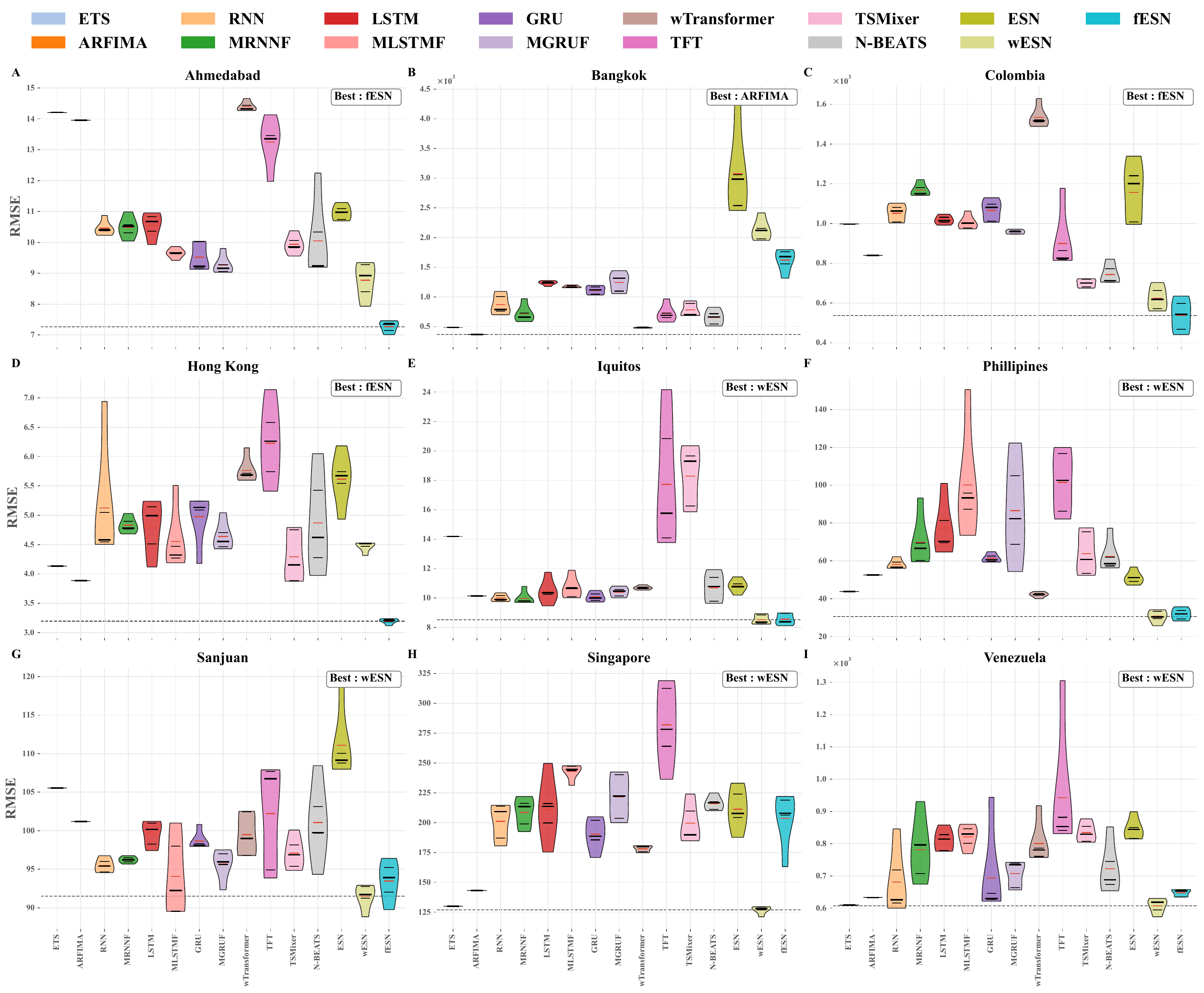}
    \caption{Box-plot comparison of long-term forecast performance (using RMSE) for the proposed and benchmark models across dengue datasets. In each box, the red horizontal line denotes the mean, the thick black horizontal line denotes the median, and the two thin black horizontal lines denote the first and third quartiles.}
    \label{fig:long_rmse}
\end{figure}
The long-term forecasting performance, summarized in Figure~\ref{fig:long_rmse} using the RMSE metric, shows that the proposed fESN and wESN models achieve the best results across most epidemic datasets. For the Ahmedabad series, fESN attains the lowest median RMSE, followed by wESN. In contrast, for the Iquitos, San Juan, and Philippines datasets, wESN most accurately captures the long-term dependencies inherent in dengue incidence dynamics. For the non-stationary dengue cases of Singapore and Venezuela, wESN outperforms the benchmark forecasters, whereas in the Colombia and Hong Kong datasets, fESN delivers competitive performance alongside wESN. Although ESN and its memory-augmented variants offer advantages such as efficient training, low computational cost, and stable reservoir dynamics, their performance can deteriorate when historical observations are scarce with sudden peaks. This limitation is particularly evident for the Bangkok dataset. Its small sample size enables the ARFIMA model, designed to capture long-range dependence with limited data, to outperform reservoir-based approaches. The training mechanism of the proposed models is consistent with that of ESN, where a portion of the initial state trajectory is discarded according to the washout ratio $\zeta$, and the readout matrix $\mathbf{W}_{\!\text{out}}$ is estimated only from the remaining observations. For the Bangkok series, a sharp increase in dengue incidence during the validation period leads to a high $\zeta$ value, further reducing the effective sample available for estimating $\mathbf{W}_{\!\text{out}}$. The challenges posed by such low-sample scenarios, including Bangkok and the Philippines, are discussed in detail in Appendix~\ref{para:results:para_limited_data}. 
Overall, the empirical results across all forecast horizons and evaluation metrics, summarized in Appendix~\ref{app:results}, indicate that the proposed fESN and wESN frameworks consistently deliver the best performance across most forecasting tasks. An important observation is that memory-augmented deep recurrent models, such as MRNN, MLSTM, and MGRU variants, do not always translate their theoretical long-memory advantage into superior forecasting performance under data scarcity. In several cases, the simpler linear ARFIMA model remains competitive because it explicitly captures long-range dependence with fewer trainable parameters. The proposed fESN and wESN models provide a compromise between these two regimes: they retain the data efficiency and computational simplicity of reservoir computing, while introducing explicit long-memory mechanisms capable of modeling nonlinear epidemic dynamics.

To assess the statistical significance of the observed performance improvements, we conduct the non-parametric Multiple Comparisons with the Best (MCB) test at the $95\%$ confidence level \citep{makridakis2022m5}, which evaluates the average ranks of all competing models based on their error measurements and computes the corresponding critical distances \citep{nemenyi1963distribution}. The results of the MCB test, presented in Figure~\ref{fig:mcb-all-metrics} for RMSE, MASE, and sMAPE, show that the fESN framework attains the lowest average rank across all metrics, implying its superiority. Its critical distance, therefore, defines the reference value (shaded area) of the test. The wESN framework provides closely comparable performance, achieving the second-lowest rank. Among the baseline models, ARFIMA ranks third, consistently following the proposed frameworks, reflecting its ability to capture long-term dependence even in limited-sample settings. ESN ranks fourth, followed by the ETS and TSMixer models. Although ARFIMA, ESN, ETS, and TSMixer do not exhibit statistically significant differences from fESN according to the MCB test, their average ranks remain higher than those of the proposed frameworks. For most other baseline models, their critical distances lie well above the reference threshold, indicating that their performance differs significantly from that of fESN and wESN.

Additionally, we perform the Diebold-Mariano (DM) test for the proposed fESN and wESN models \citep{diebold2002comparing}. This distribution-free test checks the null hypothesis that two competing forecasters have equal predictive accuracy. The results of this test for a long-term horizon, reported in Tables \ref{tab:dm_hlong_fesn} and \ref{tab:dm_hlong_wesn} for fESN and wESN models, respectively, indicate that the observed differences are statistically significant in most cases except between ESN and fESN for the Philippines dengue dataset. This is primarily attributed to the limited sample size. However, for wESN, the difference is statistically significant, and its performance is also better than that of ESN. The DM test result over the other two horizons, discussed in Appendix~\ref{app:results}, demonstrates similar patterns.

\begin{figure}[ht!]
  \centering
  \begin{subfigure}[b]{0.32\textwidth}
    \includegraphics[width=\linewidth]{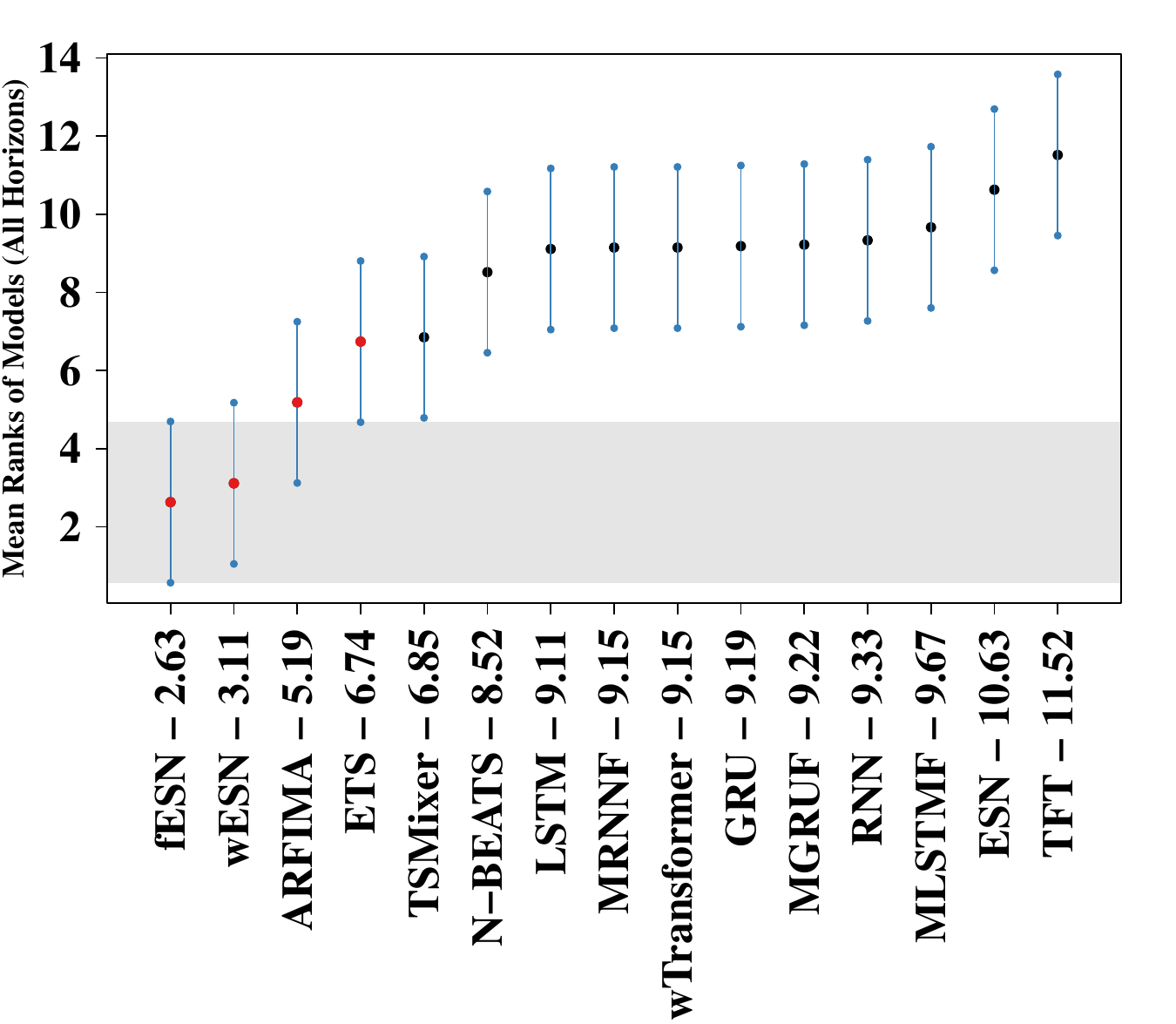}
    \caption{RMSE}
    \label{subfig:mcb-rmse}
  \end{subfigure}
  \hfill
  \begin{subfigure}[b]{0.32\textwidth}
    \includegraphics[width=\linewidth]{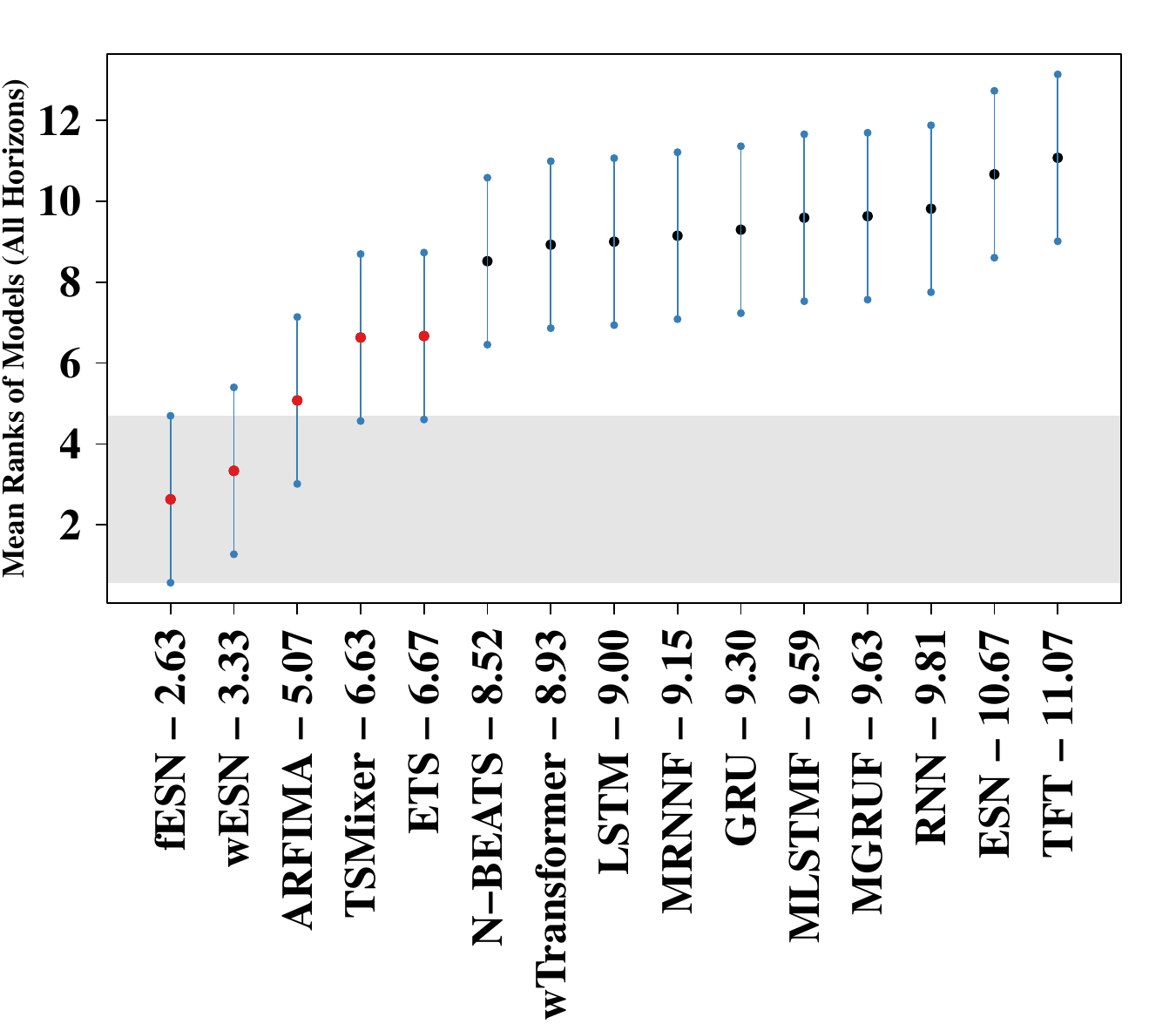}
    \caption{MASE}
    \label{subfig:mcb-mae}
  \end{subfigure}
\hfill
  \begin{subfigure}[b]{0.32\textwidth}
    \includegraphics[width=\linewidth]{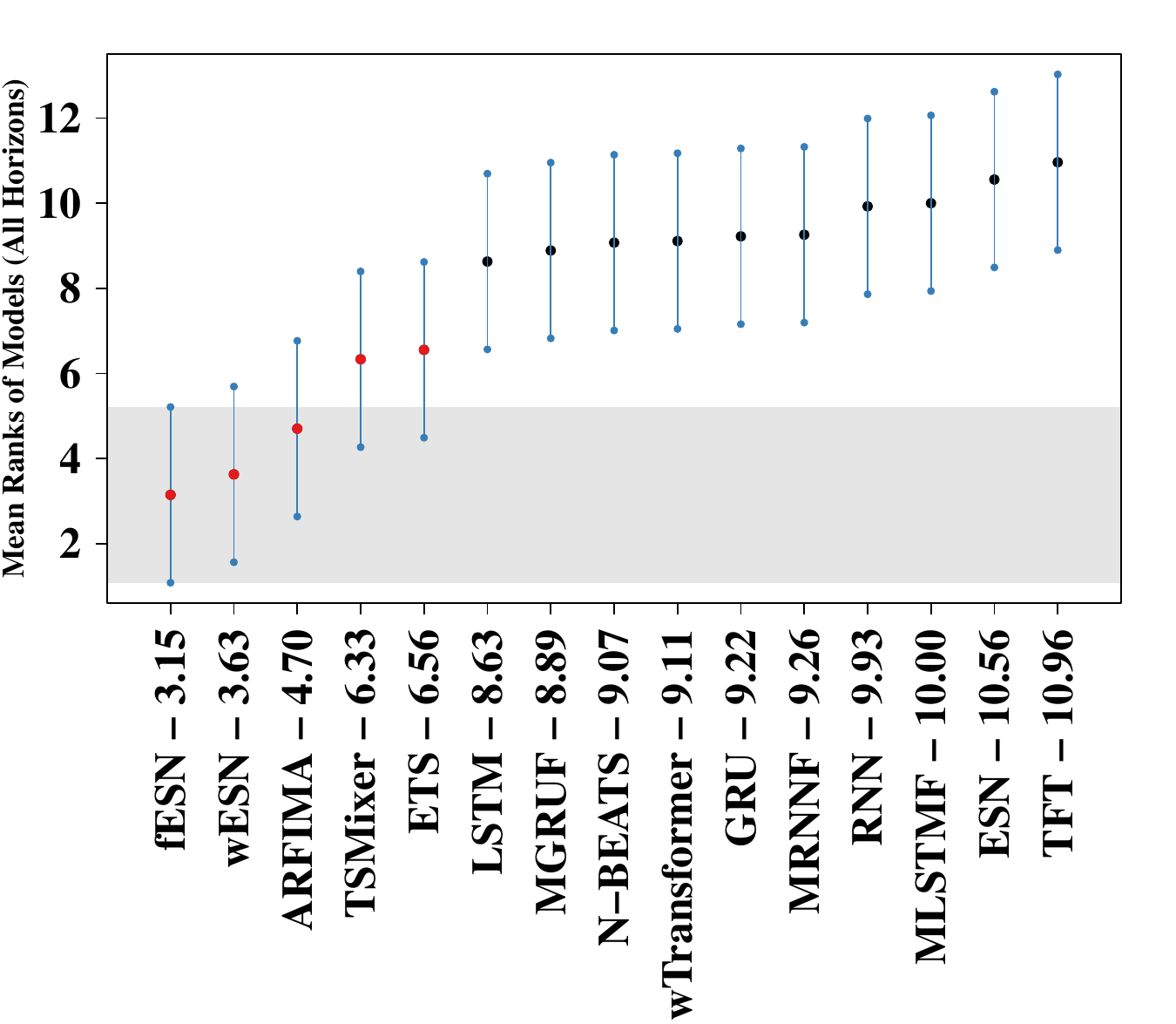}
    \caption{sMAPE}
    \label{subfig:mcb-smape}
  \end{subfigure}
  \caption{MCB test results comparing the proposed models with baseline approaches using (a) RMSE, (b) MASE, and (c) sMAPE for dengue incidence forecasting. The Y-axis shows the average ranks of the models, while the X-axis provides each model with its corresponding average rank. For example, the label `fESN - 2.63' indicates that the fESN model obtained an average rank of 2.63 based on the RMSE metric. 
}
  \label{fig:mcb-all-metrics}
\end{figure}

\begin{table}[ht!]
  \centering
  \caption{DM test p-values comparing fESN against baseline models across datasets for long-term forecasting. Significance levels: \textbf{p<0.05}, \underline{p<0.10}.}
  \label{tab:dm_hlong_fesn}
  \begin{adjustbox}{width=\textwidth,center}
  \begin{tabular}{@{\extracolsep{5pt}}l|ccccccccccccc}
  \toprule
  Dataset & ETS & ARFIMA & RNN & MRNN & LSTM & MLSTM & GRU & MGRU & wTransformer & TFT & TSMixer & N-BEATS & ESN \\
  \midrule
Ahmedabad & \textbf{\textless 0.001} & \textbf{0.002} & \textbf{\textless 0.001} & \textbf{0.001} & \textbf{\textless 0.001} & \textbf{0.002} & \textbf{\textless 0.001} & \textbf{\textless 0.001} & \textbf{\textless 0.001} & \textbf{\textless 0.001} & \textbf{\textless 0.001} & \textbf{\textless 0.001} & \textbf{\textless 0.030} \\
Bangkok & \textbf{\textless 0.013} & 0.184 & 0.149 & \textbf{\textless 0.047} & 0.689 & 0.589 & 0.432 & 0.791 & \textbf{\textless 0.013} & \underline{0.050} & \textbf{\textless 0.015} & \textbf{\textless 0.015} & \textbf{\textless 0.001} \\
Columbia & \textbf{\textless 0.001} & \textbf{\textless 0.035} & \textbf{\textless 0.001} & \textbf{\textless 0.001} & \textbf{\textless 0.001} & \textbf{\textless 0.001} & \textbf{\textless 0.001} & \textbf{\textless 0.001} & \textbf{\textless 0.001} & \textbf{\textless 0.001} & \textbf{\textless 0.001} & \textbf{\textless 0.001} & \textbf{\textless 0.001} \\
Hong Kong & \underline{0.053} & \underline{0.094} & \textbf{\textless 0.001} & \textbf{0.001} & \textbf{0.004} & \textbf{0.001} & \textbf{\textless 0.012} & \textbf{\textless 0.013} & \textbf{\textless 0.001} & \textbf{0.008} & \textbf{0.006} & \textbf{0.003} & \underline{0.067} \\
Iquitos & \textbf{\textless 0.001} & \underline{0.075} & \textbf{\textless 0.001} & \textbf{\textless 0.001} & \textbf{\textless 0.001} & \textbf{\textless 0.001} & \textbf{\textless 0.001} & \textbf{\textless 0.001} & \textbf{\textless 0.001} & \textbf{\textless 0.001} & \textbf{\textless 0.001} & \textbf{\textless 0.001} & \textbf{\textless 0.021} \\
Phillipines & \textbf{\textless 0.039} & \underline{0.088} & \textbf{\textless 0.001} & \textbf{\textless 0.001} & \textbf{\textless 0.001} & \textbf{\textless 0.001} & \textbf{\textless 0.001} & \textbf{\textless 0.001} & 0.129 & \textbf{\textless 0.001} & \textbf{\textless 0.001} & \textbf{\textless 0.001} & 0.103 \\
Sanjuan & \textbf{\textless 0.001} & \textbf{\textless 0.001} & \textbf{\textless 0.001} & \textbf{\textless 0.001} & \textbf{\textless 0.001} & \textbf{\textless 0.001} & \textbf{\textless 0.001} & \textbf{\textless 0.001} & \textbf{\textless 0.001} & \textbf{\textless 0.001} & \textbf{\textless 0.001} & \textbf{\textless 0.001} & \textbf{\textless 0.001} \\
Singapore & \textbf{\textless 0.001} & \textbf{0.003} & 0.196 & 0.157 & 0.138 & \underline{0.052} & \textbf{\textless 0.019} & \underline{0.064} & \textbf{\textless 0.001} & \textbf{\textless 0.001} & \textbf{\textless 0.001} & \textbf{\textless 0.032} & \textbf{\textless 0.001} \\
Venezuela & 0.145 & 0.462 & \textbf{\textless 0.001} & \textbf{\textless 0.001} & \textbf{\textless 0.001} & \textbf{\textless 0.001} & \textbf{\textless 0.028} & \underline{0.058} & \textbf{\textless 0.001} & \textbf{\textless 0.001} & \textbf{\textless 0.001} & \textbf{\textless 0.012} & \textbf{\textless 0.001} \\
  \bottomrule
  \end{tabular}
  \end{adjustbox}
\end{table}

\begin{table}[ht!]
  \centering
  \caption{DM test p-values comparing wESN against baseline models across datasets for long-term forecasting. Significance levels: \textbf{p<0.05}, \underline{p<0.10}.}
  \label{tab:dm_hlong_wesn}
  \begin{adjustbox}{width=\textwidth,center}
  \begin{tabular}{@{\extracolsep{5pt}}l|ccccccccccccc}
  \toprule
  Dataset & ETS & ARFIMA & RNN & MRNN & LSTM & MLSTM & GRU & MGRU & wTransformer & TFT & TSMixer & N-BEATS & ESN \\
  \midrule
Ahmedabad & \textbf{\textless 0.001} & \textbf{0.003} & 0.117 & 0.260 & 0.112 & 0.634 & 0.630 & 0.788 & \textbf{\textless 0.001} & \textbf{\textless 0.001} & 0.364 & \textbf{\textless 0.001} & \textbf{\textless 0.001} \\
Bangkok & \textbf{0.003} & 0.123 & \textbf{\textless 0.023} & \textbf{0.010} & 0.136 & 0.105 & \underline{0.077} & 0.170 & \textbf{0.004} & \textbf{0.005} & \textbf{0.002} & \textbf{0.003} & \textbf{0.001} \\
Columbia & \textbf{\textless 0.001} & \textbf{\textless 0.001} & \textbf{\textless 0.001} & \textbf{\textless 0.001} & \textbf{\textless 0.001} & \textbf{\textless 0.001} & \textbf{\textless 0.001} & \textbf{\textless 0.001} & \textbf{\textless 0.001} & \textbf{\textless 0.001} & \textbf{\textless 0.001} & \textbf{\textless 0.001} & \textbf{\textless 0.001} \\
Hong Kong & 0.384 & 0.259 & 0.760 & 0.903 & 0.772 & 0.312 & 0.870 & 0.985 & 0.492 & 0.307 & 0.117 & 0.324 & \underline{0.081} \\
Iquitos & \textbf{\textless 0.001} & 0.297 & \textbf{\textless 0.001} & \textbf{\textless 0.001} & \textbf{\textless 0.001} & \textbf{\textless 0.001} & \textbf{\textless 0.001} & \textbf{\textless 0.001} & \textbf{\textless 0.001} & \textbf{\textless 0.001} & \textbf{\textless 0.001} & \textbf{\textless 0.001} & \underline{0.080} \\
Phillipines & \textbf{\textless 0.021} & 0.213 & \textbf{\textless 0.001} & \textbf{\textless 0.001} & \textbf{\textless 0.001} & \textbf{\textless 0.001} & \textbf{\textless 0.001} & \textbf{\textless 0.001} & \underline{0.060} & \textbf{\textless 0.001} & \textbf{\textless 0.001} & \textbf{\textless 0.001} & \textbf{\textless 0.023} \\
Sanjuan & \textbf{\textless 0.001} & \textbf{\textless 0.001} & \textbf{\textless 0.001} & \textbf{\textless 0.001} & \textbf{\textless 0.001} & \textbf{\textless 0.001} & \textbf{\textless 0.001} & \textbf{\textless 0.001} & \textbf{\textless 0.001} & \textbf{\textless 0.001} & \textbf{\textless 0.001} & \textbf{\textless 0.001} & \textbf{\textless 0.001} \\
Singapore & 0.884 & \underline{0.064} & \textbf{\textless 0.001} & \textbf{\textless 0.001} & \textbf{\textless 0.001} & \textbf{\textless 0.001} & \textbf{\textless 0.001} & \textbf{\textless 0.001} & \textbf{\textless 0.001} & \textbf{\textless 0.001} & \textbf{\textless 0.001} & \textbf{\textless 0.001} & \textbf{\textless 0.001} \\
Venezuela & 0.829 & 0.765 & \textbf{0.008} & \textbf{\textless 0.001} & \textbf{\textless 0.001} & \textbf{\textless 0.001} & \textbf{0.001} & \textbf{\textless 0.001} & \textbf{\textless 0.001} & \textbf{\textless 0.001} & \textbf{\textless 0.001} & \textbf{\textless 0.001} & \textbf{\textless 0.001} \\
  \bottomrule
  \end{tabular}
  \end{adjustbox}
\end{table}

\subsection{Ablation Study} 
We conduct an ablation study to evaluate the robustness of the proposed wESN model with respect to the choice of wavelet filter and decomposition level. We examine four widely used wavelet families, namely Biorthogonal (Bior), Symlets (Sym), and Daubechies filters of order 1 (Haar) and order 2 (Db2), each applied at three different decomposition levels. The results for the three forecasting horizons reported in Appendix~\ref{subsec:wesn_extended} indicate that wESN achieves consistent performance across all filter-level combinations. Hence, we conclude that the wESN model is robust to the choice of wavelet transform. Motivated by this stability, we adopt the Haar wavelet with a single decomposition level in all our empirical experiments, as it provides the simplest representation and lowest computational cost without compromising forecast performance.

\subsection{Uncertainty Quantification} \label{sec:conformal}
Rather than estimating uncertainty by repeatedly changing ESN initializations, we use conformal prediction \citep{vovk2005conformal} because it is model-agnostic and distribution-free, providing calibrated prediction intervals with reliable empirical coverage for fESN and wESN forecasts. We quantify forecast uncertainty for the ESN, fESN, and wESN models using the model-agnostic time series conformal prediction approach \citep{xu2023conformal, zaffran2022adaptive}. Finite-sample prediction bands for the $H$-step-ahead forecasts are constructed using split conformal prediction, with a held-out validation set used exclusively for calibration. Table~\ref{tab:summary_long} reports the coverage, mean width, and Winkler score \citep{gneiting2007strictly} of the prediction intervals for the ESN, fESN, and wESN models for the long-term forecast horizon. For the Philippines dengue dataset, the conformal prediction intervals (long-term horizon) with a 90\% target coverage shown in Figure~\ref{fig:phillipines_conf} indicate that all three models achieve comparable coverage, while fESN produces the narrowest intervals, in contrast to the wider bands of ESN and wESN. A similar pattern is observed for the Ahmedabad dataset (Figure~\ref{fig:ahmbedabad_conf}), where ESN attains a coverage of $0.73$, while fESN and wESN achieve higher coverage of $0.81$ and $0.90$, respectively. Despite higher coverage, fESN maintains a more compact interval, with an average width of $18.47$ compared to $23.38$ for wESN. Results for the remaining forecast horizons, summarized in Appendix~\ref{subsec:conformal}, further confirm these trends. The fESN framework consistently offers the most favorable balance between coverage and interval width, while ESN typically requires wider bands to maintain adequate coverage, and wESN exhibits greater variability owing to occasional prediction deviations.

% In your document
\begin{figure}[t]
  \centering

  \centering
  \includegraphics[width=\linewidth]{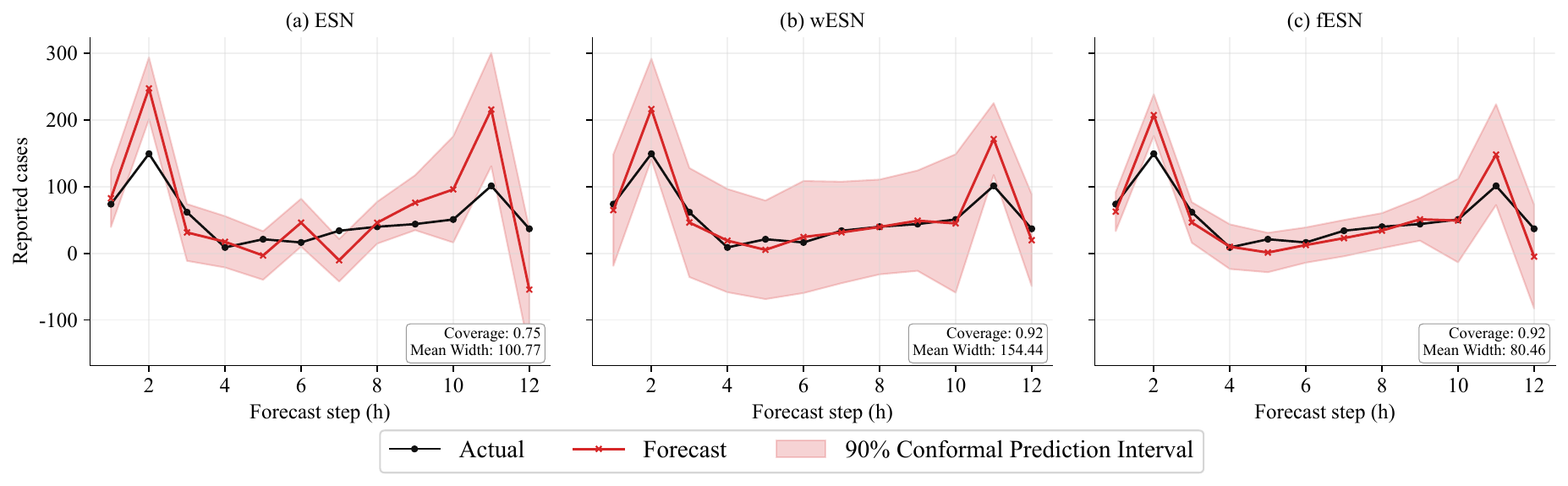}

  \caption{Conformal prediction interval of the (a) ESN, (b) wESN, and (c) fESN models for long-term forecasting of the Philippines dengue cases.}
  \label{fig:phillipines_conf}
\end{figure}

\begin{figure}[t]
  \centering

  \centering
  \includegraphics[width=\linewidth]{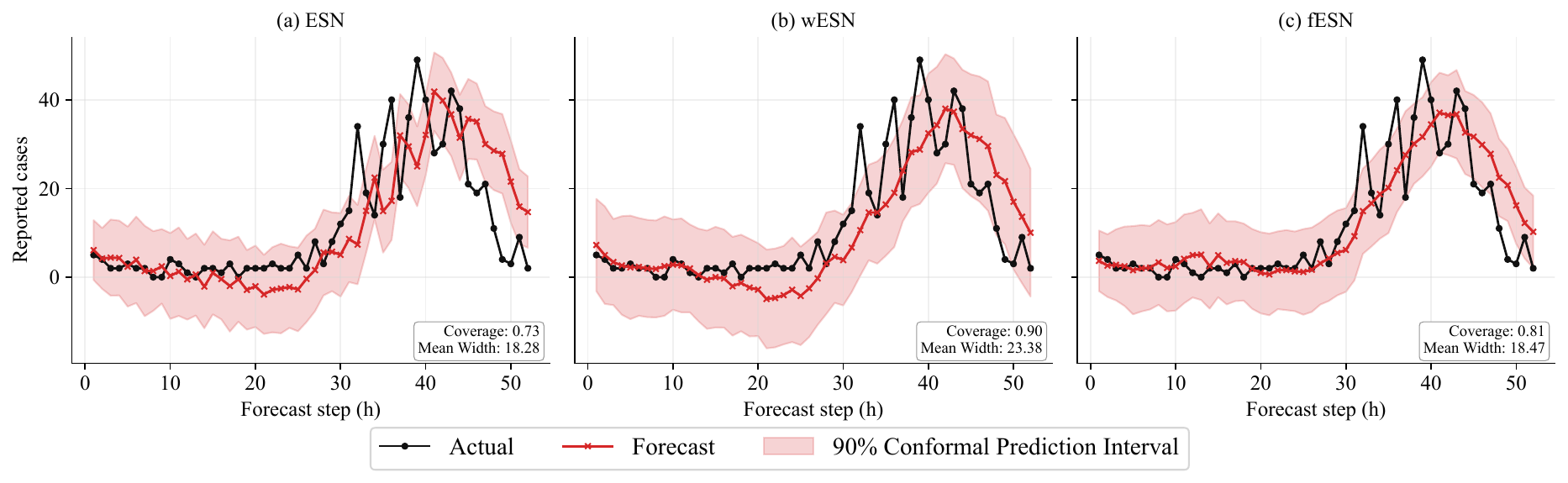}

  \caption{Conformal prediction interval of the (a) ESN, (b) wESN, and (c) fESN models for long-term forecasting of Ahmedabad dengue cases.}
  \label{fig:ahmbedabad_conf}
\end{figure}

\begin{table}[htbp]
\centering
\caption{Long-horizon conformal prediction performance, reported as Coverage (Mean Interval Width, Winkler Score). For each dataset, the minimum Winkler Score across models is shown in \textbf{bold}.}
\label{tab:summary_long}
\begin{tabular}{llll}
\toprule
Models ($\rightarrow$) &                  \multirow{2}{*}{ESN} &                          \multirow{2}{*}{wESN}   &                      \multirow{2}{*}{fESN}  \\
Datasets ($\downarrow$)     &                           &                                 &                                   \\
\midrule
Ahmedabad   &       0.73 (18.28, 61.86) &              0.90 (23.38, 36.95) &      \textbf{0.81 (18.47, 34.99)} \\
Bangkok     &  0.08 (1270.43, 34510.68) &         0.67 (2876.54, 13423.73) &  \textbf{0.67 (1331.79, 6742.02)} \\
Colombia    &  0.44 (1717.80, 13336.81) &         0.44 (1596.98, 14249.60) &  \textbf{0.83 (1670.89, 2646.53)} \\
Hong Kong   &        0.42 (9.41, 37.52) &              0.92 (13.90, 18.32) &       \textbf{0.92 (9.57, 15.31)} \\
Iquitos     &       0.62 (24.13, 51.30) &              0.88 (23.11, 33.26) &      \textbf{0.88 (25.10, 29.45)} \\
Phillipines &     0.75 (100.77, 256.25) &            0.92 (154.44, 181.42) &     \textbf{0.92 (80.46, 124.88)} \\
Sanjuan     &     0.50 (83.87, 1146.84) &   \textbf{0.67 (193.40, 703.35)} &              0.52 (89.29, 756.44) \\
Singapore   &    0.62 (234.12, 1203.66) &  \textbf{0.88 (662.91, 1065.98)} &            0.75 (241.60, 1206.81) \\
Venezuela   &   0.58 (1829.38, 4176.09) &           0.44 (766.75, 8589.17) &  \textbf{0.88 (1857.82, 2253.76)} \\
\bottomrule
\end{tabular}
\end{table}

\section{Simulation Study} \label{sec:simulation}

The purpose of this simulation study is to provide a controlled diagnostic of the memory mechanism theoretically established in Section~\ref{sec:asymptotics}. The preceding results show that a standard ESN produces exponentially decaying lag-response coefficients and therefore behaves as a short-memory system, whereas the proposed memory-augmented architecture can retain past input effects through a polynomially decaying memory pathway. We therefore examine whether this difference is visible at the level of the impulse response. 

This experiment isolates the dynamical effect of the memory reservoir by measuring how a unit shock at time $t = 0$ propagates across future lags. We estimate the lag-response coefficients $\{A_k\}_{k=1}^{T_{\max}}$ by direct impulse-response simulation of two linear dynamical systems: a standard ESN and a memory-augmented ESN. We construct an impulse input of length $T_{\mathrm{imp}}=T_{\max}+1$ by setting $u(0)=1$ and $u(t)=0$ for all $t\in\{1,\dots,T_{\max}\}$. We then simulate each model forward from a zero initial state, with $\mathbf{x}(0)=\mathbf{0}_p$ for the standard ESN and $\mathbf{m}(0)=\mathbf{0}_q$ for the memory-augmented ESN; the fractional-differencing buffer is also initialized with zeros. For the standard linear ESN, the recursion is
$\mathbf{x}(t)=\mathbf{W}_{\!xx}\mathbf{x}(t-1)+\mathbf{W}_{\!xu}u(t) + \eta_x(t)$
and
$y(t)=\mathbf{W}_{\!\mathrm{out},x}\mathbf{x}(t) + \varepsilon_x(t)$.
For the memory-augmented ESN, we first compute the strictly causal filtered input $f(t)$, then update
$\mathbf{m}(t)=\mathbf{W}_{\!mm}\mathbf{m}(t-1)+\mathbf{W}_{\!mf}f(t) + \eta_m(t)$
and compute
$y(t)=\mathbf{W}_{\!\mathrm{out},m}\mathbf{m}(t) + \varepsilon_m(t)$.

Under linearity and time invariance, the impulse response is equivalent to the input--output lag kernel. Hence, for $k\geq 1$, the coefficient $A_k$ multiplying $u(t-k)$ in the expansion
$y(t)=\sum_{k\geq 1}A_k u(t-k)$
is equal to the output at time $k$ when the system is driven by the unit impulse at time $0$. Accordingly, we set $A_k=y(k)$ for $k\in\{1,\dots,T_{\max}\}$, implemented as the slice $\{y(t)\}_{t=1}^{T_{\max}}$. Figure~\ref{fig:impulse_effect} plots the magnitude decay $|A_k|$ against lag $k$. The vanilla ESN response decays rapidly, consistent with short-memory behavior (Theorem~\ref{thm:vanilla-no-lm}), whereas the memory-augmented ESN (fESN) response decays more slowly, consistent with the polynomial-tail behavior derived in Section~\ref{sec:asymptotics} (Theorem~\ref{thm:fesn-lm}). We do not report a separate simulation for wESN because its long-memory mechanism is inherited directly from fESN after a preliminary wavelet-smoothing step. Since this finite-length wavelet transformation only modifies the input signal before the same fractional memory filter is applied, it does not alter the polynomial tail behavior established for fESN.

\begin{figure}
    \centering
    \includegraphics[width=0.95\linewidth]{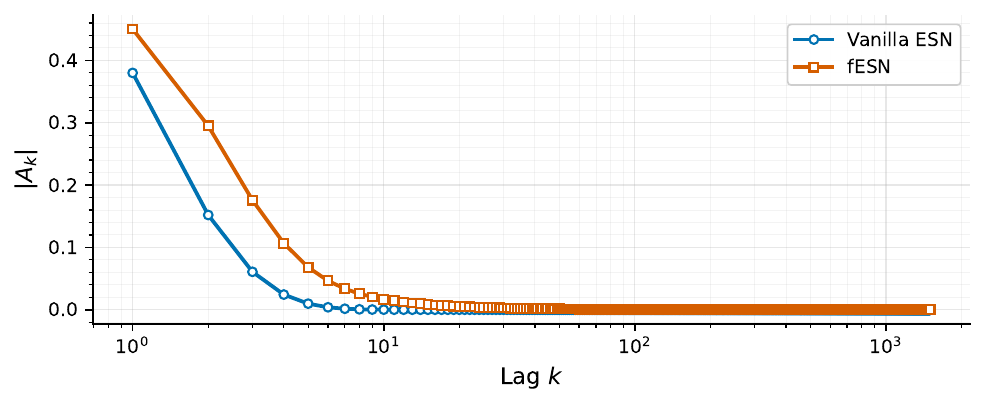}
    \caption{Impulse-response magnitude $|A_k|$ for the vanilla ESN and the fESN. The vanilla ESN response decays rapidly with lag, whereas fESN retains non-negligible lag effects over a longer horizon.}
    \label{fig:impulse_effect}
\end{figure}

As an additional diagnostic of the memory-state dynamics, Figure~\ref{fig:irreducibility} presents the first two principal components of the memory states generated from ten distinct initial conditions. The overlapping trajectories suggest that the simulated memory states do not separate into visually distinct regions of the projected state space. This diagnostic is consistent with the theoretical analysis presented in Section~\ref{sec:asymptotics}.

\begin{figure}
    \centering
    \includegraphics[width=0.90\linewidth]{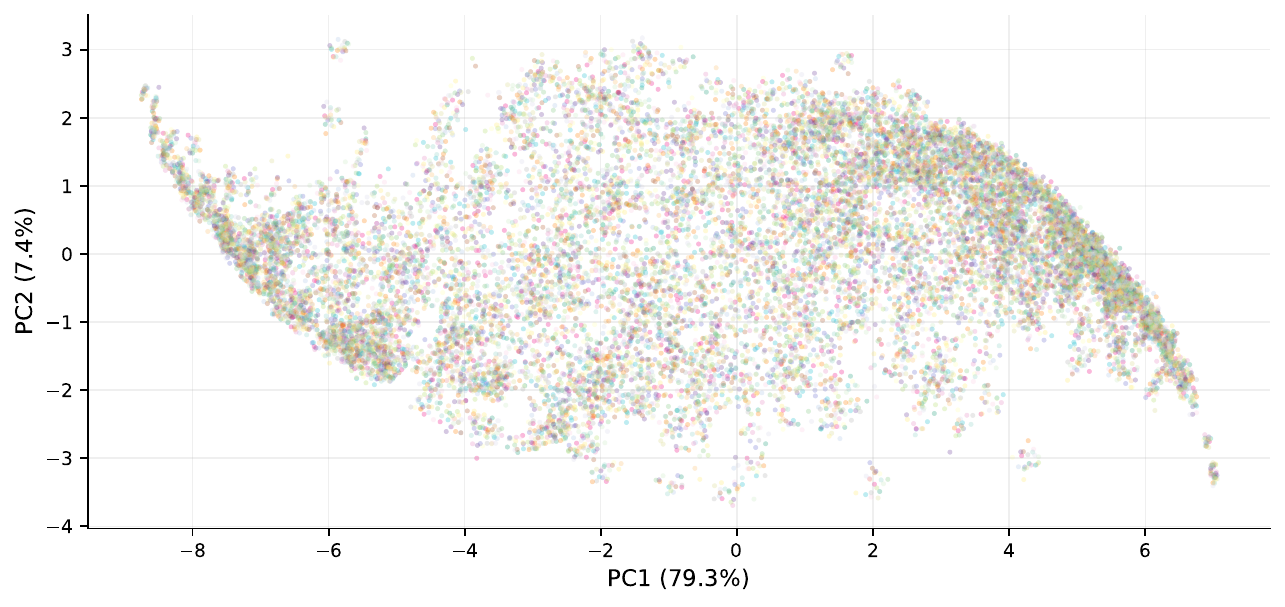}
    \caption{Two-dimensional PCA projection of memory-state trajectories generated from ten distinct initial conditions. PC1 explains most of the memory-state variation (79.3\%), while PC2 captures a smaller secondary mode (7.4\%), indicating that the trajectories are largely organized along one dominant low-dimensional direction.}
    \label{fig:irreducibility}
\end{figure}

\section{Conclusion and Discussion} \label{sec:discussion}
This paper introduced a long-memory reservoir computing framework for data-scarce dengue forecasting. ESNs are computationally efficient and well-suited to small-sample forecasting; their closed-loop generative dynamics lack statistical long memory under standard stability conditions. We show that conventional closed-loop ESNs are geometrically ergodic and therefore induce short-memory processes. In contrast, the proposed long-memory reservoir architectures generate polynomially decaying dependence, consistent with the statistical definition of long memory. This theoretical distinction clarifies why explicit memory augmentation is necessary for dengue incidence series, where persistence, non-stationarity, nonlinearity, and limited sample sizes arise simultaneously. We proposed two ESN-based long-memory architectures: the Fractional ESN (fESN) and the Wavelet ESN (wESN). Both models combine a standard short-memory reservoir, which captures local nonlinear epidemic dynamics, with a dedicated long-memory reservoir, which encodes slowly decaying temporal dependence. In fESN, fractional filtering is introduced directly into the memory pathway, allowing the model to represent persistent dependence through fractional-differencing dynamics. In wESN, wavelet smoothing first extracts a stable low-frequency structure, which is then passed through the memory-aware reservoir. Thus, the two models provide complementary mechanisms for combining the statistical strength of long-memory modeling with the nonlinear representation power and computational efficiency of reservoir computing.

The empirical results confirm the value of this design. Across nine dengue incidence datasets and multiple forecasting horizons, fESN and wESN achieve competitive or superior performance relative to statistical, recurrent, attention-based, memory-augmented, and reservoir-based benchmarks. The results also reveal an important trade-off: while memory-augmented deep recurrent models can encode long-range dependence in principle, their larger number of trainable parameters can make them difficult to estimate reliably from short epidemiological series. Conversely, ARFIMA is data-efficient and effective for long-memory structures, but remains linear. The proposed fESN and wESN models bridge these two regimes by combining explicit long-memory filtering with fixed nonlinear reservoirs and a lightweight ridge-regression readout. Diebold--Mariano tests further support the statistical significance of the forecasting gains. Further, we integrate fESN and wESN with conformal prediction, which is model-agnostic and distribution-free. This provides calibrated prediction intervals without imposing parametric assumptions on dengue forecasting errors. In particular, fESN often achieves a favorable balance between empirical coverage, interval width, and Winkler score, producing compact intervals while maintaining coverage close to the nominal level. The wESN model also improves coverage in several forecasting tasks, although its interval width varies more across datasets and horizons. Finally, the ablation analysis shows that wESN is relatively robust to the choice of wavelet family and decomposition level, supporting the use of a simple Haar wavelet with one decomposition level in the main experiments. The simulation study further verifies the theoretical memory claims by contrasting the exponential correlation decay of standard ESNs with the polynomial decay induced by the proposed long-memory reservoirs.

\subsection*{Limitations and Future Directions}

Despite these encouraging results, the proposed framework has some limitations. The present study focuses on univariate dengue incidence forecasting. This design allows us to isolate the role of long-memory reservoir dynamics, but it does not account for important epidemiological and environmental drivers, such as rainfall, humidity, temperature, population mobility, vector abundance, and intervention measures \citep{wearing2006ecological, nakhaie2026dengue}. Extending fESN and wESN to multivariate forecasting is therefore a natural direction for future research. However, such an extension is nontrivial because different covariates may exhibit different memory strengths. Imposing a common memory parameter across all variables may be too restrictive, and future work could instead estimate variable-specific memory parameters or construct separate memory pathways for epidemiological, climatic, and mobility-related inputs. Another direction for future work is to develop graph-based or spatiotemporal versions of fESN and wESN, in which each location has its own short- and long-memory reservoirs, with spatial interactions encoded via adjacency, mobility, climatic similarity, or learned graph structures \citep{pathak2026deep}. 

In the current formulation, long memory is introduced primarily through the external memory input. A promising extension would be to apply fractional filtering directly to the reservoir states, allowing persistence to operate across all reservoir dimensions rather than only through the memory pathway \citep{geweke1983estimation, carroll2022optimizing}. This may lead to richer long-range representations and could strengthen the theoretical connection between reservoir dynamics and long-memory network processes. Finally, although the proposed models are computationally efficient and effective in small samples, further validation is needed across additional countries, finer spatial resolutions, and other real-time forecasting settings.

\subsection*{Broader Impact Statement}
For public health officials, forecasting dengue incidences represents a real-world modeling challenge that demands robust, generalizable, and computationally efficient approaches capable of anticipating the impact of this vector-borne disease. This work supports public health preparedness by providing data-efficient dengue forecasting tools for settings where long historical records and computational resources may be limited. Our proposed framework can support early warning in regions where large deep learning models are impractical. The use of conformal prediction adds operational value by reporting calibrated uncertainty intervals, helping public health teams judge when forecasts are reliable or risky. However, the current models use dengue incidence alone and may miss changes caused by rainfall, temperature, mobility, reporting delays, or interventions. Therefore, the forecasts should complement, not replace, local epidemiological expertise and surveillance-based decision-making.

\section*{Code Availability Statement}
All datasets used in this study are publicly available at \url{https://github.com/mad-stat/Epicasting}. Implementation code is publicly available at \url{https://github.com/yuvrajiro/memory-esn}, as well as through our Python package `\href{https://pypi.org/project/memory-esn/}{memory-esn}'.

\section*{Acknowledgement}
We sincerely thank Ms. Donia Besher and Mr. Rajdeep Pathak of Sorbonne University for their constructive suggestions, which greatly improved the presentation of this work.

% in many journals. This is optional and at the discretion of the authors. Only add
% this information once your submission is accepted and deanonymized. 

% Only add this information once your submission is accepted and deanonymized. 
\clearpage
\bibliography{main}
\bibliographystyle{tmlr}
\clearpage
\appendix
\section*{Appendix}
\section{Proofs for Section~\ref{sec:asymptotics}}
\label{app:proofs}

\subsection{Proof of Lemma~\ref{lem:block-spectrum}}

\begin{proof}
For $Y=(y,\mathbf{x}^\top)^\top\in\mathbb R^{1+p}$ with $y\in\mathbb R$, $\mathbf{x}\in\mathbb R^p$, set
\begin{equation*}
   Z=
\begin{bmatrix}\mathbf{x}\\[2pt] y - \mathbf W_{\mathrm{out}}\mathbf{x}\end{bmatrix}
= S\,Y,
\quad
S=
\begin{bmatrix}\mathbf 0; & I_p\\ 1; &-\mathbf W_{\mathrm{out}}\end{bmatrix},
\quad
S^{-1}=
\begin{bmatrix}\mathbf W_{\mathrm{out}};&1\\ I_p;&\mathbf 0\end{bmatrix}. 
\end{equation*}
Writing $\mathbf{a}=\mathbf W_{\mathrm{out}}\mathbf W_{xy}$, $\mathbf{b}=\mathbf W_{\mathrm{out}}\mathbf W_{xx}$, $\mathbf{c}=\mathbf W_{xy}$ and $\mathbf{d}=\mathbf W_{xx}$,
we compute
\begin{equation*}
    \mathbf M_{\mathrm{ESN}}S^{-1}
=\begin{bmatrix}\mathbf{a}\mathbf W_{\mathrm{out}}+\mathbf{b}; & \mathbf{a}\\ \mathbf{c}\mathbf W_{\mathrm{out}}+\mathbf{d}; & \mathbf{c}\end{bmatrix},
\quad
S\mathbf M_{\mathrm{ESN}}S^{-1}
=\begin{bmatrix}
\mathbf{c}\mathbf W_{\mathrm{out}}+\mathbf{d}; \quad & \textbf{c}\\
\mathbf{a}\mathbf W_{\mathrm{out}}+\mathbf{b}-\mathbf W_{\mathrm{out}}(\mathbf{c}\mathbf W_{\mathrm{out}}+\mathbf{d}); \quad & \mathbf{a}-\mathbf W_{\mathrm{out}}\mathbf{c}
\end{bmatrix}.
\end{equation*}
Since $\mathbf{a}=\mathbf W_{\mathrm{out}}\mathbf{c}$ and $\textbf{b}=\mathbf W_{\mathrm{out}}\mathbf{d}$, the lower block row is zero, and with $\mathbf F=\mathbf{d}+\mathbf{c}\mathbf W_{\mathrm{out}}$, we get $ S\mathbf M_{\mathrm{ESN}}S^{-1}
=\begin{bmatrix}\mathbf F & \mathbf W_{xy}\\ \mathbf 0 & 0\end{bmatrix}$. For a block upper-triangular matrix, the characteristic polynomial factors as
\begin{equation*}
  \det\!\big(\lambda I_{p+1}-S\mathbf M_{\mathrm{ESN}}S^{-1}\big)
= \det(\lambda I_p-\mathbf F)\cdot \lambda,  
\end{equation*}
so the eigenvalues of $\mathbf M_{\mathrm{ESN}}$ are exactly those of $\mathbf F$ together with $0$. The spectral radius is therefore $\rho(\mathbf M_{\mathrm{ESN}})=\rho(\mathbf F)$.
\end{proof}

\subsection{Proof of Lemma~\ref{lem:esp-iff}}
\begin{proof}
Let two trajectories $(\mathbf{x}_1,y_1)$ and $(\mathbf{x}_2,y_2)$ be driven by the same $\{\varepsilon(t)\}$ and $\{\eta(t)\}$. Subtracting the state and readout equations gives
\begin{equation*}
   \Delta\mathbf{x}(t)=\mathbf{W}_{xx}\Delta\mathbf{x}(t-1)+\mathbf{W}_{xy}\Delta y(t-1),
\qquad
\Delta y(t)=\mathbf{W}_{\mathrm{out}}\,\Delta\mathbf{x}(t), 
\end{equation*}
since the noise cancels in the readout difference. Eliminating $\Delta y(t-1)$ via the readout at time $t-1$ yields
\begin{equation*}
    \Delta\mathbf{x}(t)=\bigl(\mathbf{W}_{xx}+\mathbf{W}_{xy}\mathbf{W}_{\mathrm{out}}\bigr)\Delta\mathbf{x}(t-1)
=\mathbf{F}\,\Delta\mathbf{x}(t-1),
\end{equation*}
and therefore by induction,
\begin{equation*}
    \Delta\mathbf{x}(t)=\mathbf{F}^{\,t}\Delta\mathbf{x}(0),\qquad
\Delta y(t)=\mathbf{W}_{\mathrm{out}}\,\Delta\mathbf{x}(t).
\end{equation*}
If $\rho(\mathbf{F})<1$, then $\mathbf{F}^{\,t}\to \mathbf{0}$ as $t\to\infty$. Hence $\Delta\mathbf{x}(t)=\mathbf{F}^{\,t}\Delta\mathbf{x}(0)\to \mathbf{0}$ for every initial difference, and consequently $\Delta y(t)=\mathbf{W}_{\mathrm{out}}\,\Delta\mathbf{x}(t)\to 0$. Thus, ESP holds.

Conversely, if $\rho(\mathbf{F})\ge 1$, there exists a nonzero vector $\mathbf{v}$ in a generalized eigenspace associated with some eigenvalue $\lambda$ with $|\lambda|\ge 1$ such that $\|\mathbf{F}^{\,t}\mathbf{v}\|\not\to 0$ (in fact it either stays bounded away from $0$ or grows at most polynomially times $|\lambda|^t$). Choosing $\Delta\mathbf{x}(0)=\mathbf{v}$ produces a pair of trajectories driven by the same noise for which $\Delta\mathbf{x}(t)$ does not converge to $\mathbf{0}$, contradicting ESP. Therefore, $\rho(\mathbf{F})<1$.

\medskip
The two implications prove (i) $\Leftrightarrow$ (ii). Finally, Lemma~\ref{lem:block-spectrum} gives $\rho(\mathbf{M}_{\mathrm{ESN}})=\rho(\mathbf{F})$, so the ESP criterion can be equivalently checked on $\mathbf{M}_{\mathrm{ESN}}$.
\end{proof}

\subsection{Proof of \ref{thm:esn-geo-ergodic}}

\begin{proof}
Let the Markov chain be $\{\mathbf{Y}(t)\}$, where $\mathbf{Y}(t)=\big[y(t);\,\mathbf{x}(t)^\top\big]^\top \in \mathbb{R}^{1+p}$. Its state space is $(\mathbb{R}^{1+p}, \mathcal{B}^{1+p})$, where $\mathcal{B}^{1+p}$ is the class of Borel sets. Under the linear setting, the ESN model is defined in Eqn.~\ref{eq:esn-linear-block} can be written as
\begin{equation}\label{eq:Y_update_proof}
\mathbf{Y}(t) = \mathbf{M}_{\mathrm{ESN}}\,\mathbf{Y}(t-1) + \mathbf{E}(t),
\end{equation}
where $\mathbf{Y}(t) \in \mathbb{R}^{1+p}$, $\mathbf{M}_{\mathrm{ESN}} \in \mathbb{R}^{(1+p) \times (1+p)}$, and $\mathbf{E}(t) = [\varepsilon_t; \eta(t)^\top]^\top \in \mathbb{R}^{1+p}$ is the noise vector. %$\mathbf{E}(t) = \begin{bmatrix}\varepsilon(t); \mathbf{0}_p\end{bmatrix}^\top$.
The transition probability 
\begin{equation*}
    \label{eq:transition}
    P(\mathbf{y}, A) = \Pr(\mathbf{Y}(t) \in A \mid \mathbf{Y}(t-1) = \mathbf{y}) = \int_Af(z - \mathbf{M}_{ESN}\mathbf{y})dz
\end{equation*}
is defined via this update. Under Assumption~\ref{ass:network1}, $\{\mathbf{Y}(t)\}$ is $\nu_{1+p}$-irreducible as the joint density is positive, where $\nu_{1+p}$ is the Lebesgue measure, following \citet{zhao2020rnn}.

We suppose $\rho(\mathbf{M}_{\mathrm{ESN}}) < 1$ implying that there exists an integer $s$ such that $\|\mathbf{M}_{\mathrm{ESN}}^s\| < 1$. In the following, we prove that the $s$-step Markov chain $\{\mathbf{Y}{(ts)}\}$ satisfies Tweedie's drift criterion \citep{tweedie1983criteria}, i.e., there exists a small set $G$ with $\nu_{1+p}(G) > 0$ and a non-negative continuous function $\psi(\mathbf{y})$ such that
\begin{equation*}
\mathbb{E}\{\psi(\mathbf{Y}{(ts)}) | \mathbf{Y}{((t-1)s)} = \mathbf{y}\} \le (1 - \epsilon)\psi(\mathbf{y}), \quad \mathbf{y} \notin G,
\end{equation*}
and
\begin{equation*}
\mathbb{E}\{\psi(\mathbf{Y}{(ts)}) | \mathbf{Y}{((t-1)s)} = \mathbf{y}\} \le M, \quad \mathbf{y} \in G,
\end{equation*}
for some constant $0 < \epsilon < 1$ and $0 < M < \infty$.

Iterating the model defined in Eqn.~\ref{eq:Y_update_proof} $s$ times, we obtain
\begin{equation*}
\mathbf{Y}{(ts)} = \mathbf{M}_{\mathrm{ESN}}^s \mathbf{Y}{((t-1)s)} + \left( \mathbf{E}{(ts)} + \sum_{j=1}^{s-1} \mathbf{M}_{\mathrm{ESN}}^j \mathbf{E}{(ts-j)} \right).
\end{equation*}
Let $\psi(\mathbf{y}) = 1 + \|\mathbf{y}\|^\kappa$, where $\kappa \ge 2$ is from Assumption~\ref{ass:network1}. Using Minkowski's inequality, it follows that
\begin{align*}
\mathbb{E}\{\psi(\mathbf{Y}{(ts)}) | \mathbf{Y}{((t-1)s)} = \mathbf{y}\}
   &= 1 + \mathbb{E}\left\| \mathbf{M}_{\mathrm{ESN}}^s \; \mathbf{y} + \left( \mathbf{E}{(ts)} + \sum_{j=1}^{s-1} \mathbf{M}_{\mathrm{ESN}}^j \mathbf{E}{(ts-j)} \right) \right\|^\kappa \\
   &\le 1 + \left( \|\mathbf{M}_{\mathrm{ESN}}^s\| \|\mathbf{y}\| + \left(\mathbb{E}\left\| \mathbf{E}{(ts)} + \sum_{j=1}^{s-1} \mathbf{M}_{\mathrm{ESN}}^j \mathbf{E}{(ts-j)} \right\|^\kappa \right)^{1/\kappa} \right)^\kappa \\
   &\le 1 + 2^{\kappa-1}\left( \|\mathbf{M}_{\mathrm{ESN}}^s\|^\kappa \|\mathbf{y}\|^\kappa + \mathbb{E}\left\| \mathbf{E}{(ts)} + \sum_{j=1}^{s-1} \mathbf{M}_{\mathrm{ESN}}^j \mathbf{E}{(ts-j)} \right\|^\kappa \right) \\
   &\le 2^{\kappa-1}\psi(\mathbf{y}) \|\mathbf{M}_{\mathrm{ESN}}^s\|^\kappa + C,
\end{align*}
where $C = 1 + 2^{\kappa-1}\left(\mathbb{E}\left\| \mathbf{E}{(ts)} + \sum_{j=1}^{s-1} \mathbf{M}_{\mathrm{ESN}}^j \mathbf{E}{(ts-j)} \right\|^\kappa - \|\mathbf{M}_{\mathrm{ESN}}^s\|^\kappa\right) < \infty$. The finiteness of $C$ is guaranteed by Assumption~\ref{ass:network1}.

Note that $\|\mathbf{M}_{\mathrm{ESN}}^s\|^\kappa < 1$. Then there exists $L > 0$, such that for some $\epsilon > 0$,
\begin{equation*}
\mathbb{E}\{\psi(\mathbf{Y}{(ts)}) | \mathbf{Y}{((t-1)s)} = \mathbf{y}\} \le (1 - \epsilon)\psi(\mathbf{y}), \quad \forall \|\mathbf{y}\| > L,
\end{equation*}
and
\begin{equation*}
\mathbb{E}\{\psi(\mathbf{Y}{(ts)}) | \mathbf{Y}{((t-1)s)} = \mathbf{y}\} \le M < \infty, \quad \forall \|\mathbf{y}\| \le L.
\end{equation*}
We define the compact set $G = \{\mathbf{y} : \|\mathbf{y}\| \le L\}$, which has $\nu_{1+p}(G) > 0$.

Moreover, for any bounded continuous function $\phi(\cdot)$, the expectation $E\{\phi(\mathbf{Y}{(ts)}) | \mathbf{Y}{((t-1)s)} = \mathbf{y}\}$ is continuous with respect to $\mathbf{y}$. Thus, $\{\mathbf{Y}{(ts)}\}$ is a Feller chain. Since it is also $\nu_{1+p}$-irreducible, this implies that $G$ is a small set \citep{feigin1985random}. By Theorem 4(ii) in \citet{tweedie1983criteria}, the $s$-step Markov chain $\{\mathbf{Y}^{(ts)}\}$ is geometrically ergodic with a unique strictly stationary solution. By Lemma 3.1 of \citet{tjostheim1990non}, the original chain $\{\mathbf{Y}^{(t)}\}$ is also geometrically ergodic.

Conversely (Necessity), suppose the model defined in Eqn.~\ref{eq:Y_update_proof} is geometrically ergodic. Then there exists a strictly stationary solution $\{\mathbf{Y}(t)\}$ \citep{feigin1985random}. This stationary solution has a distribution $\pi(\cdot)$ from which we can generate $\mathbf{Y}^{(0)}$ and iteratively obtain the sequence.

From the transition probability, it holds that $\Pr(\mathbf{Y}(t) \in A \mid \mathbf{Y}(t-1) = \mathbf{y}) = P(\mathbf{y}, A) > 0$ if $\nu_{1+p}(A) > 0$, as established by the irreducibility. Let $H$ be any affine invariant subspace of $\mathbb{R}^{1+p}$ under the model in Eqn.~\ref{eq:Y_update_proof}, i.e., $\{\mathbf{M}_{\mathrm{ESN}}\mathbf{y} + \mathbf{E}(t) : \mathbf{Y} \in H\} \subseteq H$ with probability one. If $\nu_{1+p}(\mathbb{R}^{1+p} - H) \ne 0$, then for any $\mathbf{y} \in H$, $\Pr(\mathbf{M}_{\mathrm{ESN}}\mathbf{y} + \mathbf{E}(t) \in H) < 1$. As a result, $\mathbb{R}^{1+p}$ is the unique affine invariant subspace, and hence model in Eqn.~\ref{eq:Y_update_proof} is irreducible. Thus, by Theorem 2.5 in \citet{bougerol1992strict}, the top Lyapunov exponent is strictly negative, and thus the spectral radius $\rho(\mathbf{M}_{\mathrm{ESN}}) = \lim_{s\to\infty} \|\mathbf{M}_{\mathrm{ESN}}^s\|^{1/s} < 1$. This completes the proof.
\end{proof}

\subsection{Proof of Theorem~\ref{thm:esn_nonlinear}}
\begin{proof}
Let $\mathbf{Y}(t)=[y(t);\mathbf{x}(t)]^\top\in\mathbb{R}^{1+p}$ with $y(t)\in\mathbb{R}$ and $\mathbf{x}(t)\in\mathbb{R}^p$. The ESN recursion can be written as
\begin{equation*}
  \mathbf{Y}(t)=\mathcal{M}\!\bigl(\mathbf{Y}(t-1)\bigr)+\mathbf{E}(t),\qquad
\mathbf{E}(t)=\bigl[\varepsilon(t);\eta(t)^\top\bigr]^\top,  
\end{equation*}
with
\begin{equation*}
    \mathcal{M}(y,\mathbf{x})
=
\begin{bmatrix}
g\!\big(\,\mathbf{W}_{\mathrm{out}}\,\tau(\mathbf{W}_{xx}\mathbf{x}+\mathbf{W}_{xy}y)\,\big)\\[2pt]
\tau(\mathbf{W}_{xx}\mathbf{x}+\mathbf{W}_{xy}y)
\end{bmatrix},
\end{equation*}
where $\mathbf{W}_{xx}\in\mathbb{R}^{p\times p}$, $\mathbf{W}_{xy}\in\mathbb{R}^{p\times 1}$, $\mathbf{W}_{\mathrm{out}}\in\mathbb{R}^{1\times p}$. We need to show that there always exist real numbers $a<1$ and $b$ such that
\begin{equation*}
\bigl\|\mathcal{M}(y,\mathbf{x})\bigr\|_{\ell_1} \;\le\; a\,\bigl\|(y,\mathbf{x})^\top\bigr\|_{\ell_1} + b
\end{equation*}

Since $g(\cdot)$ and $\tau(\cdot)$ are bounded, there exist positive constants $M_1$ and $M_2$ such that
\begin{equation*}
g\!\big(\,\mathbf{W}_{\mathrm{out}}\,\tau(\mathbf{W}_{xx}\mathbf{x}+\mathbf{W}_{xy}y)\,\big)\;\le\; M_1,
\qquad
\tau(\mathbf{W}_{xx}\mathbf{x}+\mathbf{W}_{xy}y) \;\le\; M_2
\end{equation*}
for any $y\in\mathbb{R}$, $\mathbf{x}\in\mathbb{R}^p$. Let $a=a_0\in(0,1)$ and $b=M_1+M_2$. We have
\begin{equation*}
\bigl\|\mathcal{M}(y,\mathbf{x})\bigr\|_{\ell_1}
- a_0\,\bigl\|(y,\mathbf{x}^\top)^\top\bigr\|_{\ell_1}
\;\le\; M_1+M_2 - a_0|y| - a_0\|\mathbf{x}\|_{\ell_1}
\;\le\; M_1+M_2 \;=\; b.
\end{equation*}
By Theorem~1 of \citet{zhao2020rnn}, ESN with bounded and continuous output and activation functions is geometrically ergodic and exhibits short memory.
\end{proof}

\subsection{Proof of Theorem~\ref{thm:vanilla-no-lm}}
\begin{proof}
For simplicity of the proof, we consider the (open-loop) linear ESN with input $u(t)$ and output $z(t)$ without the noise terms:
\begin{equation}\label{eq:open_loop_esn}
\mathbf{x}(t)=\mathbf{W}_{xx}\,\mathbf{x}(t-1)+\mathbf{W}_{xu}\,u(t),\qquad
 z(t)=\mathbf{W}_{\mathrm{out}}\,\mathbf{x}(t),
\end{equation}
with $\mathbf{W}_{xx}\in\mathbb{R}^{p\times p}$, $\mathbf{W}_{xu}\in\mathbb{R}^{p\times 1}$, $\mathbf{W}_{\mathrm{out}}\in\mathbb{R}^{1\times p}$, $\mathbf{x}(t)\in\mathbb{R}^p$, and scalar $u(t)\in\mathbb{R}$. Unrolling Eqn.~\ref{eq:open_loop_esn} gives, for any $t\ge1$,
\[
\mathbf{x}(t)=\mathbf{W}_{xx}^{\,t}\mathbf{x}(0)+\sum_{k=0}^{t-1}\mathbf{W}_{xx}^{\,k}\,\mathbf{W}_{xu}\,u(t-k).
\]
Since $\rho(\mathbf{W}_{xx})<1$, there exist $M>0$ and $r\in(0,1)$ such that $\|\mathbf{W}_{xx}^{\,k}\|\le Mr^{k}$ for all $k\ge0$. Hence, the transient $\mathbf{W}_{xx}^{\,t}\mathbf{x}(0)\to\mathbf{0}$ exponentially fast and, as $t\to\infty$, the input-output map is the linear time-invariant filter
\begin{equation*}\label{eq:conv-form}
z(t)=\sum_{k=0}^{\infty} A_k\,u(t-k),
\qquad
A_k = \mathbf{W}_{\mathrm{out}}\,\mathbf{W}_{xx}^{\,k}\,\mathbf{W}_{xu}\in\mathbb{R}.
\end{equation*}
Moreover, since $0 < r <1$, then 
\[
|A_k|\;\le\;\|\mathbf{W}_{\mathrm{out}}\|\,\|\mathbf{W}_{xu}\|\,\|\mathbf{W}_{xx}^{\,k}\|
\;\le\;\|\mathbf{W}_{\mathrm{out}}\|\,\|\mathbf{W}_{xu}\|\,M\,r^{k},
\]
so $\sum_{k=0}^{\infty}|A_k|<\infty$ and the impulse response \(\{A_k\}\) decays exponentially.

% Let $u(t)$ be square-integrable (e.g., second-order stationary). Then \eqref{eq:conv-form} defines a bounded LTI filter with absolutely summable coefficients. Such filters produce outputs with \emph{short memory}: for instance, if $u(t)$ is white noise with variance $\sigma^2$, the autocovariance of $\hat y(t)$ is
% \[
% \gamma_{\hat y}(h)=\sigma^2\sum_{k=0}^{\infty} A_k A_{k+|h|},
% \]
% and the geometric bound on $\{A_k\}$ implies $\sum_{h=-\infty}^{\infty}|\gamma_{\hat y}(h)|<\infty$. Long memory requires a non-summable (typically polynomially decaying) impulse response or, equivalently, non-summable autocovariances, which cannot occur here. Thus a vanilla ESN with $\rho(\mathbf{W}_{xx})<1$ cannot generate a long-memory process.
\end{proof}

\subsection{Proof of Theorem~\ref{thm:fesn-lm}}
\begin{proof}
Unrolling the reservoir state of fESN,
\begin{equation*}
   \mathbf{x}(t)=\mathbf{W}_{xx}^{\,t}\mathbf{x}(0)+\sum_{k=0}^{t-1}\mathbf{W}_{xx}^{k}\mathbf{W}_{xu}\,u(t-k). 
\end{equation*}
Since $\rho(\mathbf{W}_{xx})<1$, there exist $M>0$ and $\alpha\in(0,1)$ with $\|\mathbf{W}_{xx}^{\,k}\|\le M\alpha^k \  \forall k \geq 0$; the transient $\mathbf{W}_{xx}^{\,t}\mathbf{x}(0)\to\mathbf{0}$ exponentially. Thus, the reservoir contribution to the impulse response is
\begin{equation*}
   C_k=\mathbf{W}_{\mathrm{out},x}\,\mathbf{W}_{xx}^{\,k}\,\mathbf{W}_{xu},\qquad \|C_k\|=O(\alpha^k). 
\end{equation*}
% Write $(1-B)^d-1=\sum_{r\ge1}\omega_d(r)B^r$ with
% \[
% \omega_d(r)=(-1)^r\binom{d}{r}=\frac{\Gamma(d+1)}{\Gamma(r+1)\Gamma(d-r+1)}\sim \frac{1}{\Gamma(-d)}\,r^{-(d+1)}\quad(r\to\infty).
% \]
Unrolling the memory state gives us
\begin{equation*}
   \mathbf{m}(t)=\mathbf{W}_{\!xx}^{\,t}\mathbf{m}(0)+\sum_{k=0}^{t-1}\mathbf{W}_{\!mm}^{k}\mathbf{W}_{mf}((1 - B)^d - 1)u(t-k). 
\end{equation*}
Taking
\begin{equation*}
    D_k = \mathbf{W}_{\mathrm{out},m}\,\mathbf{W}_{mm}^{\,k}\,\mathbf{W}_{mf}((1 - B)^d - 1),
\end{equation*}
we see that $z(t) = \sum_{k \geq 0} C_k u(t-k) + \sum_{k \geq 0} D_k u(t-k) + \varepsilon(t) = \sum_{k \geq 0} A_k u(t-k) + \varepsilon(t)$, where $A_k = C_k + D_k$. Since $C_k$ decays exponentially but $D_k$ is dominated by the polynomially decaying weights ($\sim k^{-d-1}$), the entries of $A_k$ decay at a polynomial rate $k^{-\tilde{d}-1}$. Hence, by Definition~\ref{def:long_memory}, fESN is a long memory network process.

% Because $\rho(\mathbf{W}_{mm})<1$, there exist $M'>0$ and $\beta\in(0,1)$ with $\|\mathbf{W}_{mm}^{\,q}\|\le M'\beta^q$, hence
% \[
% \|D_k\|\ \le\ \sum_{q=0}^{k-1}\|\mathbf{W}_{\mathrm{out},m}\|\,\|\mathbf{W}_{mf}\|\,C'\beta^q\,|\omega_d(k-q)|
% \ =\ O\!\Big(\sum_{q=0}^{k-1}\beta^q\,(k-q)^{-(d+1)}\Big).
% \]
% Split the sum at $q_0=\lfloor k/2\rfloor$. For $0\le q\le q_0$, $(k-q)\asymp k$, giving $O(k^{-(d+1)}\sum_{q\le q_0}\beta^q)=O(k^{-(d+1)})$. For $q_0<q\le k-1$, $(k-q)\le k/2$ so $(k-q)^{-(d+1)}\le (k/2)^{-(d+1)}$, and $\sum_{q>q_0}\beta^q=O(\beta^{k/2})$, hence this tail is $O(k^{-(d+1)}\beta^{k/2})$. 
% Therefore $\|D_k\|=O(k^{-(d+1)})$.
% Combining the two branches yields $A_k=C_k+D_k$ with $\|C_k\|=O(\alpha^k)$ and $\|D_k\|=O(k^{-(d+1)})$, so the overall impulse tail is polynomial. By the paper’s convention, this qualifies as a \emph{long memory network process}.
\end{proof}

% \begin{remark}[Optional stronger asymptotics under non-degeneracy]
% If $\mathbf{W}_{\mathrm{out},m}\,(\mathbf{I}_q-\mathbf{W}_{mm})^{-1}\mathbf{W}_{mf}\neq 0$, then
% \[
% D_k\ \sim\ \frac{1}{\Gamma(-d)}\,\mathbf{W}_{\mathrm{out},m}\,(\mathbf{I}_q-\mathbf{W}_{mm})^{-1}\mathbf{W}_{mf}\ \cdot\ k^{-(d+1)},
% \]
% and $A_k\sim D_k$ since $C_k$ is exponentially small. Without this non-degeneracy, the universal statement is the big-$O$ bound above.
% \end{remark}

\subsection{Proof of Theorem~\ref{thm:wesn-lm} (wESN)}
Let the implemented level-1 MODWT Haar smoother be
\[
s(t)=\sum_{\ell=0}^{L-1}\tilde g_1[\ell]\;u\bigl(t-\ell\bigr),
\qquad
\tilde g_1[\ell]=\frac{g[\ell]}{\sqrt{2}},
\]
with base low-pass taps $g[\ell]$. Initial transients $W_{xx}^t x(0)$ and $W_{mm}^t m(0)$
decay exponentially and are omitted.

The model is
\begin{align*}
\mathbf{z}(t) &= \mathbf{W}_{\!\mathrm{out},x}\,\mathbf{x}(t) + \mathbf{W}_{\!\mathrm{out},m}\,\mathbf{m}(t) + \varepsilon(t),\\
\mathbf{x}(t) &= \mathbf{W}_{\!xx}\,\mathbf{x}(t-1) + \mathbf{W}_{xu}\,u(t),\\
\mathbf{m}(t) &= \mathbf{W}_{\!mm}\,\mathbf{m}(t-1) + \mathbf{W}_{\!mf}\,\big((1-B)^d - 1\big)\,s(t).
\end{align*}

Solving the recursions:
\begin{align*}
\mathbf{x}(t) &= (I-W_{xx}B)^{-1}\mathbf{W}_{\!xu}\,u(t)
     = \sum_{p=0}^{\infty}\mathbf{W}_{\!xx}^{\,p}B^p\mathbf{W}_{\!xu}\,u(t)
     = \sum_{p=0}^{\infty} \mathbf{W}_{\!xx}^{\,p}\mathbf{W}_{\!xu}\,u(t-p),\\[2mm]
\mathbf{m}(t) &= (I-\mathbf{W}_{\!mm}B)^{-1}\mathbf{W}_{\!mf}\big(1-B)^d-1\big)s(t) \\
     &= \sum_{q=0}^{\infty}\sum_{r=1}^{\infty}\sum_{\ell=0}^{L-1}
        \mathbf{W}_{\!mm}^{\,q}\mathbf{W}_{\!mf}\,\omega_r(d)\,\tilde g_1[\ell]\;
        u\bigl(t-q-r-\ell\bigr).
\end{align*}

Therefore,
\[
\mathbf{z}(t)=\sum_{k=0}^{\infty}\mathbf{A}_k\,u(t-k)+\varepsilon(t),\qquad \mathbf{A}_k=\mathbf{C}_k+\mathbf{D}_k^{(1)},
\]
with
\begin{align*}
C_k &= \mathbf{W}_{\!\mathrm{out},x}\,\mathbf{W}_{\!xx}^{\,k}\,\mathbf{W}_{\!xu} \quad \text{and}\\
D_k^{(1)} &= \sum_{\substack{\ell=0,\dots,L-1\\ \ell\le k-1}}
            \ \sum_{q=0}^{\,k-\ell-1}
            \mathbf{W}_{\!\mathrm{out},m}\,\mathbf{W}_{\!mm}^{\,q}\,\mathbf{W}_{mf}\,
            \omega_{k-\ell-q}(d)\,\tilde g_1[\ell],
\quad D_0^{(1)}=0.
\end{align*}

Since $\|\mathbf{W}_{xx}^{\,k}\|\le C_1\alpha^k$ and $\|\mathbf{W}_{mm}^{\,k}\|\le C_2\beta^k$ for some $0<\alpha,\beta<1$ by Assumption~\ref{ass:esp}.
\begin{equation*}
  \|C_k\|=\mathcal{O}(\alpha^{k}),\qquad
\|D_k^{(1)}\| \le \sum_{\ell}\sum_{q=0}^{k-\ell-1}\mathcal{O}(\beta^{q})\,
                 \mathcal{O}\big((k-\ell-q)^{-(d+1)}\big)
           = \mathcal{O}\big(k^{-(d+1)}\big),  
\end{equation*}
so $\|A_k\|=\mathcal{O}\big(k^{-(\tilde{d}+1)}\big)$. The smoother 
does not change the polynomial tail exponent of the impulse response.

\section{LOOCV Trick} \label{app:loocv}
Let $X\in\mathbb{R}^{n\times p}$ denote the feature matrix and $Y\in\mathbb{R}^{n\times d}$ the target matrix. 
We estimate the coefficient matrix $\beta\in\mathbb{R}^{p\times d}$ by ridge regression
\begin{equation*}
\label{eq:ridge_obj}
\min_{\beta}\;\; \|Y - X\beta\|_F^2 \;+\; \lambda \,\|\beta\|_F^2,
\end{equation*}
where $\lambda\ge 0$ is the ridge penalty.\footnote{When an intercept is needed, a standard approach is to center the columns of $X$ and $Y$; the intercept is then recovered as the column means of $Y$ minus the means of $X$ times $\beta$. The intercept is not penalized.} 
The solution is the familiar closed form
\begin{equation*}
\label{eq:ridge_sol}
\beta(\lambda) \;=\; (X^\top X + \lambda I_p)^{-1} X^\top Y.
\end{equation*}
Define the fitted values and the \emph{smoother} (hat) matrix
\begin{equation*}
\widehat{Y} \;=\; X\beta(\lambda) \;=\; S_\lambda\,Y,\qquad 
S_\lambda \;=\; X\,(X^\top X + \lambda I_p)^{-1} X^\top.
\end{equation*}
Leave-one-out cross-validation fits the model $n$ times, each time holding out the $i$th observation. 
For any linear smoother $\widehat{Y}=S_\lambda Y$, the leave-one-out residuals admit an exact, one-fit formula:
\begin{equation}
\label{eq:loo_res}
E \;=\; Y - \widehat{Y},\qquad 
E^{(-i)}_{i\cdot} \;=\; \frac{E_{i\cdot}}{\,1 - (S_\lambda)_{ii}\,}\quad\text{for }i=1,\ldots,n,
\end{equation}
where $E_{i\cdot}\in\mathbb{R}^{1\times d}$ is the residual row for sample $i$. 
Thus the LOOCV score for a given $\lambda$ is
\begin{equation}
\label{eq:loo_score}
\text{LOOCV}(\lambda) \;=\; \frac{1}{n}\sum_{i=1}^n \bigl\|E^{(-i)}_{i\cdot}\bigr\|_2^2 
\;=\; \frac{1}{n}\sum_{i=1}^n \frac{\|E_{i\cdot}\|_2^2}{\bigl(1 - (S_\lambda)_{ii}\bigr)^2}.
\end{equation}
No refitting is required; only the diagonal of $S_\lambda$ is needed. 
For multioutput targets $Y$, the matrix $S_\lambda$ is the same for all outputs, so Eqn.~\ref{eq:loo_res}-\ref{eq:loo_score} apply by simply accumulating squared norms across columns of $Y$. Form the $p\times p$ Gram matrix $G_\lambda = X^\top X + \lambda I_p$ and compute its Cholesky factorization $G_\lambda = R^\top R$.
Then
\begin{align*}
\beta 
&= (X^\top X + \lambda I_p)^{-1} X^\top Y 
= R^{-1}\bigl(R^{-\top} X^\top Y\bigr),\\
S_\lambda 
&= X\,(X^\top X + \lambda I_p)^{-1} X^\top 
= XR^{-1}R^{-\top}X^\top 
= CC^\top,\qquad C \equiv XR^{-1}.
\end{align*}
Hence the diagonal of $S_\lambda$ is available without forming $S_\lambda$ explicitly:
\begin{equation*}
\label{eq:diagS}
(S_\lambda)_{ii} \;=\; \|C_{i\cdot}\|_2^2,\qquad i=1,\ldots,n.
\end{equation*}
This requires one $p\times p$ Cholesky factorization and triangular solves; the cost scales with $p$ rather than the number of outputs $d$.
A common strategy is to choose $\lambda$ by minimizing the LOOCV score \ref{eq:loo_score}.

\section{Empirical Evaluation} \label{app:results}
This appendix presents the extended empirical evaluation results conducted in this study. In Subsection~\ref{subsec:metric}, we outline the mathematical formulation of the performance evaluation metrics. Subsection~\ref{subsec:hyper} details the hyperparameter selection procedure, with the search spaces outlined in Table~\ref{tab:search_space_all}. The complete set of empirical results is summarized in Subsection~\ref{subsec:results}, with long-, medium-, and short-term performances being reported in Tables~\ref{tab:long_res}, \ref{tab:medium_res}, and \ref{tab:short_res}, respectively. Furthermore, the results of the Diebold-Mariano tests, evaluating the statistical significance of performance improvements for fESN, are given in Tables~\ref{tab:dm_hmedium_fesn} and \ref{tab:dm_hshort_fesn}, while for wESN the results are discussed in Tables~\ref{tab:dm_hmedium_wesn} and \ref{tab:dm_hshort_wesn}, respectively. Additionally, 
Subsection~\ref{subsec:wesn_extended} conducts an ablation study on the choice of the wavelet filters for the wESN framework for different forecast horizons. Finally, uncertainty quantification via the conformal prediction approach with calibration and coverage is visualized in Figures~\ref{fig:cal_long}, \ref{fig:cal_medium}, and \ref{fig:cal_short} of Subsection~\ref{subsec:conformal}.

% This appendix contains the remaining experimental details and results. We first describe the evaluation metrics in Subsection~\ref{subsec:metric}. Subsection~\ref{subsec:hyper} then details the hyperparameter selection procedure, with the search spaces summarized in Table~\ref{tab:search_space_all}. The main results appear in Subsection~\ref{subsec:results}, with long-, medium-, and short-horizon outcomes reported in Tables~\ref{tab:long_res}, \ref{tab:medium_res}, and \ref{tab:short_res}. To assess statistical significance, we conduct Diebold-Mariano tests: results for fESN are given in Tables~\ref{tab:dm_hmedium_fesn} and \ref{tab:dm_hshort_fesn}, and for wESN in Tables~\ref{tab:dm_hmedium_wesn} and \ref{tab:dm_hshort_wesn}. Predictive uncertainty is evaluated via conformal prediction, with calibration and coverage visualized in Figures~\ref{fig:cal_long}, \ref{fig:cal_medium}, and \ref{fig:cal_short} in Subsection~\ref{subsec:conformal}. Finally, Subsection~\ref{subsec:wesn_extended} reports an extended analysis of wESN using different wavelet families, with performance summarized in Tables~\ref{tab:long_wesn}, \ref{tab:medium_wesn}, and \ref{tab:short_wesn} for the long, medium, and short horizons.

\subsection{Performance Metrics} \label{subsec:metric}
In this study, we consider four key performance indicators, namely RMSE, MAE, sMAPE, and MASE. The mathematical formulation of these evaluation metrics is discussed below:

% To quantify and compare the accuracy of our epidemic forecasts, we employ four complementary error measures: the root-mean-square error (RMSE), mean absolute error (MAE), symmetric mean absolute percentage error (sMAPE) and mean absolute scaled error (MASE).  Together they capture both the absolute magnitude of residuals and their relative behavior against naïve benchmarks. The root-mean-square error over a test set of size $T$ is defined by
$$
  \mathrm{RMSE}
  \;=\;
  \sqrt{\frac{1}{H}\sum_{t=1}^H\bigl(u(t) - \hat u(t)\bigr)^{2}}, \; \; \; \mathrm{MAE}
  \;=\;
  \frac{1}{H}\sum_{t=1}^H\bigl\lvert u(t) - \hat u(t)\bigr\rvert,
$$

$$
  \mathrm{sMAPE} \;=\;
  \frac{100\%}{H}\sum_{t=1}^H
    \frac{\bigl\lvert u(t) - \hat u(t)\bigr\rvert}
         {\bigl(\,\lvert u(t)\rvert + \lvert \hat u(t)\rvert\,\bigr)/2}, \; \text{ and } \; 
 \mathrm{MASE} \;=\;
  \frac{\displaystyle \sum_{t=1}^H\lvert u(t) - \hat u(t)\rvert}
       {\displaystyle \tfrac{H}{T-1}\sum_{t=2}^T\lvert u(t) - u({t-1})\rvert},
$$
where $\hat u(t)$ denotes the forecast at time $t$ corresponding to the ground truth observation $u(t)$, $H$ is the forecast horizon, and $T$ is the number of training samples. By definition, lower values of these metrics indicate better model performance.  

\subsection{Bayesian Optimization for Hyperparameter Tuning} \label{subsec:hyper}

Bayesian optimization (BO) is a state‐of‐the‐art approach for selecting hyperparameters by building a probabilistic surrogate model of the objective function, i.e., the held-out validation performance \citep{bergstra2012random}. Unlike exhaustive grid or random search, BO reuses information from all past evaluations to guide its next query and thus avoids redundant assessments of poorly performing configurations. At each iteration, BO updates its surrogate model, applies an acquisition function to balance exploration and exploitation, and proposes the next hyperparameter set to evaluate. The BO procedure maintains a surrogate model of the validation error (RMSE), \(f_o(\mathbf{v})\) using the Tree‐structured Parzen Estimator (TPE) \citep{bergstra2011algorithms}, which approximates the conditional density \(P(f_o \mid \mathbf{v})\). Subsequently, we define the search space \(\mathcal{H}_{\mathbf{v}}\) to encompass all tunable hyperparameters of our model and employ the acquisition function of expected improvement (EI),
$$
  \mathrm{EI}(\mathbf{v}) \;=\; \mathbb{E}\bigl[\max\{0, f_o(\mathbf{v}) - f_o^{\star}\}\bigr],
$$
where \(f_o^{\star}\) denotes the best observed validation error to date. EI thus guides the trade‐off between exploring unexplored regions of \(\mathcal{H}_{\mathbf{v}}\) and exploiting regions known to yield high accuracy. Thus the our objective function \(f_o\) is precisely the predictive accuracy measured on a held‐out validation set. Finally, we maintain an evaluation history \(\mathcal{D}\),  all past pairs \(\{(\mathbf{v}_i, f_o(\mathbf{v}_i))\}\), which informs both surrogate updates and the maximization of the acquisition function at each iteration. The overall TPE-based BO loop is summarized in Algorithm~\ref{alg:tpe}.

\begin{algorithm}[t]
\caption{TPE Method for Hyperparameter Tuning}
\label{alg:tpe}
\KwInput{objective \(f_o\), TPE model \(\mathcal{M}\), search space \(\mathcal{H}_{\mathbf{v}}\),
acquisition function \(\mathcal{A}\), budget \(N\).}

Initialize history \(\mathcal{D} \gets \emptyset\)\;
\For{\(i = 1,\ldots,N\)}{
  Fit surrogate \(P(f_o \mid \mathbf{v})\) to \(\mathcal{D}\) using \(\mathcal{M}\)\;
  \(\mathbf{v}_i \gets \displaystyle \arg\min_{\mathbf{v}\in\mathcal{H}_{\mathbf{v}}} \mathcal{A}(\mathbf{v})\)\;
  \(y_i \gets f_o(\mathbf{v}_i)\)\;
  \(\mathcal{D} \gets \mathcal{D} \cup \{(\mathbf{v}_i, y_i)\}\)\;
}
\(\mathbf{v}^\star \gets \displaystyle \arg\min_{(\mathbf{v},y)\in\mathcal{D}} y\)\;
\textbf{return} \(\mathbf{v}^\star\)\;
\end{algorithm}

\noindent\textbf{Hyperparameter search spaces.} All models were tuned with TPE‐based Bayesian optimization over model‐specific hyperparameter domains. Rather than partitioning by architecture family, we consolidated every model’s grid into a single reference (Table~\ref{tab:search_space_all}), which lists each model alongside its tunable parameters and their candidate ranges. The best hyperparametes are selected based upon the lowest RMSE on validation set.
\begin{table}[ht]
\centering
\caption{Hyperparameter search spaces and their descriptions for all models}
\label{tab:search_space_all}
\begin{adjustbox}{max width=\textwidth} % Use max width=\textwidth to be more responsive
\begin{tabular}{@{}llp{5cm}p{4.5cm}@{}}
\toprule
\textbf{Model} & \textbf{Hyperparameter} & \textbf{Description} & \textbf{Values} \\
\midrule

\multirow{1}{*}{ARFIMA} 
  & $p \text{ and } q$        & Order of the AR and MA  & $\{1,2,\ldots 5\}$ \\
\midrule

\multirow{2}{*}{Recurrent Networks} 
  & h   & Hidden layer dimension & $\{4,8,16,32,64,128\}$ \\
  & lr & Step size for the optimizer & $\{0.001,0.01,0.1\}$    \\
\midrule

\multirow{2}{*}{Transformer} 
  & $d_{\text{model}}$    & Embedding size & $\{16,32,64\}$         \\
  & $n_{\text{head}}$     & Number of attention head & $\{2,4\}$              \\
\midrule

\multirow{3}{*}{N-Beats} 
  & $n_{block}$    & Number of blocks per stack & $\{3,5,7\}$           \\ 
  & $n_{layers}$   & Number of layers within each block & $\{3,5,7\}$           \\ 
  & layer widths  & Width of each hidden layer in a block & $\{4,8,16\}$          \\
\midrule

\multirow{3}{*}{TFT} 
  & lstm\_layers          & Number of layers  & $\{1,2,3\}$           \\ 
  & num\_attention\_heads & Number of attention heads in the model & $\{1,2,3\}$           \\ 
  & h (hidden\_size)      & Dimensionality of the hidden state & $\{32,64,128\}$       \\
\midrule
\multirow{6}{*}{ESN family (Reservoir)} 
  & $\rho_{\mathrm{x}} / \rho_{\mathrm{m}}$& Spectral radius of the reservoir & $[0.5, \dots, 0.9]$ \\
  & $\varphi_x / \varphi_m $                  & Sparsity of  reservoir & $(0.5,0.95)$    \\
  & $\zeta$  & Washout ratio & $(0.0,0.5)$ \\
  & $\psi_x / \psi_m$   & Scaling factor for the input weights & $(0,1)$     \\
  & $\sigma_x / \sigma_m $   & Isometric noise for reservoir & $(0,0.5)$ \\
\midrule
\multirow{2}{*}{Memory-augmented models} 
  & $d$            & Differencing Parameter & $(0,0.5)$      \\ 
  & $K$            & Truncation of Fractional Integration & $\{2,3,\dots,200\}$ \\
\bottomrule
\end{tabular}
\end{adjustbox}
\end{table}

\subsection{Experimental Results} \label{subsec:results}

The forecasting performance of the proposed models and baseline architectures across all dengue datasets is summarized in Tables~\ref{tab:long_res}, \ref{tab:medium_res}, and \ref{tab:short_res}, corresponding to long-, medium-, and short-term horizons, respectively. Each table reports the mean and standard deviation of all four evaluation metrics, along with the average rank of each model computed across datasets. As evident from the results, the fESN and wESN frameworks consistently achieve the lowest average ranks across horizons, demonstrating their ability to capture the underlying temporal structure of the dengue incidence series. We further observe that statistical approaches such as ARFIMA and ETS perform competitively on several datasets; in particular, ARFIMA’s strong performance aligns with its capacity to model long-range dependence and its robustness in limited-data scenarios.

% Results for the long, medium, and short horizons appear in Tables~\ref{tab:long_res}, \ref{tab:medium_res}, and \ref{tab:short_res}, respectively. In each table, the bottom block reports the average rank across all datasets for each metric. Across all three horizons, fESN or wESN attains the top average rank. We also note that simpler baselines such as ARFIMA and ETS perform well on several datasets; the competitiveness of ARFIMA is consistent with its ability to capture long memory and its effectiveness in limited data settings.

\subsubsection{Performance on very small sample size dataset} \label{para:results:para_limited_data}

The Bangkok dengue dataset consists of monthly incidence counts recorded from 2003 to 2017. In our empirical setup, data from 2003-2014 are used for training, observations from 2015-2016 serve as the validation set for tuning model hyperparameters, and the long-term (12-month) forecasts for 2017 are evaluated against the ground truth. As shown in Figure~\ref{fig:bangkok_long}, the historical series exhibits an abrupt and unprecedented surge in dengue cases at the end of 2015, with the magnitude of this spike reaching nearly four times the previous peak observed around the same month in 2013. This extreme event directly affects the behavior of the washout ratio $\zeta$ in the ESN framework. The unusually large observations in 2015 substantially inflate the validation RMSE, making the validation window less representative of the broader historical dynamics. As a result, the hyperparameter search favors a large washout ratio ($\zeta \approx 0.46$), which discards a substantial portion of the early low-incidence years (2003-2007). Consequently, the output weights $W_{\text{out}}$ are estimated using a reduced historical observations, thereby limiting generalization. In contrast, the Philippines dataset, despite having a similar number of historical data, yields a much smaller washout ratio ($\zeta \approx 0.02$), retaining roughly 105 effective samples and producing better long-term forecasting performance. This comparison demonstrates that reservoir computing remains robust even when relatively few observations are available for estimating the output weights. Therefore, the weaker long-term performance of fESN for Bangkok is plausibly attributed to the extreme surge within the validation period, which distorts the hyperparameter tuning criterion, drives $\zeta$ upward, and significantly reduces the amount of usable historical data. This interpretation is further supported by the medium-term forecasting results, where the surge falls outside the validation window, the washout ratio remains stable, and fESN achieves its best overall performance.

\begin{figure}
    \centering
    \includegraphics[width=0.85\linewidth]{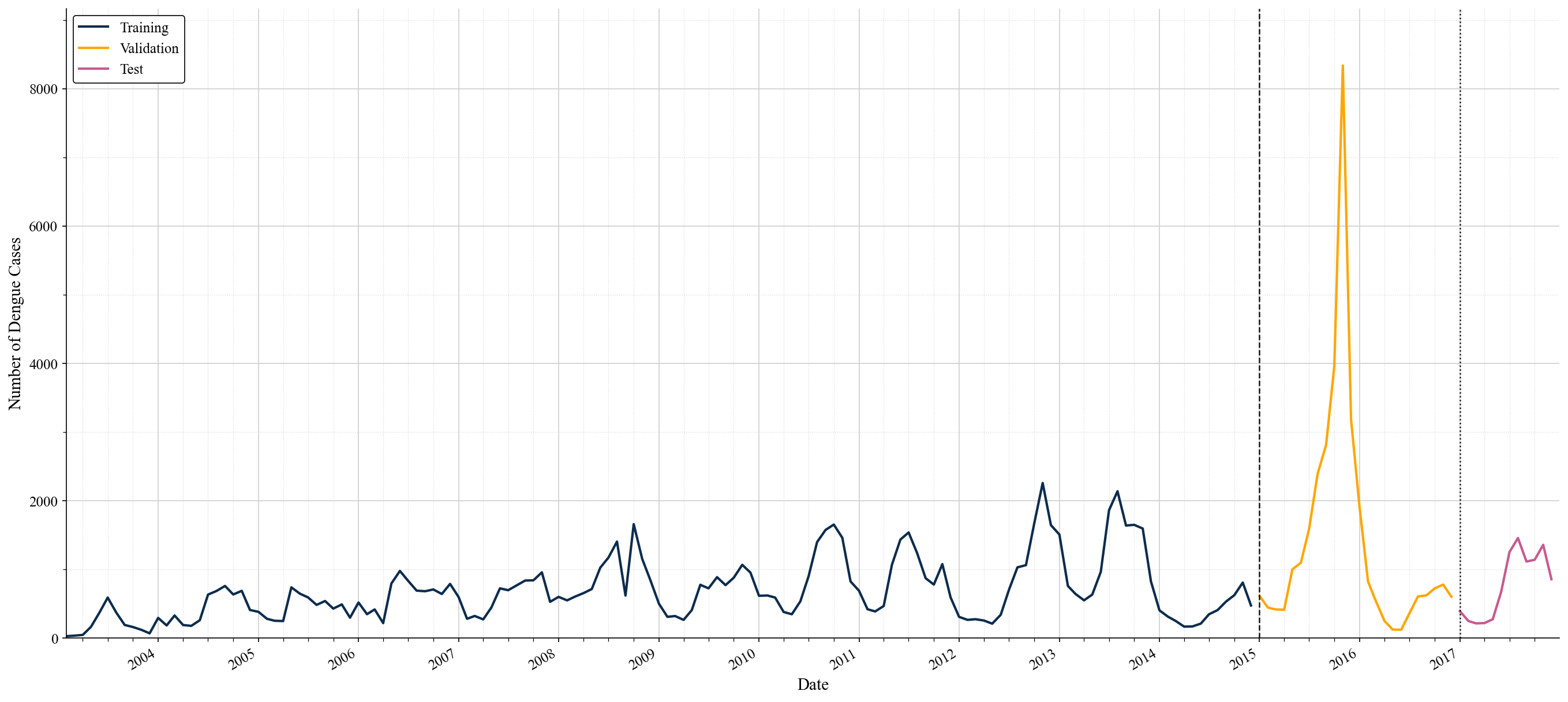}
    \caption{Train, validation, and test split of the Bangkok dengue incidence dataset.}
    \label{fig:bangkok_long}
\end{figure}

\begin{table}[ht]
\centering
\caption{Average long-term forecast performance along with standard deviations (SDs) of baseline models and memory-augmented ESN approaches for different dengue datasets. The best results are made \textbf{bold}.}
\label{tab:long_res}
\begin{adjustbox}{max width=\textwidth}
 \\
\midrule
\multirow{4}{*}{\rotatebox[origin=c]{90}{Rank}} & MAE & 6.78 & 5.33 & 8.67 & 8.22 & 9.67 & 9.44 & 8.11 & 9.00 & 8.44 & 12.22 & 7.11 & 8.33 & 11.67 & 3.67 & \textbf{3.33} \\
& MASE & 6.78 & 5.33 & 8.67 & 8.22 & 9.67 & 9.44 & 8.11 & 9.00 & 8.44 & 12.22 & 7.11 & 8.33 & 11.67 & 3.67 & \textbf{3.33} \\
& RMSE & 7.00 & 5.67 & 7.33 & 8.33 & 10.00 & 9.11 & 7.56 & 8.22 & 9.11 & 12.33 & 7.89 & 8.44 & 12.11 & \textbf{3.11} & 3.78 \\
& sMAPE & 6.89 & 4.78 & 9.22 & 8.89 & 8.89 & 9.67 & 8.33 & 8.22 & 8.44 & 11.56 & 6.22 & 9.11 & 11.44 & \textbf{4.00} & 4.33 \\
\midrule
\bottomrule
\end{tabular}
\end{adjustbox}
\end{table}

\begin{table}[ht]
\centering
\caption{Average medium-term forecast performance along with standard deviations (SDs) of baseline models and memory-augmented ESN approaches for different dengue datasets. The best results are made \textbf{bold}.}
\label{tab:medium_res}
\begin{adjustbox}{max width=\textwidth}
} \\
\midrule
\multirow{4}{*}{\rotatebox[origin=c]{90}{Rank}} & MAE & 5.44 & 4.67 & 9.78 & 8.89 & 8.44 & 9.33 & 9.00 & 10.22 & 8.22 & 11.56 & 7.89 & 9.89 & 10.78 & 3.56 & \textbf{2.33} \\
& MASE & 5.44 & 4.67 & 9.78 & 8.89 & 8.44 & 9.33 & 9.00 & 10.22 & 8.22 & 11.56 & 7.89 & 9.89 & 10.78 & 3.56 & \textbf{2.33} \\
& RMSE & 5.33 & 5.00 & 9.78 & 8.67 & 8.56 & 9.22 & 9.22 & 10.22 & 8.00 & 12.22 & 8.00 & 9.56 & 10.78 & 3.44 & \textbf{2.00} \\
& sMAPE & 5.22 & 4.56 & 9.89 & 9.00 & 8.11 & 10.22 & 9.56 & 9.44 & 7.67 & 11.67 & 8.00 & 10.00 & 10.44 & 4.00 & \textbf{2.22} \\
\midrule
\bottomrule
\end{tabular}
\end{adjustbox}
\end{table}

\begin{table}[ht]
\centering
\caption{Average short-term forecast performance along with standard deviations (SDs) of baseline models and memory-augmented ESN approaches for different dengue datasets. The best results are made \textbf{bold}.}
\label{tab:short_res}
\begin{adjustbox}{max width=\textwidth}
 \\
\midrule
\multirow{4}{*}{\rotatebox[origin=c]{90}{Rank}} & MAE & 7.78 & 5.22 & 11.00 & 10.33 & 8.89 & 10.00 & 10.78 & 9.67 & 10.11 & 9.44 & 4.89 & 7.33 & 9.56 & 2.78 & \textbf{2.22} \\
& MASE & 7.78 & 5.22 & 11.00 & 10.33 & 8.89 & 10.00 & 10.78 & 9.67 & 10.11 & 9.44 & 4.89 & 7.33 & 9.56 & 2.78 & \textbf{2.22} \\
& RMSE & 7.89 & 4.89 & 10.89 & 10.44 & 8.78 & 10.67 & 10.78 & 9.22 & 10.33 & 10.00 & 4.67 & 7.56 & 9.00 & 2.78 & \textbf{2.11} \\
& sMAPE & 7.56 & 4.78 & 10.67 & 9.89 & 8.89 & 10.11 & 9.78 & 9.00 & 11.22 & 9.67 & 4.78 & 8.11 & 9.78 & \textbf{2.89} & \textbf{2.89} \\
\midrule
\bottomrule
\end{tabular}
\end{adjustbox}
\end{table}

% ==============================================================================
% HORIZON-WISE Diebold-Mariano test results (h=medium) for fESN.
% Saved to dm_test_tables_horizon_hmedium_fesn.tex
% Required LaTeX packages: \usepackage{booktabs}, \usepackage{multirow}, \usepackage{adjustbox}, \usepackage{amsmath}
% ==============================================================================

\begin{table}[ht!]
  \centering
  \caption{Diebold-Mariano p-values comparing \textbf{fESN} against competitors across datasets for medium horizon. Significance levels: \textbf{p<0.05}, \underline{p<0.10}.}
  \label{tab:dm_hmedium_fesn}
  \begin{adjustbox}{width=\textwidth,center}
  \begin{tabular}{@{\extracolsep{5pt}}l|ccccccccccccc}
  \toprule
  Dataset & ETS & ARFIMA & RNN & MRNN & LSTM & MLSTM & GRU & MGRU & wTransformer & TFT & TSMixer & N-BEATS & ESN \\
  \midrule
Ahmedabad & \textbf{\textless 0.001} & \textbf{\textless 0.012} & \textbf{\textless 0.016} & \textbf{\textless 0.028} & \textbf{0.003} & \underline{0.058} & \textbf{\textless 0.049} & \textbf{\textless 0.021} & \textbf{\textless 0.001} & \textbf{\textless 0.001} & \textbf{\textless 0.037} & \textbf{\textless 0.012} & 0.121 \\
Bangkok & \textbf{\textless 0.001} & \textbf{\textless 0.012} & \textbf{\textless 0.001} & \textbf{\textless 0.001} & \textbf{\textless 0.001} & \textbf{\textless 0.001} & \textbf{\textless 0.001} & \textbf{\textless 0.001} & \textbf{\textless 0.001} & \textbf{\textless 0.001} & \textbf{\textless 0.001} & \textbf{0.002} & 0.108 \\
Columbia & \textbf{\textless 0.001} & \textbf{\textless 0.001} & \textbf{\textless 0.001} & \textbf{\textless 0.001} & \textbf{\textless 0.001} & \textbf{\textless 0.001} & \textbf{\textless 0.001} & \textbf{\textless 0.001} & \textbf{\textless 0.001} & \textbf{\textless 0.001} & \textbf{\textless 0.001} & \textbf{\textless 0.001} & \textbf{\textless 0.001} \\
Hong Kong & 0.172 & 0.495 & \textbf{\textless 0.011} & 0.125 & 0.148 & 0.124 & \textbf{0.004} & 0.106 & \textbf{\textless 0.039} & \textbf{0.002} & 0.230 & 0.276 & \textbf{0.007} \\
Iquitos & \textbf{\textless 0.001} & \textbf{0.001} & \textbf{\textless 0.001} & \textbf{\textless 0.001} & \textbf{\textless 0.001} & \textbf{\textless 0.001} & \textbf{\textless 0.001} & \textbf{\textless 0.001} & 0.216 & \textbf{0.008} & 0.410 & \textbf{\textless 0.001} & \textbf{\textless 0.001} \\
Phillipines & 0.874 & 0.657 & \textbf{0.004} & \textbf{\textless 0.011} & \textbf{\textless 0.017} & \textbf{0.004} & \textbf{0.002} & \textbf{\textless 0.019} & \textbf{\textless 0.001} & 0.209 & \textbf{\textless 0.043} & \textbf{\textless 0.001} & 0.114 \\
Sanjuan & 0.878 & 0.236 & \textbf{\textless 0.001} & \textbf{\textless 0.001} & \textbf{\textless 0.001} & \textbf{\textless 0.001} & \textbf{\textless 0.001} & \textbf{\textless 0.001} & \textbf{\textless 0.001} & \textbf{\textless 0.001} & \textbf{\textless 0.001} & \textbf{\textless 0.001} & \textbf{\textless 0.001} \\
Singapore & \textbf{\textless 0.001} & 0.241 & 0.102 & \textbf{\textless 0.049} & \textbf{\textless 0.001} & \textbf{\textless 0.001} & \textbf{\textless 0.001} & 0.141 & 0.314 & \textbf{\textless 0.001} & \textbf{\textless 0.001} & 0.146 & \textbf{\textless 0.001} \\
Venezuela & 0.431 & 0.333 & \textbf{\textless 0.001} & \textbf{\textless 0.001} & \textbf{\textless 0.001} & \textbf{\textless 0.001} & \textbf{\textless 0.001} & \textbf{\textless 0.001} & 0.402 & \textbf{\textless 0.001} & \textbf{\textless 0.001} & \textbf{\textless 0.001} & \textbf{\textless 0.001} \\
  \bottomrule
  \end{tabular}
  \end{adjustbox}
\end{table}

% ==============================================================================
% HORIZON-WISE Diebold-Mariano test results (h=medium) for wESN.
% Saved to dm_test_tables_horizon_hmedium_wesn.tex
% Required LaTeX packages: \usepackage{booktabs}, \usepackage{multirow}, \usepackage{adjustbox}, \usepackage{amsmath}
% ==============================================================================

\begin{table}[ht!]
  \centering
  \caption{Diebold-Mariano p-values comparing \textbf{wESN} against competitors across datasets for medium horizon. Significance levels: \textbf{p<0.05}, \underline{p<0.10}.}
  \label{tab:dm_hmedium_wesn}
  \begin{adjustbox}{width=\textwidth,center}
  \begin{tabular}{@{\extracolsep{5pt}}l|ccccccccccccc}
  \toprule
  Dataset & ETS & ARFIMA & RNN & MRNN & LSTM & MLSTM & GRU & MGRU & wTransformer & TFT & TSMixer & N-BEATS & ESN \\
  \midrule
Ahmedabad & \textbf{\textless 0.001} & \textbf{\textless 0.015} & \underline{0.095} & 0.148 & \textbf{\textless 0.019} & 0.402 & 0.255 & 0.125 & \textbf{\textless 0.001} & \textbf{0.003} & 0.210 & \underline{0.063} & 0.118 \\
Bangkok & \textbf{\textless 0.001} & \underline{0.055} & \textbf{\textless 0.001} & \textbf{\textless 0.001} & \textbf{\textless 0.001} & \textbf{\textless 0.001} & \textbf{\textless 0.001} & \textbf{\textless 0.001} & \textbf{0.002} & \textbf{\textless 0.001} & \textbf{\textless 0.001} & \textbf{0.002} & \underline{0.059} \\
Columbia & \textbf{\textless 0.001} & \textbf{\textless 0.001} & \textbf{\textless 0.001} & \textbf{\textless 0.001} & \textbf{\textless 0.001} & \textbf{\textless 0.001} & \textbf{\textless 0.001} & \textbf{\textless 0.001} & \textbf{\textless 0.001} & \textbf{\textless 0.001} & \textbf{\textless 0.001} & \textbf{\textless 0.001} & \textbf{\textless 0.001} \\
Hong Kong & 0.144 & \underline{0.097} & 0.289 & \underline{0.098} & 0.873 & 0.528 & 0.130 & 0.339 & 0.799 & \underline{0.054} & 0.254 & 0.370 & 0.794 \\
Iquitos & \textbf{\textless 0.001} & \textbf{0.003} & \textbf{\textless 0.001} & \textbf{\textless 0.001} & \textbf{\textless 0.001} & \textbf{\textless 0.001} & \textbf{\textless 0.001} & \textbf{\textless 0.001} & \textbf{\textless 0.001} & \textbf{\textless 0.001} & \textbf{\textless 0.001} & \textbf{\textless 0.001} & \textbf{\textless 0.001} \\
Phillipines & 0.844 & 0.515 & \textbf{0.005} & \textbf{\textless 0.013} & \textbf{\textless 0.001} & \textbf{0.004} & \underline{0.016} & \textbf{\textless 0.027} & \textbf{\textless 0.001} & 0.177 & \underline{0.056} & \textbf{\textless 0.001} & 0.499 \\
Sanjuan & 0.786 & \textbf{\textless 0.025} & \textbf{\textless 0.001} & \textbf{\textless 0.001} & \textbf{\textless 0.001} & \textbf{\textless 0.001} & \textbf{\textless 0.001} & \textbf{\textless 0.001} & \textbf{\textless 0.001} & \textbf{\textless 0.001} & \textbf{\textless 0.001} & \textbf{\textless 0.001} & \textbf{\textless 0.001} \\
Singapore & \textbf{\textless 0.001} & \textbf{\textless 0.023} & \textbf{\textless 0.020} & 0.289 & \textbf{\textless 0.001} & \textbf{\textless 0.001} & \textbf{\textless 0.001} & \textbf{\textless 0.011} & 0.312 & \textbf{\textless 0.001} & \textbf{\textless 0.001} & \textbf{\textless 0.034} & \textbf{\textless 0.001} \\
Venezuela & 0.317 & 0.932 & \textbf{\textless 0.001} & \textbf{\textless 0.001} & \textbf{\textless 0.071} & \textbf{\textless 0.001} & \textbf{\textless 0.001} & \textbf{\textless 0.001} & 0.817 & \textbf{\textless 0.001} & \textbf{\textless 0.001} & \textbf{\textless 0.001} & \underline{0.051} \\
  \bottomrule
  \end{tabular}
  \end{adjustbox}
\end{table}
% ==============================================================================
% HORIZON-WISE Diebold-Mariano test results (h=short) for fESN.
% Saved to dm_test_tables_horizon_hshort_fesn.tex
% Required LaTeX packages: \usepackage{booktabs}, \usepackage{multirow}, \usepackage{adjustbox}, \usepackage{amsmath}
% ==============================================================================

\begin{table}[ht!]
  \centering
  \caption{Diebold-Mariano p-values comparing \textbf{fESN} against competitors across datasets for short horizon. Significance levels: \textbf{p<0.05}, \underline{p<0.10}.}
  \label{tab:dm_hshort_fesn}
  \begin{adjustbox}{width=\textwidth,center}
  \begin{tabular}{@{\extracolsep{5pt}}l|ccccccccccccc}
  \toprule
  Dataset & ETS & ARFIMA & RNN & MRNN & LSTM & MLSTM & GRU & MGRU & wTransformer & TFT & TSMixer & N-BEATS & ESN \\
  \midrule
Ahmedabad & \textbf{\textless 0.001} & \textbf{\textless 0.010} & \textbf{\textless 0.017} & \underline{0.059} & 0.702 & 0.785 & 0.182 & 0.256 & \textbf{\textless 0.001} & \textbf{0.003} & 0.691 & 0.192 & \textbf{\textless 0.015} \\
Bangkok & 0.885 & 0.892 & 0.130 & 0.388 & 0.192 & 0.198 & 0.190 & 0.205 & 0.113 & \textbf{\textless 0.045} & 0.706 & \underline{0.066} & 0.588 \\
Columbia & \textbf{\textless 0.001} & \textbf{0.004} & \textbf{\textless 0.001} & \textbf{\textless 0.001} & \textbf{\textless 0.001} & \textbf{\textless 0.001} & \textbf{\textless 0.001} & \textbf{\textless 0.001} & \textbf{\textless 0.001} & \underline{0.073} & \textbf{\textless 0.001} & \textbf{\textless 0.001} & \textbf{\textless 0.001} \\
Hong Kong & 0.510 & 0.238 & \textbf{\textless 0.022} & 0.182 & \textbf{\textless 0.001} & \textbf{\textless 0.001} & 0.210 & \textbf{\textless 0.001} & 0.352 & 0.170 & 0.612 & 0.139 & 0.276 \\
Iquitos & \textbf{\textless 0.001} & 0.234 & \textbf{\textless 0.001} & \textbf{\textless 0.001} & \textbf{0.003} & \textbf{\textless 0.001} & \textbf{0.003} & \textbf{0.006} & 0.416 & 0.199 & 0.986 & \underline{0.053} & \underline{0.082} \\
Phillipines & 0.612 & 0.689 & 0.487 & \underline{0.089} & 0.503 & 0.582 & 0.406 & 0.320 & 0.129 & 0.319 & 0.393 & 0.383 & 0.740 \\
Sanjuan & \textbf{\textless 0.001} & \textbf{\textless 0.001} & \textbf{\textless 0.001} & \textbf{\textless 0.001} & \textbf{\textless 0.001} & \textbf{\textless 0.001} & \textbf{\textless 0.001} & \textbf{\textless 0.001} & \textbf{\textless 0.001} & \textbf{\textless 0.047} & 0.439 & \textbf{\textless 0.001} & \textbf{\textless 0.001} \\
Singapore & \textbf{0.001} & \underline{0.068} & \textbf{\textless 0.001} & \textbf{0.005} & \textbf{\textless 0.001} & \textbf{0.006} & \textbf{\textless 0.049} & \textbf{\textless 0.001} & \textbf{\textless 0.001} & 0.635 & \textbf{0.004} & \underline{0.074} & \textbf{\textless 0.001} \\
Venezuela & \textbf{0.001} & \textbf{\textless 0.025} & \textbf{\textless 0.001} & \textbf{\textless 0.015} & \textbf{\textless 0.001} & 0.286 & \textbf{\textless 0.001} & 0.445 & \textbf{\textless 0.001} & 0.110 & \textbf{\textless 0.001} & \textbf{\textless 0.001} & \textbf{\textless 0.001} \\
  \bottomrule
  \end{tabular}
  \end{adjustbox}
\end{table}

% ==============================================================================
% HORIZON-WISE Diebold-Mariano test results (h=short) for wESN.
% Saved to dm_test_tables_horizon_hshort_wesn.tex
% Required LaTeX packages: \usepackage{booktabs}, \usepackage{multirow}, \usepackage{adjustbox}, \usepackage{amsmath}
% ==============================================================================

\begin{table}[ht!]
  \centering
  \caption{Diebold-Mariano p-values comparing \textbf{wESN} against competitors across datasets for short horizon. Significance levels: \textbf{p<0.05}, \underline{p<0.10}.}
  \label{tab:dm_hshort_wesn}
  \begin{adjustbox}{width=\textwidth,center}
  \begin{tabular}{@{\extracolsep{5pt}}l|ccccccccccccc}
  \toprule
  Dataset & ETS & ARFIMA & RNN & MRNN & LSTM & MLSTM & GRU & MGRU & wTransformer & TFT & TSMixer & N-BEATS & ESN \\
  \midrule
Ahmedabad & \textbf{\textless 0.001} & \textbf{0.007} & 0.150 & 0.442 & 0.970 & 0.989 & 0.489 & 0.807 & \textbf{\textless 0.001} & \textbf{\textless 0.032} & 0.787 & 0.639 & 0.860 \\
Bangkok & 0.882 & 0.436 & 0.232 & 0.314 & 0.269 & 0.320 & 0.280 & 0.188 & 0.127 & \textbf{\textless 0.046} & 0.768 & 0.114 & 0.669 \\
Columbia & \textbf{\textless 0.001} & \textbf{\textless 0.022} & \textbf{\textless 0.001} & \textbf{\textless 0.001} & \textbf{\textless 0.001} & \textbf{\textless 0.001} & \textbf{\textless 0.001} & \textbf{\textless 0.001} & \textbf{\textless 0.001} & \underline{0.095} & \textbf{0.005} & \textbf{0.001} & \textbf{\textless 0.001} \\
Hong Kong & 0.498 & \underline{0.068} & \underline{0.090} & 0.239 & \textbf{\textless 0.001} & \textbf{\textless 0.001} & 0.141 & \textbf{\textless 0.001} & 0.315 & 0.193 & 0.789 & 0.175 & 0.992 \\
Iquitos & \textbf{\textless 0.001} & 0.432 & \textbf{0.003} & \textbf{\textless 0.001} & \textbf{\textless 0.040} & \textbf{\textless 0.001} & \textbf{\textless 0.015} & \textbf{\textless 0.026} & 0.666 & 0.406 & 0.319 & 0.172 & \underline{0.072} \\
Phillipines & 0.691 & \textbf{\textless 0.047} & 0.535 & 0.108 & 0.556 & 0.635 & 0.470 & 0.338 & 0.334 & 0.380 & 0.545 & 0.459 & 0.552 \\
Sanjuan & \textbf{\textless 0.001} & \textbf{\textless 0.001} & \textbf{\textless 0.001} & \textbf{\textless 0.001} & \textbf{\textless 0.001} & \textbf{\textless 0.001} & \textbf{\textless 0.001} & \textbf{\textless 0.001} & \textbf{\textless 0.001} & \underline{0.061} & 0.983 & \textbf{\textless 0.001} & \textbf{\textless 0.001} \\
Singapore & \textbf{\textless 0.001} & \textbf{\textless 0.011} & \textbf{\textless 0.001} & \textbf{\textless 0.001} & \textbf{\textless 0.001} & \textbf{\textless 0.001} & \textbf{\textless 0.001} & \textbf{\textless 0.001} & \textbf{\textless 0.001} & \underline{0.092} & \textbf{\textless 0.001} & \textbf{\textless 0.001} & \textbf{\textless 0.001} \\
Venezuela & \textbf{\textless 0.001} & \textbf{\textless 0.037} & \textbf{\textless 0.001} & \textbf{\textless 0.001} & \textbf{0.007} & \textbf{0.006} & \textbf{\textless 0.001} & \textbf{0.008} & \textbf{\textless 0.001} & \textbf{0.001} & \textbf{\textless 0.001} & \textbf{\textless 0.001} & \textbf{\textless 0.001} \\
  \bottomrule
  \end{tabular}
  \end{adjustbox}
\end{table}

\subsection{Ablation Study} \label{subsec:wesn_extended}

In this section, we present the results of various wavelets applied to wESN to examine its robustness to the choice of wavelet filter. We utilize Biorthogonal (Bior), Symlets (Sym), and Daubechies filters of order 1 (Haar) and order 2 (Db2). All wavelet families are applied at three different decomposition levels. The forecasting results reported in Tables~\ref{tab:long_wesn}, \ref{tab:medium_wesn}, and \ref{tab:short_wesn} for long, medium, and short-term horizons indicate that wESN achieves consistent performance across all filter-level combinations. To verify whether any specific configuration yields a statistically significant performance improvement, we apply the nonparametric Friedman test \citep{friedman1937use}. Based on the p-values (RMSE: 0.3718, MAE: 0.1552, sMAPE: 0.2363, MASE: 0.1552), we fail to reject the null hypothesis of identical model performance. Therefore, the WESN model is robust to the choice of wavelet transform. 

% In this section, we present the results of various wavelets applied to wESN. We have selected the Bior, Db2, Haar, and Sym wavelets, each evaluated at three levels. The results are detailed in Tables~\ref{tab:long_wesn}, ~\ref{tab:medium_wesn}, and ~\ref{tab:short_wesn} for Long, Medium, and Short Horizons, respectively. 

% To assess whether any combination of wavelet and level showed significant differences, we conducted a Friedman test. However, we did not find any wavelet that demonstrated a significant advantage; the performance across all four metrics was identical. Specifically, the p-values for the performance metrics were as follows: RMSE 0.3718, MAE 0.1552, sMAPE 0.2363, and MASE 0.1552. Therefore, we fail to reject the null hypothesis, concluding that the performances are identical.
\begin{table}[htbp]
\centering
\caption{Average long-term forecasting performance (with standard deviation) of wESN across different wavelet filters and decomposition levels. For example, wESN-Bior-1 denotes the use of a Biorthogonal filter with one decomposition level. Best results are highlighted in \textbf{bold}.}
\label{tab:long_wesn}
\begin{adjustbox}{max width=\textwidth}

\end{adjustbox}
\end{table}

\begin{table}[htbp]
\centering
\caption{Average medium-term forecasting performance (with standard deviation) of wESN across different wavelet filters and decomposition levels. For example, wESN-Bior-1 denotes the use of a Biorthogonal filter with one decomposition level. Best results are highlighted in \textbf{bold}.}
\label{tab:medium_wesn}
\begin{adjustbox}{max width=\textwidth}
%
\end{adjustbox}
\end{table}

\begin{table}[htbp]
\centering
\caption{Average short-term forecasting performance (with standard deviation) of wESN across different wavelet filters and decomposition levels. For example, wESN-Bior-1 denotes the use of a Biorthogonal filter with one decomposition level. Best results are highlighted in \textbf{bold}.}
\label{tab:short_wesn}
\begin{adjustbox}{width=\textwidth,center}
%
} \\
\bottomrule
\end{tabular}
\end{adjustbox}
\end{table}

\subsection{Conformal Prediction Intervals} \label{subsec:conformal}
To quantify the uncertainty associated with the proposed models, we utilize the conformal prediction approach. This distribution-free framework provides a set of model-agnostic techniques to convert point forecasts into prediction intervals that are guaranteed to contain the true value with a user-specified confidence. The construction of these conformal prediction intervals relies on a calibration (validation) set to ensure no data leakage. After training a forecasting model on the historical data, the model's residuals are computed on a separate validation window that immediately precedes the test period. These residuals, when scaled using the predictions of an uncertainty model, form conformal scores whose empirical distribution captures the model’s prediction errors. The prediction interval for each future time point is then obtained by selecting the appropriate quantile from this distribution, chosen to achieve the desired coverage level. In our empirical setup, %we partition our dataset into disjoint subsets $(\mathcal{D}{\mathrm{train}}, \mathcal{D}{\mathrm{val}}, \text{ and } \mathcal{D}_{\mathrm{test}})$. 
given the lagged training data $\tilde{u}{(t-1)}$, this framework fits the forecasting model (FM) and an uncertainty model (UM) to generate a scalar notion of uncertainty. The conformal score $S_t$ is then computed as:
$$S_t = \frac{|u(t) - \operatorname{FM}(\tilde{u}({t-1}))|}{\operatorname{UM}(\tilde{u}({t-1}))}.$$
Using the temporal dependencies of the time series, the conformal quantile can be calculated with a fixed $\kappa$-sized window $\gamma_{t^{\prime}}=\mathbf{1}\left\{t^{\prime} \geq t-\kappa\right), \forall t^{\prime}<t$ as
$$
\mathbb{Q}_t=\inf \left\{g: \frac{1}{\min \left(\kappa, t^{\prime}-1\right)+1} \sum_{t^{\prime}=1}^{t-1} S_{t^{\prime}} \gamma_{t^{\prime}} \geq 1-\alpha\right\}.
$$
These weighted conformal quantiles are used to compute the prediction interval at time $t$ as $$\mathbb{CQ}_t = \left[\operatorname{FM}\left(\tilde{u}({t-1})\right)\pm\mathbb{Q}_t \operatorname{UM}\left(\tilde{u}({t-1})\right)\right].$$ The conformal prediction bands with a 90\% target coverage, alongside the point forecast of ESN, fESN, and wESN models with the ground truth observations, are depicted in Figures \ref{fig:cal_long}, \ref{fig:cal_medium}, and \ref{fig:cal_short} for long, medium, and short-term forecast horizons, respectively. As evident from the plots, the proposed memory-augmented ESN frameworks can consistently capture the complex dynamics of the dengue incidence cases for most of the dataset.

Additionally, to comprehensively evaluate the performance of the conformal prediction intervals during the test period, we compute coverage, mean width, and the Winkler score. Coverage determines reliability by measuring the proportion of ground-truth observations that fall within the prediction intervals, while the mean width quantifies its precision. The proper scoring rule-based Winkler Score \citep{gneiting2007strictly} offers a trade-off measure quantifying both reliability and precision of these intervals. Tables \ref{tab:summary_medium} and \ref{tab:summary_short} summarize these metrics for medium and short-term forecast horizons. Overall, fESN achieves the lowest Winkler score for most datasets, indicating the best balance between coverage and mean width, except for the short-term Bangkok forecast. The medium-term horizons show a decline in calibration with increasing lead time. All three models fail to provide higher coverage to the San Juan dengue dataset due to its abrupt regime changes. This suggests that systematic forecast errors cannot be fully corrected by marginal conformalization. For the short-term forecast horizon, the fESN model provides a higher coverage with more precise intervals than for the medium-term horizon. These empirical results highlight that the proposed memory-augmented ESN models, particularly fESN, provide reliable and informative conformal prediction intervals, effectively balancing uncertainty quantification with point forecast accuracy.

\begin{figure}
    \centering
    
    \includegraphics[width=0.9\linewidth, height=0.9\textheight]{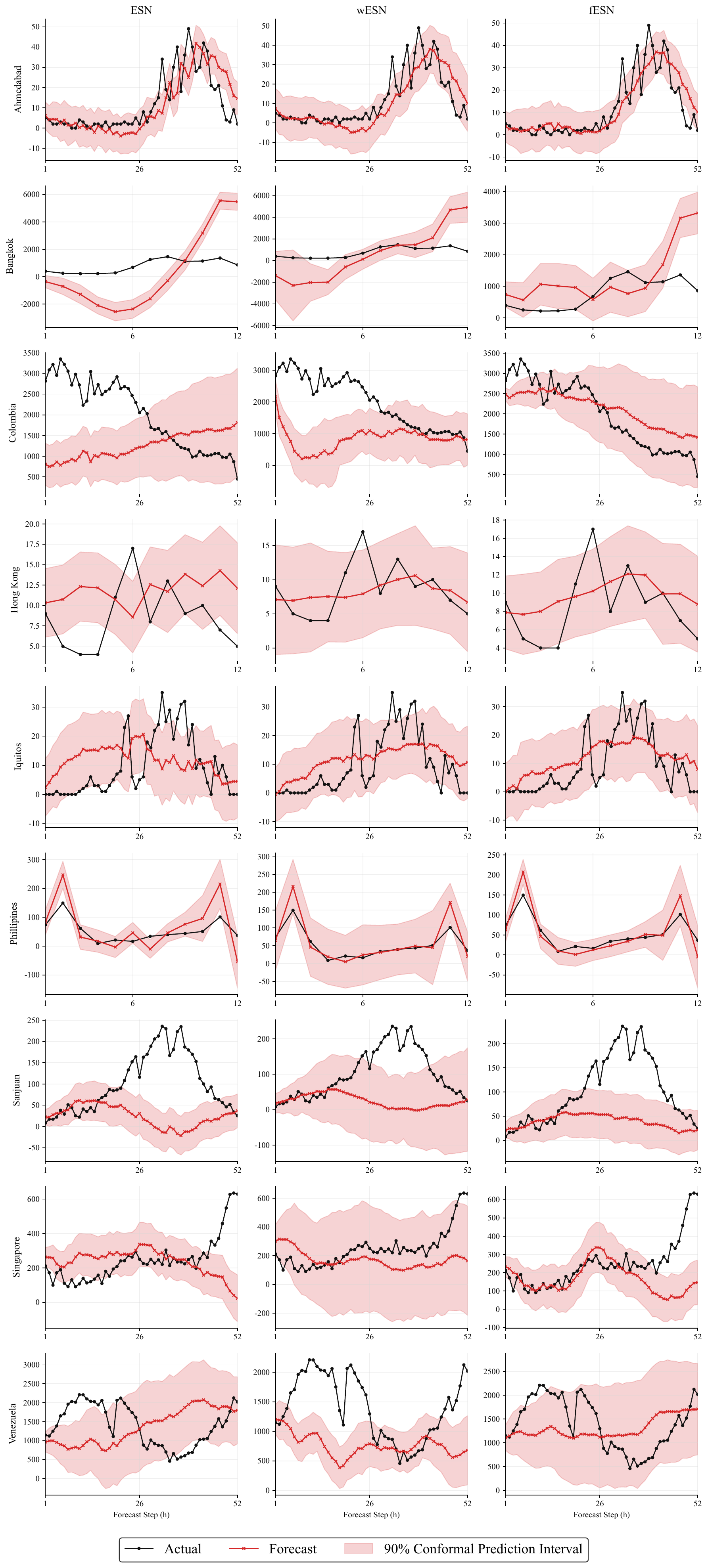}
    \caption{Conformal prediction interval (shaded region) of ESN, wESN, and fESN along with the point forecasts (red) and ground truth (black) observations for long-term horizon.}
    \label{fig:cal_long}
\end{figure}

\begin{figure}
    \centering
    \includegraphics[width=0.9\linewidth, height=0.9\textheight]{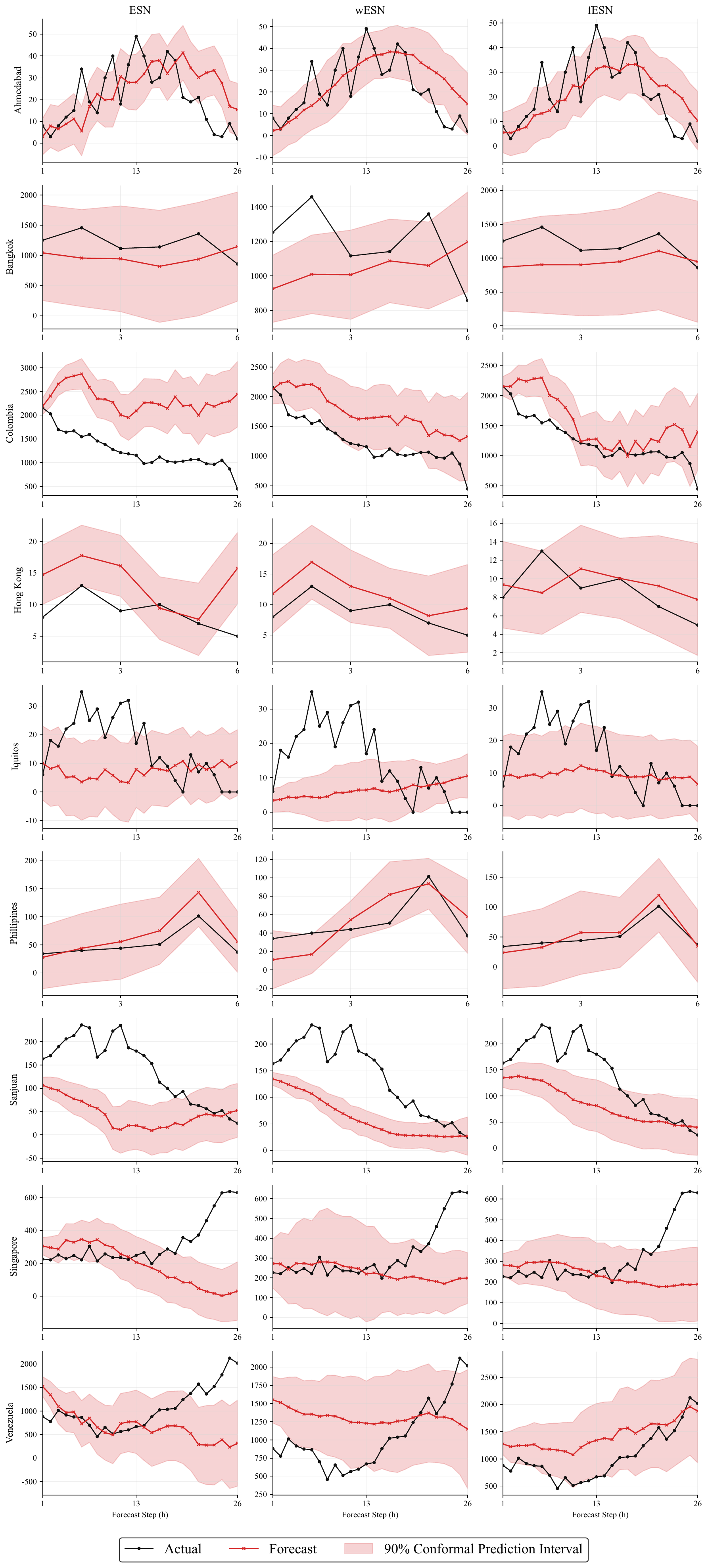}
    \caption{Conformal prediction interval (shaded region) of ESN, wESN, and fESN along with the point forecasts (red) and ground truth (black) observations for medium-term horizon.}
    \label{fig:cal_medium}
\end{figure}

\begin{figure}
    \centering
    \includegraphics[width=0.9\linewidth, height=0.9\textheight]{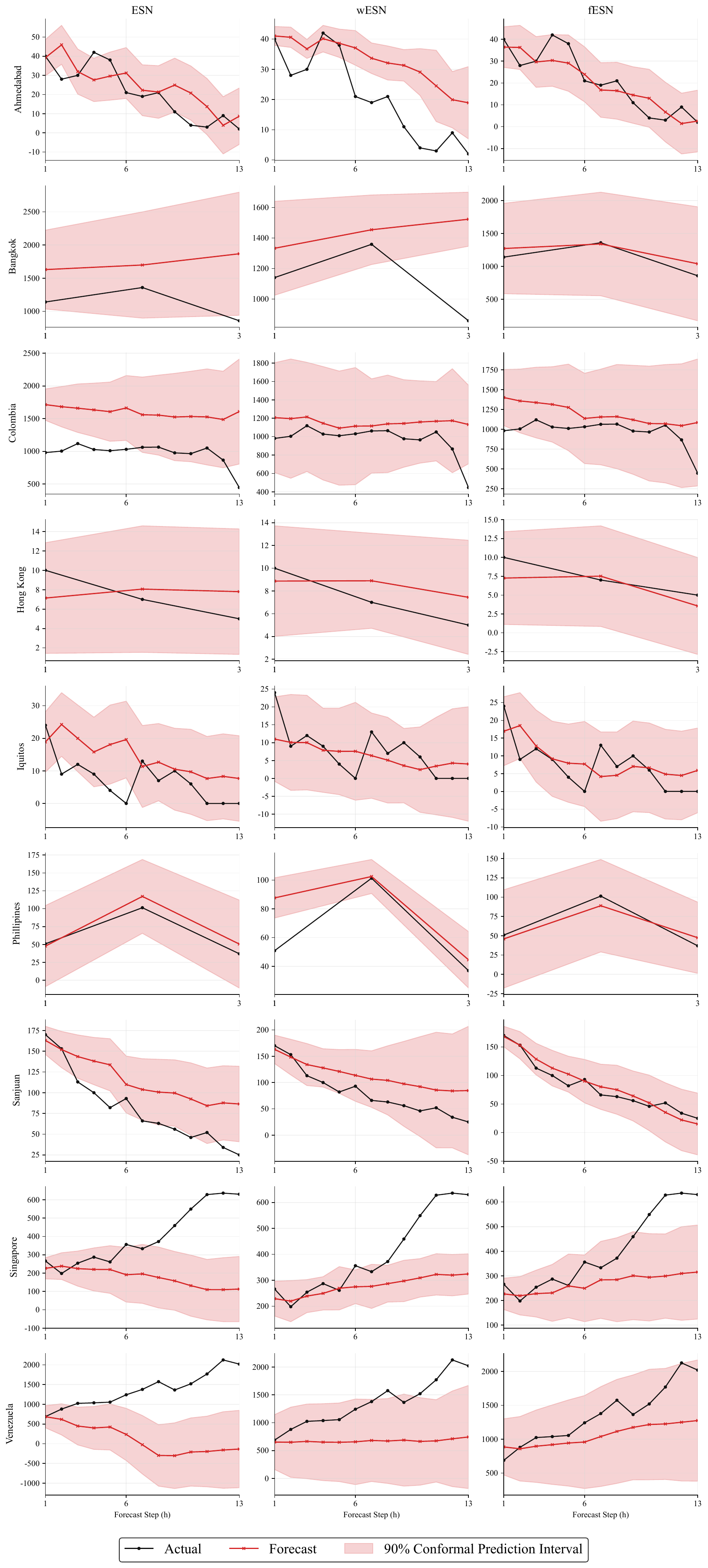}
    \caption{Conformal prediction interval (shaded region) of ESN, wESN, and fESN along with the point forecasts (red) and ground truth (black) observations for short-term horizon.}
    \label{fig:cal_short}
\end{figure}

\begin{table}[htbp]
\centering
\caption{Medium-horizon conformal prediction performance, reported as Coverage (Mean Interval Width, Winkler Score). For each dataset, the minimum Winkler Score across models is shown in \textbf{bold}.}
\label{tab:summary_medium}
\begin{adjustbox}{}
%{max width=\textwidth}
% --- CONTENT FROM PANDAS ---
\begin{tabular}{llll}
\toprule
Model ($\rightarrow$) &                  \multirow{2}{*}{ESN} &                          \multirow{2}{*}{wESN}   &                      \multirow{2}{*}{fESN}  \\
Dataset ($\downarrow$)     &                           &                                 &                                   \\
\midrule
Ahmedabad   &      0.65 (23.54, 90.38) &               0.77 (25.25, 49.80) &      \textbf{0.77 (22.58, 48.91)} \\
Bangkok     &  1.00 (1747.88, 1747.88) &            0.33 (487.24, 1991.14) &  \textbf{1.00 (1557.43, 1557.43)} \\
Colombia    &  0.04 (944.07, 13038.67) &  \textbf{0.46 (1011.86, 2125.41)} &            0.69 (892.65, 2389.77) \\
Hong Kong   &      0.50 (10.22, 41.64) &               1.00 (12.33, 12.33) &       \textbf{0.83 (9.91, 10.02)} \\
Iquitos     &      0.65 (25.54, 90.60) &              0.35 (13.52, 148.66) &      \textbf{0.65 (25.08, 56.30)} \\
Phillipines &    1.00 (119.85, 119.85) &      \textbf{0.83 (58.50, 65.85)} &             1.00 (124.99, 124.99) \\
Sanjuan     &    0.27 (93.17, 1333.59) &             0.12 (47.28, 1181.42) &     \textbf{0.42 (88.82, 626.16)} \\
Singapore   &   0.50 (274.76, 2141.85) &            0.54 (392.37, 2287.54) &   \textbf{0.69 (284.69, 1160.68)} \\
Venezuela   &  0.65 (1289.62, 4564.18) &           0.46 (1171.94, 2683.15) &  \textbf{0.73 (1276.45, 1777.16)} \\
\bottomrule
\end{tabular}

% --- END CONTENT ---
\end{adjustbox}
\end{table}

\begin{table}[htbp]
\centering
\caption{Short-horizon conformal prediction performance, reported as Coverage (Mean Interval Width, Winkler Score). For each dataset, the minimum Winkler Score across models is shown in \textbf{bold}.}
\label{tab:summary_short}
\begin{adjustbox}{}
%{max width=\textwidth}
% --- CONTENT FROM PANDAS ---
\begin{tabular}{llll}
\toprule
Model ($\rightarrow$) &                  \multirow{2}{*}{ESN} &                          \multirow{2}{*}{wESN}   &                      \multirow{2}{*}{fESN}  \\
Dataset ($\downarrow$)     &                           &                                 &                                   \\
\midrule
Ahmedabad   &       0.77 (25.85, 46.56) &        0.23 (12.50, 146.27) &      \textbf{1.00 (24.90, 24.90)} \\
Bangkok     &   0.67 (1548.99, 2095.49) &      0.67 (475.35, 3729.58) &  \textbf{1.00 (1559.14, 1559.14)} \\
Colombia    &   0.46 (1092.49, 3968.84) &     0.92 (1093.26, 1491.23) &  \textbf{0.92 (1197.94, 1298.42)} \\
Hong Kong   &       1.00 (12.47, 12.47) &  \textbf{1.00 (9.37, 9.37)} &               1.00 (12.83, 12.83) \\
Iquitos     &       0.77 (23.56, 47.00) &         0.92 (25.73, 27.53) &      \textbf{0.92 (23.06, 23.45)} \\
Phillipines &     1.00 (113.27, 113.27) &        0.67 (30.36, 181.87) &    \textbf{1.00 (113.16, 113.16)} \\
Sanjuan     &      0.38 (70.12, 171.90) &       1.00 (132.93, 132.93) &      \textbf{1.00 (77.66, 77.66)} \\
Singapore   &    0.46 (276.00, 2556.36) &      0.31 (149.31, 3060.77) &   \textbf{0.69 (285.00, 1046.63)} \\
Venezuela   &  0.15 (1388.07, 13064.93) &     0.62 (1475.14, 3749.08) &  \textbf{0.92 (1386.93, 1398.85)} \\
\bottomrule
\end{tabular}

% --- END CONTENT ---
\end{adjustbox}
\end{table}

%% or include bibliography directly:
% \begin{thebibliography}{}
% \bibitem[\protect\citeauthoryear{???}{???}]{b1}
% \end{thebibliography}
\end{document}